\documentclass{article}

\usepackage[preprint]{neurips_2021}

\usepackage{multibib}
\newcites{supp}{References}

\usepackage[utf8]{inputenc} %
\usepackage[T1]{fontenc}    %
\usepackage{hyperref}       %
\usepackage{url}            %
\usepackage{booktabs}       %
\usepackage{amsfonts,amsmath,amsthm}       %
\usepackage{nicefrac}       %
\usepackage{microtype}      %
\usepackage{xcolor}         %
\usepackage{bbm}
\usepackage{mathtools}
\usepackage{etextools}
\usepackage{ifthen}
\usepackage{booktabs}
\usepackage{cancel}

\usepackage{float}

\allowdisplaybreaks

\usepackage{caption,subcaption}

\usepackage{algorithm}
\usepackage{algorithmic}

\usepackage{wrapfig} 
\usepackage{floatflt}

\usepackage{placeins} 

\usepackage{graphicx}
\graphicspath{ {figures/} }
\DeclareGraphicsExtensions{.pdf,.png,.jpg}

\usepackage{xspace}
\makeatletter
\DeclareRobustCommand\onedot{\futurelet\@let@token\@onedot}
\def\@onedot{\ifx\@let@token.\else.\null\fi\xspace}

\def\iid{{i.i.d}\onedot}
\def\eg{{e.g}\onedot} 
\def\ie{{i.e}\onedot}

\makeatother

\newcommand*{\nolink}[1]{%
  {\protect\NoHyper#1\protect\endNoHyper}%
}

\newtheorem{theorem}{Theorem}
\newtheorem{lemma}{Lemma}

\theoremstyle{definition}
\newtheorem{definition}{Definition}

\DeclareMathOperator{\disp}{\operatorname{disp}}
\DeclareMathOperator{\disc}{\operatorname{disc}}
\DeclareMathOperator{\disb}{\operatorname{disb}}

\DeclareMathOperator{\TV}{\operatorname{TV}}

\newcommand{\A}{\mathcal{A}}
\newcommand{\E}{\mathbb{E}}
\renewcommand{\H}{\mathcal{H}}

\newcommand{\R}{\mathbb{R}}
\newcommand{\X}{\mathcal{X}}
\newcommand{\Y}{\mathcal{Y}}

\newcommand{\bbmone}{\mathbbm{1}}

\newcommand{\unfairS}{\Gamma\!_S}
\newcommand{\unfair}[1]{\Gamma\!_{#1}}

\newcommand{\methodname}{FLEA}
\newcommand{\method}{\methodname\xspace} 
\newcommand{\filtername}{\textsc{FilterSources}}
\newcommand{\filter}{\filtername\xspace}

\newcommand{\compas}{\texttt{COMPAS}\xspace}
\newcommand{\drugs}{\texttt{drugs}\xspace}
\newcommand{\adult}{\texttt{adult}\xspace}
\newcommand{\german}{\texttt{germancredit\xspace}}

\newcommand{\folktables}{\texttt{folktables}\xspace}

\newcommand{\halfquad}{\hspace{0.5em}}

\renewcommand{\thefootnote}{\arabic{footnote}} 
\newcommand{\myparagraph}[1]{\noindent\textbf{#1}\quad}

\newcommand{\jenc}{n\tau}
\newcommand{\jencwithoutn}{\tau}
\newcommand{\jenepsilon}{9\eta n + 7n\Delta\left(\frac{\delta}{6N}\right)}
\newcommand{\jenepsilonwithoutn}{9\eta + 7\Delta\left(\frac{\delta}{6N}\right)}
\newcommand{\jene}{\frac{\jenepsilonwithoutn}{\jencwithoutn-\left(\jenepsilonwithoutn\right)}}
\newcommand{\jenf}{\frac{\jenepsilonwithoutn}{1 - \jencwithoutn-\left(\jenepsilonwithoutn\right)}}

\title{\method: Provably Robust Fair Multisource \\Learning from Unreliable Training Data}
\author{%
  Eugenia Iofinova\footnote{}\\
  IST Austria %
  \And
  Nikola Konstantinov\footnotemark[2] \footnotemark[1]\\
  ETH AI Center and \\ETH Department of Computer Science \\
  \And
  Christoph H. Lampert\\
  IST Austria
}

\DeclareMathOperator{\RISK}{\mathcal{R}}%
\DeclareMathOperator{\HYPS}{\mathcal{H}}%

\newcommand{\prodspace}{\mathcal{X}\times A \times \mathcal{Y}}
\newcommand{\Bin}{\operatorname{Bin}}
\newcommand{\VC}{\operatorname{VC}}

\begin{document}

\maketitle

\begin{abstract} Fairness-aware learning aims at constructing 
classifiers that not only make accurate predictions, but also do not 
discriminate against specific groups. It is a fast-growing 
area of machine learning with far-reaching societal impact. 
However, existing fair learning methods are vulnerable to 
accidental or malicious artifacts in the training data, 
which can cause them to unknowingly produce unfair classifiers. 
In this work we address the problem of fair learning from 
unreliable training data in the robust multisource setting, where 
the available training data comes from multiple sources, a fraction of 
which might not be representative of the true data distribution. We 
introduce FLEA, a filtering-based algorithm that identifies and 
suppresses those data sources that would have a negative impact 
on fairness or accuracy if they were used for training. 
As such, FLEA is not a replacement of prior fairness-aware learning 
methods but rather an augmentation that makes any of them robust 
against unreliable training data.
We show the effectiveness of our approach by a diverse range of 
experiments on multiple datasets. 
Additionally, we prove formally that --given enough data-- \method 
protects the learner against corruptions as long as 
the fraction of affected data sources is less than half.
Our source code and documentation are available at \href{https://github.com/ISTAustria-CVML/FLEA}{\url{https://github.com/ISTAustria-CVML/FLEA}}.
\end{abstract}

\renewcommand{\thefootnote}{\fnsymbol{footnote}} 
\footnotetext[1]{Equal contribution}
\footnotetext[2]{Work performed partially while at IST Austria}
\setcounter{footnote}{0}
\renewcommand{\thefootnote}{\arabic{footnote}} 

\section{Introduction}
Machine learning systems have started to permeate many aspects 
of our everyday life, such as finance (\eg credit scoring), 
employment (\eg judging job applications) or even judiciary 
(\eg recidivism prediction).
In the wake of this trend, other aspects besides 
prediction accuracy become important to consider. 
One crucial aspect is \emph{(group) fairness}, which 
aims at preventing learned classifiers from acting in 
a discriminatory way.
To achieve this goal, fairness-aware learning methods adjust the classifier parameters in order to fulfill an appropriate measure of fairness. 
This strategy is highly successful, but only 
under idealized conditions of clean \iid-sampled data.
Unfortunately, \emph{fairness-aware learning methods 
are not robust against unintentional errors or 
intentional manipulations of the training data.} %

In this work, we address this problem in a setting where the training data is not one monolithic 
block, but rather a centralized collection of data obtained from multiple 
sources. This is, in fact, a common scenario. For instance, 
organizations that specialize in large-scale data mining, such 
as large hospital chains or political analysis firms, may
receive data that is collected separately from multiple physical
locations or data vendors. In such cases, it may be that not all of these data sources are completely
trustworthy, and so robustness concerns arise.

In order to achieve robustness to unreliable data in such contexts, we propose a new algorithm, \method (\underline{F}air 
\underline{LE}arning against \underline{A}dversaries). \method adds a filtering step on top of any 
standard fairness-aware learning algorithm and effectively identifies and suppresses data sources that could have 
a negative impact on the classifier fairness or accuracy. %
Thereby, \method acts as a procedure that guarantees robustness in the context of fair learning.

To accomplish this, we introduce a new dissimilarity measure, 
\emph{disparity}, that measures the maximum achievable difference 
in classifier fairness between two data sources. We use this measure as a filtering criterion, since it has the property of flagging changes in the data distribution that can be potentially harmful for the end-classifier fairness. We combine 
this with the existing \emph{discrepancy} measure, which %
plays an analogous role for the classifier accuracy, 
and the \emph{disbalance}, which measures changes to 
the group composition of the training data. 
We show both empirically and theoretically
that a combination of these three measures provides 
a sufficient criterion for detecting harmful data, 
as long as the fraction of harmful sources is
less than half.\footnote{The case where half or more sources are harmful is impossible to solve in general, see \eg~\citet{charikar2017learning}} 

While previous method for robust fairness-aware learning were
only able to protect against specific data issues, such 
as random label flips, \method ensures that even 
a worst-case adversary is unable to negatively affect 
the training process: either the changes to the data 
are minor and will not hurt learning, or they
are large enough so that the affected data sources are identified and removed.
Our theoretical analysis provides finite sample guarantees and certifies the ability of
\method to learn classifiers with optimal fairness and 
accuracy in the infinite sample size limit.
Our extensive experimental evaluation demonstrates 
\method's practical usefulness in suppressing the effect of corrupted data when learning fair models, even in cases where previous robust methods fail.

\begin{figure}\centering
\includegraphics[width=\textwidth]{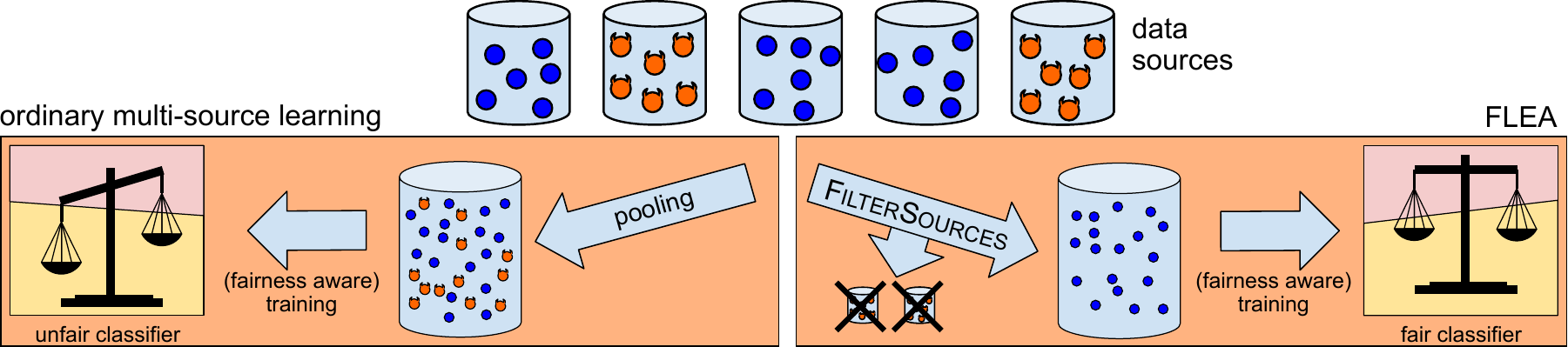}
\caption{Illustration of robust multisource learning with FLEA: (top) We are given multiple sources, some of which might contain noisy or manipulated data. 
(left) Ordinary multisource learning pools the data from all sources, which can cause the resulting classifier to be inaccurate and/or unfair, even
if fairness-aware training is employed. (right) FLEA filters the data before pooling, thereby suppressing likely corrupted sources. This allows fairness-aware training to succeed.}
\label{fig:FLEA-illustration}
\end{figure}

\section{Preliminaries and related work}

\subsection{Fair classification} 
Throughout this work, we adopt a standard classification 
setting in which the task is to predict a binary label 
$y\in\Y=\{0,1\}$ for any $x\in\X$. 
For a fixed data distribution $p(x,y)\in \mathcal{P}(\X\times\Y)$, 
the classic goal of learning is to find a prediction function 
$f:\X\to\Y$ with high accuracy, \ie small \emph{risk}, 
$\mathcal{R}_p(f)=\E_{p(x,y)}\bbmone\{ y\neq f(x)\}$,
where $\bbmone\{P\}=1$ if a predicate $P$ is true
and $\bbmone\{P\}=0$ otherwise. 

With the recent trend to consider not only the accuracy but 
also the \emph{fairness} of a classifier, a number of statistical 
measures have been proposed to formalize this notion. 
In this work, we focus on the most common and simplest one, 
\emph{demographic parity (DP)}~\citep{calders2009parity}. 
It postulates that the probability of a positive classifier 
decision should be equal for all subgroups of the population. 
Formally, we assume that each example $(x,y)$ also possesses a 
\emph{protected attribute}, $a\in\A$, which indicates 
its membership in a specific subgroup of the population. 
For example, $a$ could indicate \emph{race}, \emph{gender} or 
a \emph{disability}. 
For simplicity of exposition, we treat the protected attribute as 
binary-valued, but extensions to multi-valued attributes are 
straightforward by summing over all pairwise terms.
Note that $a$ might be a component or a function of $x$, in 
which case it is available at prediction time, or it might 
be contextual information, in which case it would only be 
available for the learning algorithm at training time, but 
not for the resulting classifier at prediction time.
We cover both aspects by treating $a$ as an additional random 
variable, and write the underlying joint data distribution as 
$p(x,y,a)$. 

For a classifier $f:\X\to\{0,1\}$, the 
\emph{demographic parity violation}, $\unfair{p}$, 
and the empirical counterpart, $\unfairS$ for a 
dataset $S\subset \X\times\Y\times\A$ are defined as~\citep{calders2009parity, dwork2012fairness}, 
\begin{align}
\unfair{p}(f) &= \E_{p(x|a=0)}f(x)-\E_{p(x|a=1)}f(x),
\qquad\qquad
\unfair{S}(f) = \frac{1}{n^{a=0}}\!\!\!\sum_{x\in S^{a=0}}\!\!\!\!f(x) -\frac{1}{n^{a=1}}\!\!\!\sum_{x\in S^{a=1}}\!\!\!\!f(x) %
 \label{eq:gamma}
 \end{align}
Negative values of the demographic parity violation indicate unfairness
against group $0$, while positive values indicate unfairness against
group $1$. 
Analogous quantities can be defined for related fairness measures, such 
as \emph{equality of opportunity} or \emph{equalized odds}~\citep{Hardt2016Equality,zafar2017fairness}.
A detailed description of these and many others choices can 
be found in~\citet{barocas-hardt-narayanan}.

\myparagraph{Fairness-aware learning}
In the last years, a plethora of algorithms have been developed that 
are able to learn classifiers that are not only accurate but also 
fair, see, for example~\citet{mehrabi2021survey} for an overview. %
They mostly rely on one or multiple of four core mechanisms.
\emph{Preprocessing methods}~\citep{kamiran2012data,calmon2017optimized,wang2019repairing,celis2020data} 
change the training data to remove a potential bias.
This is often simple and effective, but comes with
the danger of reduced accuracy, since the data 
distribution at training time will not reflect 
the distribution at prediction time anymore. 
\emph{Postprocessing methods}~\citep{Hardt2016Equality, woodworth17learning,Chzhen2020FairRegressionWasserstein} adjust 
the acceptance thresholds of a previously trained 
classifier for each protected group, so that 
the desired fairness criterion is met.
This is a simple, reliable and often effective 
method, but it requires the protected attribute 
to be available at prediction time. 
\emph{Penalty-based methods}~\citep{kamishima2012fairness, zemel2013learning, zafar17FairnessConstraints, donini2018empirical, Mandal2020EnsuringFairness, chuang2021fair} add a  %
regularizer or constraints to the learning 
objective that penalize or prevent parameter 
choices that lead to unfair decisions. 
\emph{Adversarial methods}~\citep{beutel2017data, wadsworth2018, zhang2018MitigatingBiasAdversarial, Lahoti2020AdversarialReweighting} 
train an adversary in parallel to the classifier 
 that tries to predict the protected 
attribute from the model outputs; if this cannot be done better than chance level, 
fairness is achieved.

Many other methods have been proposed, \eg 
based on distributionally robust optimization~\citep{rezaei2020fairness} or tailored to a specific family of classifiers or optimization procedures \citep{cho2020fair, tan2020learning}. %
They all share, however, the property that accurate 
information about the data distribution and the protected 
attribute is needed at training time. 

If the training data is not representative of the 
actual data distribution, \eg it is noisy, biased, 
or has been manipulated, then fairness-enforcing 
mechanisms fall short~\citep{kallus2019assessing,mehrabi2020exacerbating}.
Partial solutions have been proposed, \eg, when 
only the protected attribute or only the label is noisy~\citep{lamy2019noise, wang2020robustfairnessnoise, celis2021adversarial, celis2021fair,  mehrotra2021setselection, roh2021sample}.
However, as shown in~\citet{konstantinov2022theotherone}, 
full protection against malicious manipulations of 
the training data is provably impossible when learning from a
single dataset.

\subsection{Robust multisource learning}\label{subsec:multisource}
\myparagraph{Learning from multiple sources}
The multisource learning setting formalizes the 
increasingly frequent situation in which the 
training data is not collected as a single batch, 
but from multiple data sources~\citep{ting1997model, imagenet}.
For fairness-aware learning, this means we are 
given $N$ datasets, $S_1,\dots,S_N\subset \X\times\Y\times\A$.
Each $S_i=\{(x^{(i)}_1,y^{(i)}_1,a^{(i)}_1),\dots,(x^{(i)}_{n_i},y^{(i)}_{n_i},a^{(i)}_{n_i})\}$ %
contains \iid samples from a data distribution $p_i(x,y,a)$. 
Given these datasets, the learning algorithm has the 
goal of selecting a prediction function $f$ from a 
hypothesis class $\mathcal{H}$ that has as-small-as-possible 
risk (expected prediction error) and unfairness 
(\eg demographic parity violation) with respect 
to the unknown distribution at prediction time, $p(x,y,a)$, 
(also called \emph{target distribution}).
The classical setting of $p_1=p_2=\dots=p_N=p$, we 
call \emph{homogeneous multisource learning}.
Otherwise, we call the setting \emph{heterogeneous}. 

In a \emph{clean data scenario}, when all data 
distribution are the same or very similar to 
each other, then there is no drawback to simply 
merging all sources and training on the 
resulting large dataset.
However, merging all data is not the best strategy 
when some of the data sources are \emph{unrepresentative}, 
\ie their data distribution differs a lot 
from the target one.
Such data can occur accidentally, for instance 
due to biases in the data collection or 
annotation process. 
In some cases, such issues can be overcome by domain 
adaptation techniques~\citep{ben2010theory,crammer2008learning,natarajan2013learning}. %
Unrepresentative sources can also be the result 
of intentional manipulations, which are typically 
harder to detect and compensate for~\citep{feng2019learning,fowl2021adversarial}.
In fact, the datasets might not be samples from any
probability distribution in that case, but adversarially constructed.

\myparagraph{Robust multisource learning}
In this work we aim to cover as wide a range of possible 
problems with some of the data sources as possible. Therefore, 
we study the multisource learning problem in the presence of 
an \textit{adversary}\footnote{\emph{Adversary} is the 
common computer science term for a process whose aim it 
is to prevent a system from operating as intended. 
Our adversaries manipulate the training data and should 
not be confused with adversaries in \emph{adversarial machine learning}, 
such as \emph{adversarial examples}~\citep{goodfellow2015explaining}, or 
\emph{generative adversarial networks}~\citep{goodfellow2014generative}.}.
In this setting, the adversary observes an original collection of $N$ datasets, $\tilde S_1,\dots,\tilde S_N$,
where each $\tilde S_i$ contains \iid samples from a data 
distribution $p_i(x,y,a)$. %
Next, the adversary manipulates the data in an 
arbitrary (deterministic or randomized) way 
with the only restriction that for a fixed subset 
of indices, $G\subset \{1,\dots,N\}$, the data 
source remains unaffected. That is, $S_i = \tilde S_i$ 
for all $i\in G$, and $S_i$ is arbitrary for $i\not\in G$.
The subset $G$ is unknown to the learning algorithm, 
of course.
The adversary model places no restrictions on the 
corruptions, and thus subsumes many scenarios that
have to otherwise be studied in isolation. In particular,
both data-quality issues, such as sampling bias, 
data entry errors or label noise, as well as malicious 
manipulations, such as class erasure or data poisoning,
are covered as special cases.

Multisource learning with protection against potential 
manipulations is known as \emph{robust multisource learning}~\citep{erfani2017shared}. In order to detect harmful sources, a natural approach is to compare all pairs of datasets with an appropriate distance measure and then use the pairwise distances to filter out sources that are far from the others. Key to the success of such an approach is using the right definition of distance.
On the one hand one must be able to estimate the 
measure from finite sample sets in a statistically
efficient way.
Many common information-theoretic measures, such as \emph{Kullback-Leibler divergence}~\citep{kullback1951information},
\emph{total variation}~\citep{Tsybakov:1315296} or \emph{Wasserstein distance}~\citep{villani2009optimal}, 
do not fulfill this criterion. 
On the other hand, the measure must be sensitive enough 
such that if two sources appear similar then training 
on either of them must yield similar classifiers.
Classical two-sample tests, such as Student's 
$t$-test~\citep{student1908probable} or MMD~\citep{JMLR:v13:gretton12a}, 
fail to guarantee this.

In the context of multisource learning a measure that combines 
both useful properties is the \emph{(empirical) discrepancy distance}~\citep{kifer2004detecting,mohri2012new}.
For two datasets, $S_1$, $S_2$, and a hypothesis set $\H\subset\{h:\X\to\Y\}$, 
it measures the maximal amount by which their estimates of 
the classification accuracy can differ:
\begin{equation}
\disc(S_1,S_2) = \sup_{h\in\H} \Big| \mathcal{R}_{S_1}(h) - \mathcal{R}_{S_2}(h)\Big|, \label{eq:discrepancy}
\end{equation}
where $\mathcal{R}_{S}(h)=\frac{1}{|S|}\sum_{(x,y)\in S}\bbmone\{ y\neq h(x)\}$ 
is the \emph{empirical risk} of $h$ on $S$.
In~\citet{konstantinov2020sample} the discrepancy is used as 
a distance measure to identify and suppress data sources that
might harm the classifier's accuracy.
However, the associated algorithm is mostly of theoretical 
interest: it only suppresses those sources of which it is 
certain that they have been manipulated using thresholds 
that are derived from its generalization bound. 
As a consequence, it requires training sets that are too 
large to be practical. Similarly, \citet{jain2020general} provide 
an analysis of the learning-theoretic limitations of robust 
multisource learning. 
\citet{konstantinov2019robust} also use the discrepancy measure for detecting harmful data sources, but 
the proposed algorithm requires access to a reference set 
that is guaranteed to be free of data manipulations. 
In~\citet{qiao2018learning,chen2019efficiently,jain2020_optimalrobustlearning}
robust multisource learning is addressed using tools from robust
statistics, but only in the context of discrete density estimation. The problem of achieving robustness to noisy data annotators is also related (\cite{awasthi2017efficient, khetan2017learning}), but more restricted, as in our context we allow for arbitrary changes of the inputs and protected attributes, in addition to the labels.

All of the above works are tailored to the task of ensuring 
high accuracy of the learned classifiers or estimators, but 
they are not sensitive to issues of fairness. To our awareness, the only prior work that considers 
achieving fairness in a multisource learning setting and in the presence of data corruption is the one of \citet{li2021ditto}.
However, that paper focuses on \textit{personalized federated learning} and on a fairness objective tailored to federated learning, which 
postulates that models' performances should be relatively similar across edge devices. In contrast, we 
study a centralized setup, where privacy and communication issues are not present and where a \textit{single global model is trained}. In addition, we aim to ensure that this model
does not act discriminatory against members of protected subgroups, aligned with the classic notions of group fairness in supervised learning. %

\section{Fair multisource learning}
The goal of this work is to develop a 
method that allows fairness-aware learning, 
even if some of the available data sources 
are unrepresentative of the true training 
distribution. 
For this, we introduce \method, a filtering-based 
algorithm that identifies and suppresses those 
data sources that would negatively impact 
the fairness of the trained classifier. 
Its main innovation is the \emph{disparity measure} 
for comparing datasets in terms of their fairness estimates. %
\begin{definition}[Empirical Disparity]\label{def:disparity}
For two datasets $S_1,S_2\subset\X\times\Y\times\A$, 
their \emph{empirical disparity} with respect to a 
hypothesis class $\H$ is 
\begin{equation}
\disp(S_1,S_2) = \sup_{h\in\H} \Big| \unfair{S_1}(h) - \unfair{S_2}(h)\Big|. \label{eq:disparity}
\end{equation}
where $\unfairS:\H\to\R$ is an empirical (un)fairness measure, 
such as the \emph{demographic parity violation}~\eqref{eq:gamma}.\end{definition}

The \emph{disparity} measures the maximal amount by which the 
estimated fairness of a classifier in $\H$ can differ between 
using $S_1$ or $S_2$ as the basis of the estimate. 
A small disparity value implies that if we construct a 
classifier that is fair with respect to $S_1$, then it 
will also be fair with respect to $S_2$. %

Definition~\ref{def:disparity} is inspired by the 
\emph{empirical discrepancy}~\eqref{eq:discrepancy}. 
Just as low discrepancy implies that a classifier learned 
on one dataset will have comparable accuracy as one learned 
on the other, low disparity implies that the two classifiers 
will have comparable fairness. %
\method makes use of the discrepancy as well as the disparity, 
because ensuring fairness alone does not suffice (\eg a 
constant classifier is perfectly fair). 
As a third relevant quantity we introduce the \emph{(empirical) disbalance}.
\begin{equation}
\disb(S_1,S_2) = \Big| 
\frac{|S^{a=1}_1|}{|S_1|} - \frac{|S^{a=1}_2|}{|S_2|}\Big|. \label{eq:disbalance}
\end{equation}
The disbalance compares the relative sizes of the protected groups 
of two datasets. Its inclusion is a technical requirement %
to be able to also formally prove that demographic 
parity fairness remains unaffected by corruptions.

In combination, disparity, discrepancy, and disbalance 
form an effective criterion for detecting dataset manipulations.
This is most apparent in the homogeneous setting: if two datasets 
of sufficient size are sampled \iid from distributions close to 
the target one, then by the law of large numbers we can expect 
all three measures to be small. 
If one of the datasets is sampled like this (called 
\emph{clean} from now on) but the other is manipulated, 
then there are two possibilities. 
It is still possible that all three values are small. 
In this case, equations \eqref{eq:discrepancy}--\eqref{eq:disbalance} 
ensure that neither accuracy 
nor fairness would be negatively affected, and we 
call such manipulations \emph{benign}.
If at least one of the values is large, training on 
such a manipulated datasource could have %
undesirable consequences. Such 
manipulations we will call \emph{malignant}. 
Finally, when comparing two manipulated datasets, 
discrepancy, disparity, and disbalance can each
have arbitrary values. %

In the heterogeneous setting, a path of similar reasoning 
applies, though the measures for clean sources will not 
approach exactly zero due to the difference in their data 
distributions.

\subsection{\method: Fair learning against adversaries}\label{subsec:method}

We now introduce the \method algorithm, which is able to 
learn fair classifiers even if up to half of the datasets are noisy, 
biased or have been manipulated. 
Similar to classic outlier rejection techniques~\citep{barnett1984outliers}
and statistical two-sample tests~\citep{corder2014nonparametric}, 
the main algorithm (Algorithm~\ref{alg:method}) takes a filtering 
approach. 
Given the available data sources and additional parameters, 
it calls a subroutine that identifies a subset of clean or 
benign sources, merges the training data from these, and 
trains a (presumably fairness-aware) learning algorithm 
on the resulting dataset.

\begin{wrapfigure}[20]{r}{0.53\textwidth}
\vspace{-1.8\baselineskip}
\begin{minipage}{0.53\textwidth}
\begin{algorithm}[H]
\caption{\methodname}\label{alg:method}
\begin{algorithmic}[1]
\REQUIRE datasets $S_1,\dots,S_N$
\REQUIRE quantile parameter $\beta$ %
\REQUIRE (fairness-aware) learning algorithm $\mathcal{L}$
\STATE $I\leftarrow \textsc{FilterSources}(S_1,\dots,S_N; \beta)$
\STATE $S \leftarrow \bigcup_{i\in I}S_i$
\STATE $f \leftarrow \mathcal{L}(S)$
\ENSURE trained model $f:\X\to\Y$
\end{algorithmic}
\end{algorithm}
\end{minipage}
\vskip-.5\baselineskip
\begin{minipage}{0.53\textwidth}
\begin{algorithm}[H]
\caption*{\textbf{Subroutine} \textsc{FilterSources}}%
\begin{algorithmic}[1]
\REQUIRE $S_1,\dots,S_N$; $\beta$ %
\FOR{$i=1,\dots,N$}
\FOR{$j=1,\dots,N$}
\STATE\!\!$D_{i,j} \leftarrow \disc(S_i,S_j)+\disp(S_i,S_j)+\disb(S_i,S_j)$
\ENDFOR
\STATE $q_i \leftarrow \operatorname{\beta-quantile}(D_{i,1},\dots,D_{i,N})$
\ENDFOR
\STATE $I \leftarrow \big\{ i : q_i \leq \operatorname{\beta-quantile}(q_1,\dots,q_N) \big\}$
\ENSURE index set $I$
\end{algorithmic}
\end{algorithm}
\end{minipage}
\end{wrapfigure}

\method's crucial component is the filtering subroutine. 
This estimates the pairwise disparity, discrepancy and
disbalance between all pairs of data sources and combines
them, by summing,\footnote{Other 
aggregation methods would be possible, as long as they ensure to 
preserve large values, such as the maximum. This would yield 
similar theoretical guarantees, but we did not find it to 
perform better in practice.} into a matrix of dissimilarity scores (short: $D$-scores).
As discussed above, large values indicate that at least 
one of the two compared sources must be \emph{malignant}. 
It is not a priori clear, though, how to use this information.
On the one hand, we do not know which of the two datasets 
is malignant or if both are. 

On the other hand, 
malignant sources can also occur in pairs with small $D$-score, 
when both datasets were manipulated in similar ways. %
Finally, even the $D$-scores between two clean or benign sources 
will have non-zero values, which depend on a number of factors, 
in particular the data distributions and the hypothesis class.

\method overcomes this problem by using tools 
from robust statistics. 
For any dataset $S_i$, it computes a value $q_i$ 
(called $q$-value) as the $\beta$-quantile of the
$D$-scores to all other datasets, where $\beta$ 
is a hyperparameter we discuss below.
It then computes the $\beta$-quantile of all 
such values and selects those datasets with 
$q$-values up to this threshold.

To see that this procedure has the desired 
effect of filtering out malignant datasets,
we first look at the case in which the sources
are homogeneous and $\beta=\frac{K}{N}$, 
where $K = |G| >\frac{N}{2}$ is the number of 
clean data sources.

For any clean dataset $S_i$, by assumption 
there are at least $K-1$ other clean sources 
with which it is compared. 
We can expect the $D$-scores of these pairs 
are small, and, of course, that $D_{ii}=0$. %
Because $\beta=\frac{K}{N}$, the $\beta$- quantile, 
$q_i$, is simply the $K$th-smallest of $S_i$'s 
$D$-scores. 
Consequently, $q_i$ will be at least as small 
as the result of comparing two clean sources.
For benign sources, the same reasoning applies, 
since their $D$-scores are indistinguishable 
from clean ones. 
For a malignant $S_i$, at least $K$ of the 
$D$-scores will be large, namely the ones 
where $S_i$ is compared to a clean source. 
Hence, there can be at most $N-K$ small $D$-scores 
for $S_i$. Because $\beta N=K$ and $K>N-K$, 
the $\beta$-quantile $q_i$ will be at least as 
large as comparing a clean dataset to a malignant 
one.

Choosing those sources that fall into the 
$\beta$-quantile of the $q_i$ values means 
selecting the $K$ sources of smallest $q_i$ 
value. 
By the above argument, these will either 
be not manipulated at all, or only in a 
way that does not have a negative effect 
on either the fairness or the accuracy of the training 
process.
In practice, the regimes of \emph{large} 
and \emph{small} $D$-scores can overlap
due to noise in the sampling process,
and the perfect filtering property will 
only hold approximately. %
We later discuss a 
generalization bound that makes this reasoning 
rigorous. %

Revisiting the above arguments one sees 
that the guarantees on the $q_i$ follow 
also for any $\beta>\frac{N-K}{N}$, so 
in particular for $\beta\geq \frac12$. 
To obtain the guarantee on the selected 
sources, $\beta\leq\frac{K}{N}$ suffices.
Therefore, even if the exact value of $K$ is 
unknown in practice, setting $\beta=\frac12+\frac{1}{N}$ 
for even $N$ and $\beta=\frac{1}{2}+\frac{1}{2N}$ 
for odd $N$ will always be working choices. 
These are also the values we use in our experiments. 

In the heterogeneous situation, the $D$-scores 
between clean sources might not tend to zero
for large $n$ anymore. However, they will approach 
the true discrepancy, disparity and disbalance 
values between the sources' distributions. From this, 
one can obtain a guarantee that the selected sources 
are not more dissimilar from each other than the clean 
sources are, which is the best one can hope for 
in the heterogeneous setting.

\subsection{Implementation}\label{subsec:implementation}
\method is straightforward to implement, with only
the {discrepancy} and {disparity} estimates in 
the \textsc{FilterSources} routine requiring some 
consideration.
Naively, these would require optimizing combinatorial 
functions (the differences of fraction of errors or 
positive decisions) over all functions in the hypothesis 
class. This task is at least as hard as the problem of
separating two point sets by a hyperplane, which is 
known to be NP-hard~\citep{marcotte1992novel} and even 
difficult to approximate under any real-world conditions.
Instead, %
we exploit the structure 
of the optimization problems to derive tractable approximations. 

We describe the procedure here and 
provide pseudocode in Appendix~\ref{sec:detailedalgo}.
For the discrepancy~\eqref{eq:discrepancy} such a method was 
originally proposed in the domain adaptation literature~\citep{ben2010theory}: 
finding the hypothesis with maximal accuracy difference 
between two datasets is equivalent to training a binary 
classifier on their union with the labels of one of the 
datasets flipped. 

For the disparity~\eqref{eq:disparity}, we propose 
an analogous route. 
Intuitively, the optimization step requires finding 
a hypothesis that is as unfair as possible on $S_1$ 
(\ie maximizes $\unfair{S_1}$) while being as unfair 
as possible in the opposite direction on $S_2$ 
(\ie minimizes $\unfair{S_2}$), or vice versa.
From Equation~\eqref{eq:gamma} one sees that a 
hypothesis $f$ is maximally positively unfair 
if it outputs $f(x)=1$ on $S_1^{a=1}$ and %
$f(x)=0$ on $S_1^{a=0}$, and maximally negatively
unfair if it has the opposite outputs. %
Consequently, to estimate the disparity, we can
use a classifier trained to 
predict $f(x)=a$ on $S_1$ and $f(x)=1-a$ on $S_2$.
To give both protected groups equal importance, 
as the definition requires, we use per-sample 
weights that are inversely proportional to the 
group sizes.

\subsection{Theoretical analysis}\label{subsec:theory}
The informal justification of \method can be made 
precise in the form of a generalization bound. 
In this section we present our theoretical guarantees for \method. We begin by stating formally the assumptions we make on the data generating process, both for the heterogeneous and the homogeneous cases discussed above. We then state our main theoretical result, which certifies the performance of \method in the both the homogeneous and the more general heterogeneous case. Finally, we briefly outline the main proof steps.
The full proofs can be found in Appendix~\ref{sec:proof}.

\subsubsection{Assumptions and formal adversarial model}
\label{sec:theory_assumptions}
First we present our formal set of assumptions, directly in the general setting of \emph{heterogeneous} data sources. A crucial parameter here setup is $\eta$, which denotes the amount of variability between the clean sources' distributions. The case of $\eta = 0$ recovers the \textit{homogeneous} setup.

We assume the following data generation model, similar to the one of \citet{qiao2018learning}. 
By $p(x,y,a)$ we denote the target distribution. %
It is unknown to the learning algorithm, though potentially known to 
the adversary.
Initially, there are $N$ datasets $\tilde S_1,\dots,\tilde S_N$, with the $i$-th set of samples being drawn \iid from a distribution $p_i(x,y,a)$. 
These distributions might differ from the target distribution $p$ by at most $\eta$ in terms 
of total variation both with respect to the overall distributions as well as the conditional 
distributions with respect to $a$.
Formally, we assume the following conditions for $i=1,\dots,N$:
\begin{align}
\label{eqn:closeness_assumption}
\TV(p_i(x,y,a), p(x,y,a)) \leq \eta,\quad\text{and}\quad \max_{z\in\mathcal{A}}\big\{\TV(p_i(x,y|a=z), p(x,y|a=z))\big\} \leq \eta,
\end{align}
where $\TV$ is the \emph{total variation distance} between probability distributions~\cite{halmos2013measure}.

Once the clean datasets $\tilde S_1,\dots,\tilde S_N$ are sampled, an \emph{adversary}
operates on them.
This results in new datasets, $S_1,\dots,S_N$, 
which the learning algorithm receives as input. 
The adversary is an arbitrary (deterministic or 
randomized) function  $\mathcal{F}:\prod_{i=1}^N\left(\mathcal{X}\times\mathcal{Y}\times\mathcal{A}\right)^{n_i} \rightarrow \prod_{i=1}^N\left(\mathcal{X}\times\mathcal{Y}\times\mathcal{A}\right)^{n_i}$,
with the only restriction that for a fixed subset 
of indices, $G\subset \{1,\dots,N\}$, the data remains unchanged. That is, $S_i = \tilde S_i$ 
for all $i\in G$, and $S_i$ is arbitrary for $i\not\in G$. For simplicity, we refer to a dataset $S_i$ or a source $i\in [N]$ as clean if $i\in G$.

\subsubsection{Theoretical guarantees on \method}
\label{sec:theory_results}
We are now ready to state our theoretical guarantee on \method. 
For simplicity of notation, we present the case where all 
sources have the same number of samples. Results for general 
sample sizes can be obtain in an analogous way.
We first state the guarantees for the homogeneous situations, 
which we obtain in fact as a corollary for $\eta=0$ of 
the general theorem later in this section.

\begin{theorem}[Homogeneous setting]\label{cor:method}

Assume that $\mathcal{H}$ has a finite VC-dimension $d \geq 1$. 
Let $p$ be an arbitrary target data distribution and without 
loss of generality let $\tau = p(a=0) \in \left(0, 0.5\right]$. 
Let $S_1,\dots,S_N$ be $N$ datasets, each consisting of $n$ samples, 
out of which $K>\frac{N}{2}$ are sampled $\iid$ from the 
distribution $p$. 
For $\frac{1}{2} < \beta \leq \frac{K}{N}$ and 
$I=\textsc{FilterSources}(S_1,\dots,S_N;\beta)$ set $S=\bigcup_{i\in I} S_i$. 
Let $\delta>0$. Then there exists a constant $C = C(\delta, \tau, d, N, \eta)$, such that for any $n \geq C$, the following inequalities hold with 
probability at least $1-\delta$ uniformly over all $f\in\H$ and against any adversary:
\begin{align}
\big|\unfairS(f) - \unfair{p}(f)\big| &\leq \widetilde{\mathcal{O}}\left(\sqrt{\frac{1}{n}}\right),
\qquad\qquad
\big|\mathcal{R}_S(f) - \mathcal{R}_p(f)\big| \leq \widetilde{\mathcal{O}}\left(\sqrt{\frac{1}{n}}\right),
\label{eqn:theory_results_homogeneous}
\end{align}
where $\widetilde{\mathcal{O}}$ indicates Landau's big-O notation for function growth up to logarithmic factors~\citep{cormen2009introduction}.
\end{theorem}

\paragraph{Discussion}
To analyze the statement, we observe that 
Equation~\eqref{eqn:theory_results_homogeneous} 
ensures that for large enough training sets the 
filtered training data $S$ becomes an arbitrarily 
good representative of the true underlying data distribution 
with respect to the classification accuracy as well as 
the fairness.
Moreover, the approximation holds uniformly across all hypotheses in 
the class.
We note that a similar generalization bound
  for accuracy in the homogeneous setting is given
  in \cite{konstantinov2020sample}.

This uniform convergence property is similar to the 
  classic concentration results from learning theory for learning with clean data
  \citep{shwartz2014understanding, woodworth17learning} and essentially ensures 
  that using the data $S$ is \emph{safe} for the purposes of fairness-aware learning. 
  Indeed, since the empirical risk and fairness deviation on the filtered data $S$ 
  are good estimates of the true population measures, any algorithm that uses the data 
  $S$ to learn a hypothesis with good empirical fairness and accuracy will also perform 
  well at prediction time.

Note that despite the intuitive conclusion, the 
result from Theorem \ref{cor:method} is highly non-trivial, due to the presence of data corruption. 
For example, in the case of learning from a single datasource in which a constant fraction of 
the data can be manipulated, an analogous theorem is provably 
impossible~\citep{kearns1993learning,konstantinov2022theotherone}.
This observation also implies that no learning 
algorithm can guarantee accurate and fair learning if it is given
access to the training data only after all sources have been merged.

For the general situation ($\eta\geq 0$), we obtain the 
following guarantees:

\setcounter{theorem}{0}
\begin{theorem}[Heterogeneous setting]\label{thm:method}
Assume that $\mathcal{H}$ has a finite VC-dimension $d \geq 1$. Let $p$ be an arbitrary target data distribution and without loss of generality let $\tau = p(a=0) \in \left(0, 0.5\right]$. Let $S_1,\dots,S_N$ be $N$ datasets, each consisting of $n$ samples, out of which $K>\frac{N}{2}$ 
are sampled $\iid$ from data distributions $p_i$ that are $\eta$-close to the distribution $p$ in the sense of Section \ref{sec:theory_assumptions}. Assume that $18\eta < \tau$.
For $\frac{1}{2} < \beta \leq \frac{K}{N}$ and 
$I=\textsc{FilterSources}(S_1,\dots,S_N;\beta)$ set $S=\bigcup_{i\in I} S_i$. 
Let $\delta>0$. Then there exists a constant $C = C(\delta, \tau, d, N, \eta)$, such that for any $n \geq C$, the following inequalities hold with 
probability at least $1-\delta$ uniformly over all $f\in\H$ and against any adversary:
\begin{align}
\big|\unfairS(f) - \unfair{p}(f)\big| &\leq \mathcal{O}(\eta) + \widetilde{\mathcal{O}}\Big(\sqrt{\frac{1}{n}}\Big),
\qquad\qquad
\big|\mathcal{R}_S(f) - \mathcal{R}_p(f)\big| \leq \mathcal{O}(\eta) + \widetilde{\mathcal{O}}\Big(\sqrt{\frac{1}{n}}\Big).
\label{eqn:theory_results}
\end{align}
\end{theorem}

\paragraph{Discussion} In contrast to the homogeneous situation, an 
additional factor linear in $\eta$ enters the right hand 
side of the bound. 
As discussed in Section~\ref{subsec:method}, we believe that
such a factor will be unavoidable in the heterogeneous case: 
$\eta$ is a measure of the dissimilarity between the clean sources 
and the target distribution. Therefore, even without data corruptions, 
the accuracy of a learned classifier for the unknown target 
distribution $p$ will be limited by how close that is to the 
training distributions~\citep{bartlett1992learning,hanneke2020no}.
That the increase in risk is of order $\eta$ can be seen from a 
simple binary classification example: let $p(x,y)=\bbmone\{ x\geq 0.5\}$
and $p_1(x,y)=\bbmone\{ x\geq 0.5+\eta\}$, such that 
$\TV(p, p_1)=\eta$.
Then, the optimal classifier learned with respect to $p_1$ 
has expected error $\eta$ with respect to $p$.

In cases when $\eta$ is small, the stated result still certifies that 
the empirical risk and fairness deviation on the data $S$ are good 
estimates of the underlying population values. Therefore, the 
discussion from the homogeneous case applies here as well, meaning that 
the data $S$ is ``safe'' to train on for the purposes of fairness-aware 
learning.

\myparagraph{Proof sketch} The proof consists of three steps. 
First, we characterize a set of values into which the empirical 
risks and empirical deviation measures of the clean data 
sources fall with probability at least $1-\delta$. 
Then we show that because the clean datasets cluster in such a way, 
any individual dataset that is accepted by the \filter 
algorithm provides good empirical estimates of the true risk 
and the true unfairness measure.
Finally, we show that the same holds for the union of these 
sets, $S$, which implies the inequalities in the theorem.
For the risk, the last step is a straightforward
consequence of the second. For the fairness, which 
is not simply an expectation or average over per-sample 
contributions, a more careful derivation is needed that 
crucially uses the disbalance measure as well.
For details of the steps, please see Appendix~\ref{sec:proof}.

\subsection{Computational complexity of \method}

In order to apply \method, we must train two classifiers 
(one to estimate $\disc$ and one to estimate $\disp$) 
for every pair of sources in the dataset. Assuming that the maximum number
of points in every data source (after adversarial perturbation) is $n$,
the complexity of training all of these classifiers is therefore bounded above by
$\mathcal{O}(N^2 F(2n))$, where $N$ is the number of data sources, and
$F(t)$ is the computational complexity of running the chosen method of learning
a classifier on a data set of size $t$. Then, all but $K$ sources are
filtered out and the data in the rest is combined, resulting in a total computational
complexity of $\mathcal{O}(N^2 F(2n) + F(Kn))$. Since $\frac{N}{2} < K \leq N$,
which term dominates depends 
on the complexity of the learning algorithm. In the case that $F$ is subquadratic, 
the total complexity is dominated by the first term; if $F$ is quadratic, 
by neither term, and if $F$ is of higher complexity then the combined training 
dominates. Please see Appendix sections~\ref{subsec:training_objectives} and
\ref{subsec:computing_resources} for specific details and running time of 
our experimental setup.

\section{Experiments}
\method's claim is that it allows learning classifiers that are 
fair even in the presence of perturbations in the training data. 
Due to its filtering approach it can be used in combination 
with any existing learning method. For our experiments, 
we run it in combination with four fairness-aware learning 
methods as well as one fairness-unaware one 
against a variety of adversaries on five established fair classification datasets. %
We benchmark our method against the corresponding base learning algorithms without pre-filtering, as well as against four robust learning baselines.%

\subsection{Experimental setup}\label{subsec:experimentalsetup}
We report experiments in two setups: for \emph{homogeneous} and 
\emph{heterogeneous} data sources. 

\myparagraph{Datasets}
For the homogeneous setup we use four standard benchmark 
datasets from the fair classification literature: \compas~\citep{Angwin_compas} (6171 examples),  %
\adult (48841), \german (1000) and \drugs (1885)~\citep{UCI}.
To obtain multiple identically distributed sources, we 
randomly split each training set into $N\in\{3,5,7,9,11\}$ 
equal-sized parts, out of which the adversary 
can manipulate $\lfloor\frac{N-1}{2}\rfloor$.
For the heterogeneous case we use the 2018 US census 
data of the \folktables~dataset\citep{ding2021retiring}. 
We form 51 similarly but not identically distributed 
data sources by using up to 10000 examples from each 
of the US-states. Out of these $5$, $10$, $15$, $20$ 
or $25$ can be manipulated.
Details about the data preprocessing and feature extraction 
steps can be found in the supplemental material. 

In all cases, we use \emph{gender} as the exemplary 
protected attribute, because it is present in all
feature sets. 
We train linear classifiers by logistic regression 
without regularization, using 80\% of the data for 
training and the remaining 20\% for evaluation.
All experiments are repeated ten times with different 
train-test splits and random seeds. We measure 
the mean and standard deviation of the accuracy 
and the fairness of the learned classifiers, where
we compute fairness as $1-\Gamma_S$, where $\Gamma_S$
is the demographic parity violation on the test set.

\myparagraph{Fairness-Aware Learners}
We use \method in combination with four fairness-aware learning 
methods that have found wide adoption in research and practice.
In all cases, we use \emph{logistic regression} as the underlying 
classification model.

\noindent$\bullet$\halfquad\emph{Fairness regularization}~\citep{kamishima2012fairness} 
learns a fair classifier by minimizing a linear combination of the 
classification loss and the empirical unfairness measure $\unfairS$, 
where for numeric stability, in the latter the binary-valued 
classifier decisions $f(x)$ are replaced by the real-valued 
confidences $p(f(x)=1|x)$. 

\noindent$\bullet$\halfquad\emph{Data preprocessing}~\citep{kamiran2012data} 
modifies the training data to remove potential biases. 
Specifically, it creates a new dataset by \emph{uniform resampling} 
(with repetition) from the original dataset, such that the the fractions 
of positive and negative labels are the same for each protected group. 
On the resulting unbiased dataset it trains an ordinary fairness-unaware
classifier.

\noindent$\bullet$\halfquad\emph{Score postprocessing}~\citep{Hardt2016Equality} 
first learns an ordinary (fairness-unaware) classifier 
on the available data. Afterwards, it determines which decision 
thresholds for each protected groups achieve (approximate) 
demographic parity on the training set, finally picking 
the fair thresholds with highest training accuracy.

\noindent$\bullet$\halfquad\emph{Adversarial fairness}~\citep{wadsworth2018} 
learns by minimizing a weighted difference between two terms. 
One is the loss of the actual classifier; 
the other is the loss of a classifier 
that tries to predict the protected attribute from the 
real-valued outputs of the main classifier. 

For completeness, we also include plain logistic 
regression as a \emph{fairness-unaware learner}. 
The supplemental material details the learners' 
implementations and parameters.

\myparagraph{Adversaries}\label{subsec:adversaries}
In a real-world setting, one does not know what kind of 
data quality issues will occur. 
Therefore, we test the baselines and \method for a range 
of adversaries that reflect potentially unintentional errors 
as well as intentional manipulations. 

\noindent$\bullet$\halfquad\emph{flip protected (FP), flip label (FL), flip both (FB)}: 
the adversary flips the value of protected attribute, of the label, 
or both, in all sources it can manipulate.

\noindent$\bullet$\halfquad\emph{shuffle protected (SP)}: the adversary 
shuffles the protected attribute entry in each effected batch.

\noindent$\bullet$\halfquad\emph{overwrite protected (OP), overwrite label (OL)}: 
the adversary overwrites the protected attribute of each sample 
in the affected batch by its label, or vice versa. 

\noindent$\bullet$\halfquad\emph{resample protected (RP)}: the adversary samples new 
batches of data in the following ways: all original samples of protected 
group $a=0$ with labels $y=1$ are replaced by data samples from other 
sources which also have $a=0$ but $y=0$. 
Analogously, all samples of group $a=1$ with labels $y=0$ are replaced 
by data samples from other sources with $a=1$ and $y=1$. 

\noindent$\bullet$\halfquad\emph{random anchor (RA0/RA1)}: 
these adversaries follow the protocol introduced in \citet{mehrabi2020exacerbating}.
After picking \emph{anchor} points from each protected group 
they create poisoned datasets consisting of examples that 
lie close to the anchors but have opposite label to them. 
The difference between RA0 and RA1 lies in which combinations
of label and protected attribute are encouraged or discouraged.

\noindent$\bullet$\halfquad\emph{random (RND)}: the adversary randomly 
picks one of the strategies above for each source.

\noindent$\bullet$\halfquad\emph{identity (ID)}: the adversary makes no changes to the data.

\medskip
We include ID to certify that \method does not unnecessarily damage the learning process in the case when the training data is actually clean.
The other adversaries either weaken the correlations between 
the protected attribute and the target data, thereby masking 
a potential existing bias in the data, or they strengthen the 
correlation between the protected attribute and the target 
label, thereby increasing the chance that the learned 
classifier will use the protected attribute as a basis for its 
decisions. 
In both cases, the dataset statistics at training time will differ 
from the situation at test time, and the efficacy of a potential 
mechanisms to ensure fairness at training time can be expected 
to suffer. 
For a more detailed discussion of the adversaries' effects, please see the supplemental material.

\myparagraph{Baselines}\label{subsec:baselines}
To the best of our knowledge, \method is the only existing 
method to tackle fair learning under arbitrary data manipulations. 
To nevertheless put our results into context, we 
compare it to four baselines: 1) a \emph{robust 
ensemble} (similar to~\citet{smith2018robustness}), which learns 
separate classifiers on each datasource and then combines their 
decisions by a majority vote. 
2) A \emph{distributionally robust optimization (DRO)} approach 
as proposed in~\citet{wang2020robustfairnessnoise} to 
address noisy protected attributes. 
3) Hierarchical \emph{tilted empirical risk minimization}
(hTERM)~\citet{li2021tilted}, which aims at enforcing \emph{robustness} 
by a \emph{softmin} across per-sources losses, which themselves express
a form of \emph{fairness} by a \emph{softmax}-loss across protected groups.
4) The \emph{filtering approach} 
of \citet{konstantinov2020sample} which uses discrepancy to identify 
manipulated sources but does not specifically aim to preserve fairness.
More details on these can be found in the supplemental material.
Further candidates could be~\citet{roh20FRtrain,konstantinov2019robust},
but these are not applicable in our setting, as they require access
to guaranteed clean validation data.

\begin{table*}[t]\centering\footnotesize
\caption{Result of \method and baselines for robust fairness-aware multisource learning with homogeneous and heterogeneous data sources. 
Reported accuracy and fairness values are the minimal (worst-case) ones across all tested data manipulations in the respective settings. 
See main text for an explanation of the methods and details of the experimental setup.}
\label{tab:results}
\begin{subfigure}{\textwidth}\centering
\caption{homogeneous: \adult, \compas, \drugs and \german\ datasets with 
$5$ sources of which $2$ are unreliable.}\label{tab:results-homogeneous}
\begin{tabular}{l|ll|ll|ll|ll}
& \multicolumn{2}{c}{\adult} & \multicolumn{2}{c}{\compas} & \multicolumn{2}{c}{\drugs} & \multicolumn{2}{c}{\german}
\\
\textbf{method} & accuracy & fairness & accuracy & fairness & accuracy & fairness & accuracy & fairness 
\\\hline
naive	         & $66.2_{\pm 1.1}$ & $77.6_{\pm 1.2}$ 
                 & $63.1_{\pm 1.8}$ & $78.9_{\pm 2.3}$
                 & $60.0_{\pm 2.5}$ & $72.3_{\pm 3.1}$ 
                 & $58.0_{\pm 4.0}$ & $78.7_{\pm 5.3}$
\\\hline
robust ensemble  & $69.9_{\pm 0.4}$& $90.9_{\pm 1.6}$  
                 & $64.9_{\pm 1.1}$& $87.1_{\pm 2.6}$
                 & $61.0_{\pm 2.1}$& $66.8_{\pm 4.5}$
                 & $61.9_{\pm 2.9}$& $62.3_{\pm 7.6}$
\\
DRO \nolink{\citep{wang2020robustfairnessnoise}}& $52.8_{\pm 0.3}$ & $15.4_{\pm 1.3}$ 
                                                & $53.7_{\pm 1.4}$ & $57.1_{\pm 23.9}$ 
                                                & $55.2_{\pm 2.5}$ & $48.5_{\pm 27.0}$ 
                                                & $34.4_{\pm 6.0}$ & $81.1_{\pm 12.5}$
\\
hTERM \nolink{\citep{li2021tilted}}& $66.8_{\pm 0.9}$ & $50.7_{\pm 1.8}$  
                                  & $52.9_{\pm 2.5}$ & $29.3_{\pm 12.9}$ 
                                  & $54.6_{\pm 3.0}$ & $40.7_{\pm 9.1}$
                                  & $41.2_{\pm 4.0}$ & $25.2_{\pm 9.4}$    
\\
\nolink{\citep{konstantinov2020sample}}& $69.3_{\pm 0.4}$ & $77.6_{\pm 1.2}$ 
                                       & $63.1_{\pm 1.8}$ & $78.9_{\pm 2.3}$ 
                                       & $60.0_{\pm 2.5}$ & $72.3_{\pm 3.1}$ 
                                       & $58.0_{\pm 4.0}$ & $78.7_{\pm 5.3}$
\\\hline
FLEA (proposed)             & $70.2_{\pm 0.4}$ & $97.9_{\pm 1.1}$
                            & $65.9_{\pm 1.0}$ & $94.5_{\pm 3.0}$ 
                            & $64.3_{\pm 1.4}$ & $92.6_{\pm 4.2}$ 
                            & $65.9_{\pm 3.0}$ & $93.4_{\pm 3.9}$
\\\hline
oracle	        & $70.3_{\pm 0.4}$ & $98.2_{\pm 1.0}$ 
                & $66.2_{\pm 1.1}$ & $96.2_{\pm 1.3}$ 
                & $64.4_{\pm 1.5}$ & $93.6_{\pm 3.3}$ 
                & $67.3_{\pm 3.0}$ & $94.4_{\pm 4.0}$
\end{tabular}
\end{subfigure}
\vskip\baselineskip
\begin{subfigure}{\textwidth}\centering\footnotesize\setlength{\tabcolsep}{2pt}
\caption{heterogeneous: \folktables dataset with $N=51$ sources 
of which $N-K\in\{5,10,15,20,25\}$ are unreliable.}\label{tab:results-heterogeneous}
\begin{tabular}{l|ll|ll|ll|ll|ll}
& \multicolumn{2}{c}{$N-K=5$} & \multicolumn{2}{c}{$N-K=10$} 
& \multicolumn{2}{c}{$N-K=15$} & \multicolumn{2}{c}{$N-K=20$}
& \multicolumn{2}{c}{$N-K=25$}
\\
\textbf{method} & accuracy & fairness & accuracy & fairness & accuracy & fairness & accuracy & fairness & accuracy & fairness 
\\\hline	
naive	            & $74.4_{\pm 0.2}$ & $93.4_{\pm 0.8}$ 
                    & $73.7_{\pm 0.2}$ & $87.0_{\pm 0.8}$ 
                    & $72.9_{\pm 0.5}$ & $80.1_{\pm 0.9}$ 
                    & $71.2_{\pm 0.8}$ & $73.4_{\pm 0.6}$ 
                    & $58.2_{\pm 6.2}$ & $73.9_{\pm 1.0}$
\\\hline
robust ensemble     & $74.9_{\pm 0.2}$ & $97.1_{\pm 0.3}$
                    & $74.3_{\pm 0.2}$ & $93.8_{\pm 0.4}$
                    & $73.5_{\pm 0.3}$ & $89.1_{\pm 0.5}$
                    & $71.9_{\pm 0.3}$ & $81.7_{\pm 0.7}$
                    & $65.8_{\pm 1.1}$ & $60.4_{\pm 2.2}$
\\
DRO \nolink{\citep{wang2020robustfairnessnoise}}& $65.2_{\pm 0.8}$ & $96.0_{\pm 0.7}$
                                                & $68.1_{\pm 1.5}$ & $95.2_{\pm 0.7}$ 
                                                & $66.2_{\pm 0.9}$ & $85.8_{\pm 2.6}$
                                                & $66.1_{\pm 1.3}$ & $77.4_{\pm 12.2}$
                                                & $58.1_{\pm 5.6}$ & $\,~6.7_{\pm 8.5}$
\\
hTERM \nolink{\citep{li2021tilted}} & $76.3_{\pm 0.3}$ & $73.9_{\pm 2.0}$
                                    & $74.3_{\pm 0.6}$ & $63.4_{\pm 1.3}$
                                    & $71.0_{\pm 0.7}$ & $52.2_{\pm 1.8}$
                                    & $65.3_{\pm 1.1}$ & $45.9_{\pm 1.1}$
                                    & $64.7_{\pm 0.4}$ & $39.6_{\pm 1.4}$
\\
\nolink{\citep{konstantinov2020sample}} & $74.3_{\pm 0.2}$ & $93.4_{\pm 0.8}$ 
                                        & $73.7_{\pm 0.2}$ & $87.0_{\pm 0.8}$ 
                                        & $72.9_{\pm 0.5}$ & $80.1_{\pm 0.9}$ 
                                        & $71.2_{\pm 0.8}$ & $73.4_{\pm 0.6}$ 
                                        & $58.2_{\pm 6.2}$ & $73.9_{\pm 1.0}$ 
\\\hline
FLEA (proposed)             & $75.4_{\pm 0.2}$ & $99.4_{\pm 0.2}$ 
                            & $75.4_{\pm 0.2}$ & $99.5_{\pm 0.2}$ 
                            & $75.4_{\pm 0.2}$ & $99.5_{\pm 0.2}$ 
                            & $75.3_{\pm 0.2}$ & $99.4_{\pm 0.2}$ 
                            & $74.0_{\pm 1.4}$ & $94.2_{\pm 1.5}$
\\\hline
oracle	           & $75.2_{\pm 0.2}$ & $99.5_{\pm 0.3}$ 
                   & $75.2_{\pm 0.2}$ & $99.6_{\pm 0.2}$ 
                   & $75.3_{\pm 0.2}$ & $99.7_{\pm 0.2}$ 
                   & $75.3_{\pm 0.2}$ & $99.7_{\pm 0.2}$ 
                   & $75.1_{\pm 0.3}$ & $99.6_{\pm 0.4}$ 
\end{tabular}
\end{subfigure}
\end{table*}

\subsection{Results}\label{subsec:results}
The results of our experiments show a very consistent 
picture across different datasets, base learners and 
adversaries.
For the sake of conciseness, for FLEA we only present
the results using a regularization-based fairness-aware 
learner in the main manuscript. 
Results for other learners are qualitatively the same
and can be found, together with more detailed results
and ablation studies, in the supplemental material.

In Table~\ref{tab:results} we report results for 
six learning methods: an ordinary learner that is 
fairness-aware but not protected against data 
manipulations (naive), the proposed \method,  %
and the four baseline methods: the robust ensemble, 
DRO (adapted from \citet{wang2020robustfairnessnoise}), 
hTERM (following \citet{li2021tilted}), 
and discrepancy-based filtering \citet{konstantinov2020sample}.
In addition, we report the value of a hypothetical 
oracle-based learner that knows which of the sources 
are actually clean and learns only on their data. 

Each entry in the table is the \emph{minimum} accuracy 
and fairness in the respective setting across all 
eleven tested adversaries. 
We choose this worst-case measure because it allows 
a compact representation and reflects the fact that
a real-world system should be robust against all 
possible data errors or manipulation simultaneously. 
Results broken down by individual adversaries are 
provided in the supplemental material.

Examining the results, a comparison of the \emph{naive} results
with the \emph{oracle} confirms that the need for
robust learning method is real: naive fairness-aware 
learning is not sufficient to ensure fair (or accurate) 
classifiers in the presence of unreliable data.

An ideal robust method should achieve results 
approximately as good as the \emph{oracle} result, 
as this would indicate that the adversary was indeed 
not able to negatively affect the learning process 
beyond the unavoidable loss of some training data.
The results show that \method comes close to 
this behavior, but none of the other methods does.
In the homogeneous setting (Table~\ref{tab:results-homogeneous}),
for the largest dataset, \adult, \method reliably 
suppresses the effects of all tested adversaries.
It learns classifiers with accuracy and fairness 
almost exactly those of a fair classifier trained 
only on the clean data sources. 
For the other datasets, \compas, \drugs and \german, 
\method increases the accuracy and fairness to levels
only slightly below the oracle.
In all cases, \method's results are as good as 
or better than the baselines; the robust ensemble 
is also able to improve fairness to some extent, 
but it does not reach the oracle results.

The DRO-based and hTERM approaches 
show highly volatile behavior. For some adversaries 
they improve fairness or accuracy, but for some adversaries 
they fail severely.
Consequently, their min-aggregated values in 
the table are often even lower than for the naive 
method.
Note that these results should be seen in 
context though: \citet{wang2020robustfairnessnoise}
is designed for a different and less challenging 
data manipulation model. 
\citet{li2021tilted}'s notion of robustness and 
fairness differ from the ones we employ in this 
work.

The approach from \citet{konstantinov2020sample} has 
almost no effect. Only for the largest dataset, 
\adult, it yields a slight accuracy improvement.
This can be explained by the fact that the method only 
removes sources that it can confidently identify as 
manipulated. The theory-derived thresholds for this 
are quite strict, so the method is ineffective unless 
a lot of data is available.
The observed characteristics of the different methods 
hold also for the other base learners, see the supplemental material.

In the heterogeneous setting (Table~\ref{tab:results-heterogeneous})
the results show similar trends: for $N-K\in\{5,10,15,20\}$, 
\method manages reliably to filter out the malignant 
sources, such that the accuracy and fairness of the 
learned classifiers matches the one of the oracle method
almost perfectly.  
The robust ensemble has a positive effect, but less 
so than \method. For this data, DRO somewhat improves 
fairness, but this comes at a loss of accuracy. 
The method from~\citet{konstantinov2020sample} has 
no noticeable effect. 
For $N-K=25$, \method still performs best, although 
a bit worse than the hypothetical best oracle. 
Presumably, this is because the combined effect of 
distribution differences between the sources and 
the uncertainty due to finite sampling when estimating
$\disc$, $\disp$ and $\disb$ are too large to perfectly 
allow a decision which $26$ sources to keep and which 
$25$ to exclude.

\section{Conclusion}
We studied the task of fairness-aware 
classification in the setting when data from multiple 
sources is available, but some of them might by noisy,
contain unintentional errors, or even have been maliciously 
manipulated.
Ordinary fairness-aware learning methods are not robust 
against such problems and often fail to produce 
fair classifiers. 
We proposed a filtering-based algorithm, \method, that 
is able to identify and suppress those data sources 
that would negatively affect the training process, 
thereby restoring the property that fairness-aware 
learning methods actually produce fair classifiers. 
We showed the effectiveness of \method experimentally, 
and we also presented a theorem that provides formal 
guarantees of \method's efficacy.

Despite our promising results, we consider \method just 
a first step on the path toward making fairness-aware 
learning more robust. 
One potential future step is to include other notions
of fairness besides \emph{demographic parity}.
So far, \method can already be used as it is with classifiers 
that enforce other fairness criteria. 
However, our theoretical guarantees do not holds for these, 
as the \emph{disparity} measure that enters our filtering step 
is not tailored to them. 
We do not see fundamental problems in deriving filtering steps 
for other fairness notions that are also defined in terms 
of properties of the joint distribution of inputs, outputs, 
and protected attributes, such as \emph{equality of opportunity} 
or \emph{equalized odds}. However, the theoretical analysis 
and the practical implementation could get more involved. 

On the algorithmic side, \method as we formulated it, requires 
computing all pairwise similarities between the sources. 
This could render it inefficient when the number of sources 
is very large (\eg thousands). 
We expect that it will be possible to overcome this, for 
example by randomization of the sources, but we leave this 
step to future work.

\section{Acknowledgements}
The authors would like to thank Bernd Prach, Elias Frantar, 
Alexandra Peste, Mahdi Nikdan, and Peter S\a'uken\a'ik 
for their helpful feedback.
This research was supported by the Scientific Service Units
(SSU) of IST Austria through resources provided by Scientific
Computing (SciComp). This publication was made possible by an ETH AI Center postdoctoral fellowship granted to Nikola Konstantinov.
Eugenia Iofinova was supported in part by the FWF DK VGSCO, grant agreement number W1260-N35.

\bibliographystyle{icml2022}
\bibliography{ms}

\begin{thebibliography}{79}
\providecommand{\natexlab}[1]{#1}
\providecommand{\url}[1]{\texttt{#1}}
\expandafter\ifx\csname urlstyle\endcsname\relax
  \providecommand{\doi}[1]{doi: #1}\else
  \providecommand{\doi}{doi: \begingroup \urlstyle{rm}\Url}\fi

\bibitem[Agarwal et~al.(2018)Agarwal, Beygelzimer, Dudik, Langford, and
  Wallach]{agarwal2018reductions}
Agarwal, A., Beygelzimer, A., Dudik, M., Langford, J., and Wallach, H.
\newblock A reductions approach to fair classification.
\newblock In \emph{International Conference on Machine Learing (ICML)}, 2018.

\bibitem[Aingwin et~al.(2016)Aingwin, Larson, Mattu, and
  Kirchner]{Angwin_compas}
Aingwin, J., Larson, J., Mattu, S., and Kirchner, L.
\newblock Machine bias: There's software used across the country to predict
  future criminals and its biased against blacks., 2016.
\newblock URL \url{https://github.com/propublica/compas-analysis}.

\bibitem[Awasthi et~al.(2017)Awasthi, Blum, Haghtalab, and
  Mansour]{awasthi2017efficient}
Awasthi, P., Blum, A., Haghtalab, N., and Mansour, Y.
\newblock Efficient {PAC} learning from the crowd.
\newblock In \emph{Workshop on Computational Learning Theory (COLT)}, 2017.

\bibitem[Barnett \& Lewis(1984)Barnett and Lewis]{barnett1984outliers}
Barnett, V. and Lewis, T.
\newblock \emph{Outliers in statistical data}.
\newblock Wiley, 1984.

\bibitem[Barocas et~al.(2019)Barocas, Hardt, and
  Narayanan]{barocas-hardt-narayanan}
Barocas, S., Hardt, M., and Narayanan, A.
\newblock \emph{Fairness and Machine Learning}.
\newblock fairmlbook.org, 2019.

\bibitem[Bartlett(1992)]{bartlett1992learning}
Bartlett, P.~L.
\newblock Learning with a slowly changing distribution.
\newblock In \emph{Workshop on Computational Learning Theory (COLT)}, 1992.

\bibitem[Ben-David et~al.(2010)Ben-David, Blitzer, Crammer, Kulesza, Pereira,
  and Vaughan]{ben2010theory}
Ben-David, S., Blitzer, J., Crammer, K., Kulesza, A., Pereira, F., and Vaughan,
  J.~W.
\newblock A theory of learning from different domains.
\newblock \emph{Machine Learning (ML)}, 2010.

\bibitem[Beutel et~al.(2017)Beutel, Chen, Zhao, and Chi]{beutel2017data}
Beutel, A., Chen, J., Zhao, Z., and Chi, E.~H.
\newblock Data decisions and theoretical implications when adversarially
  learning fair representations.
\newblock In \emph{Conference on Fairness, Accountability and Transparency
  (FAccT)}, 2017.

\bibitem[Calders et~al.(2009)Calders, Kamiran, and
  Pechenizkiy]{calders2009parity}
Calders, T., Kamiran, F., and Pechenizkiy, M.
\newblock Building classifiers with independency constraints.
\newblock In \emph{International Conference on Data Mining Workshops (IDCMW)},
  2009.

\bibitem[Calmon et~al.(2017)Calmon, Wei, Vinzamuri, Natesan~Ramamurthy, and
  Varshney]{calmon2017optimized}
Calmon, F., Wei, D., Vinzamuri, B., Natesan~Ramamurthy, K., and Varshney, K.~R.
\newblock Optimized pre-processing for discrimination prevention.
\newblock In \emph{Conference on Neural Information Processing Systems
  (NeurIPS)}, 2017.

\bibitem[Celis et~al.(2020)Celis, Keswani, and Vishnoi]{celis2020data}
Celis, L.~E., Keswani, V., and Vishnoi, N.
\newblock Data preprocessing to mitigate bias: A maximum entropy based
  approach.
\newblock In \emph{International Conference on Machine Learing (ICML)}, 2020.

\bibitem[Celis et~al.(2021{\natexlab{a}})Celis, Huang, Keswani, and
  Vishnoi]{celis2021fair}
Celis, L.~E., Huang, L., Keswani, V., and Vishnoi, N.~K.
\newblock Fair classification with noisy protected attributes: A framework with
  provable guarantees.
\newblock In \emph{International Conference on Machine Learing (ICML)},
  2021{\natexlab{a}}.

\bibitem[Celis et~al.(2021{\natexlab{b}})Celis, Mehrotra, and
  Vishnoi]{celis2021adversarial}
Celis, L.~E., Mehrotra, A., and Vishnoi, N.~K.
\newblock Fair classification with adversarial perturbations.
\newblock In \emph{Conference on Neural Information Processing Systems
  (NeurIPS)}, 2021{\natexlab{b}}.

\bibitem[Charikar et~al.(2017)Charikar, Steinhardt, and
  Valiant]{charikar2017learning}
Charikar, M., Steinhardt, J., and Valiant, G.
\newblock Learning from untrusted data.
\newblock In \emph{Symposium on Theory of Computing (STOC)}, 2017.

\bibitem[Chen et~al.(2019)Chen, Li, and Moitra]{chen2019efficiently}
Chen, S., Li, J., and Moitra, A.
\newblock Efficiently learning structured distributions from untrusted batches.
\newblock In \emph{Symposium on Theory of Computing (STOC)}, 2019.

\bibitem[Cho et~al.(2020)Cho, Hwang, and Suh]{cho2020fair}
Cho, J., Hwang, G., and Suh, C.
\newblock A fair classifier using kernel density estimation.
\newblock In \emph{Conference on Neural Information Processing Systems
  (NeurIPS)}, 2020.

\bibitem[Chuang \& Mroueh(2021)Chuang and Mroueh]{chuang2021fair}
Chuang, C.-Y. and Mroueh, Y.
\newblock Fair mixup: Fairness via interpolation.
\newblock In \emph{International Conference on Learning Representations
  (ICLR)}, 2021.

\bibitem[Chzhen et~al.(2020)Chzhen, Denis, Hebiri, Oneto, and
  Pontil]{Chzhen2020FairRegressionWasserstein}
Chzhen, E., Denis, C., Hebiri, M., Oneto, L., and Pontil, M.
\newblock Fair regression with {Wasserstein} barycenters.
\newblock In \emph{Conference on Neural Information Processing Systems
  (NeurIPS)}, 2020.

\bibitem[Corder \& Foreman(2014)Corder and Foreman]{corder2014nonparametric}
Corder, G.~W. and Foreman, D.~I.
\newblock \emph{Nonparametric statistics: A step-by-step approach}.
\newblock John Wiley \& Sons, 2014.

\bibitem[Cormen et~al.(2009)Cormen, Leiserson, Rivest, and
  Stein]{cormen2009introduction}
Cormen, T.~H., Leiserson, C.~E., Rivest, R.~L., and Stein, C.
\newblock \emph{Introduction to algorithms}.
\newblock The MIT Press, 2009.

\bibitem[Crammer et~al.(2008)Crammer, Kearns, and Wortman]{crammer2008learning}
Crammer, K., Kearns, M., and Wortman, J.
\newblock Learning from multiple sources.
\newblock \emph{Journal of Machine Learning Research (JMLR)}, 2008.

\bibitem[Ding et~al.(2021)Ding, Hardt, Miller, and Schmidt]{ding2021retiring}
Ding, F., Hardt, M., Miller, J., and Schmidt, L.
\newblock Retiring adult: New datasets for fair machine learning.
\newblock In \emph{Conference on Neural Information Processing Systems
  (NeurIPS)}, 2021.

\bibitem[Donini et~al.(2018)Donini, Oneto, Ben-David, Shawe-Taylor, and
  Pontil]{donini2018empirical}
Donini, M., Oneto, L., Ben-David, S., Shawe-Taylor, J.~S., and Pontil, M.
\newblock Empirical risk minimization under fairness constraints.
\newblock In \emph{Conference on Neural Information Processing Systems
  (NeurIPS)}, 2018.

\bibitem[Dua \& Graff(2017)Dua and Graff]{UCI}
Dua, D. and Graff, C.
\newblock {UCI} machine learning repository, 2017.
\newblock URL \url{http://archive.ics.uci.edu/ml}.

\bibitem[Dwork et~al.(2012)Dwork, Hardt, Pitassi, Reingold, and
  Zemel]{dwork2012fairness}
Dwork, C., Hardt, M., Pitassi, T., Reingold, O., and Zemel, R.
\newblock Fairness through awareness.
\newblock In \emph{Innovations in Theoretical Computer Science Conference
  (ITCS)}, 2012.

\bibitem[Erfani et~al.(2017)Erfani, Baktashmotlagh, Moshtaghi, Nguyen, Leckie,
  Bailey, and Ramamohanarao]{erfani2017shared}
Erfani, S., Baktashmotlagh, M., Moshtaghi, M., Nguyen, V., Leckie, C., Bailey,
  J., and Ramamohanarao, K.
\newblock From shared subspaces to shared landmarks: A robust multi-source
  classification approach.
\newblock In \emph{Conference on Artificial Intelligence (AAAI)}, 2017.

\bibitem[Feng et~al.(2019)Feng, Cai, and Zhou]{feng2019learning}
Feng, J., Cai, Q.-Z., and Zhou, Z.-H.
\newblock Learning to confuse: Generating training time adversarial data with
  auto-encoder.
\newblock In \emph{Conference on Neural Information Processing Systems
  (NeurIPS)}, 2019.

\bibitem[Fowl et~al.(2021)Fowl, Goldblum, Chiang, Geiping, Czaja, and
  Goldstein]{fowl2021adversarial}
Fowl, L., Goldblum, M., Chiang, P.-y., Geiping, J., Czaja, W., and Goldstein,
  T.
\newblock Adversarial examples make strong poisons.
\newblock In \emph{Conference on Neural Information Processing Systems
  (NeurIPS)}, 2021.

\bibitem[Friedman(2001)]{friedman2001greedy}
Friedman, J.~H.
\newblock Greedy function approximation: a gradient boosting machine.
\newblock \emph{Annals of Statistics}, 2001.

\bibitem[Goodfellow et~al.(2014)Goodfellow, Pouget-Abadie, Mirza, Xu,
  Warde-Farley, Ozair, Courville, and Bengio]{goodfellow2014generative}
Goodfellow, I.~J., Pouget-Abadie, J., Mirza, M., Xu, B., Warde-Farley, D.,
  Ozair, S., Courville, A., and Bengio, Y.
\newblock Generative adversarial networks.
\newblock In \emph{Conference on Neural Information Processing Systems
  (NeurIPS)}, 2014.

\bibitem[Goodfellow et~al.(2015)Goodfellow, Shlens, and
  Szegedy]{goodfellow2015explaining}
Goodfellow, I.~J., Shlens, J., and Szegedy, C.
\newblock Explaining and harnessing adversarial examples.
\newblock In \emph{International Conference on Learning Representations
  (ICLR)}, 2015.

\bibitem[Gretton et~al.(2012)Gretton, Borgwardt, Rasch, Sch{{\"o}}lkopf, and
  Smola]{JMLR:v13:gretton12a}
Gretton, A., Borgwardt, K.~M., Rasch, M.~J., Sch{{\"o}}lkopf, B., and Smola, A.
\newblock A kernel two-sample test.
\newblock \emph{Journal of Machine Learning Research (JMLR)}, 2012.

\bibitem[Halmos(2013)]{halmos2013measure}
Halmos, P.~R.
\newblock \emph{Measure Theory}.
\newblock Springer, 2013.

\bibitem[Hanneke \& Kpotufe(2020)Hanneke and Kpotufe]{hanneke2020no}
Hanneke, S. and Kpotufe, S.
\newblock A no-free-lunch theorem for multitask learning.
\newblock \emph{arXiv preprint arXiv:2006.15785}, 2020.

\bibitem[Hardt et~al.(2016)Hardt, Price, and Srebro]{Hardt2016Equality}
Hardt, M., Price, E., and Srebro, N.
\newblock Equality of opportunity in supervised learning.
\newblock In \emph{Conference on Neural Information Processing Systems
  (NeurIPS)}, 2016.

\bibitem[Jain \& Orlitsky(2020{\natexlab{a}})Jain and
  Orlitsky]{jain2020_optimalrobustlearning}
Jain, A. and Orlitsky, A.
\newblock Optimal robust learning of discrete distributions from batches.
\newblock In \emph{International Conference on Machine Learing (ICML)},
  2020{\natexlab{a}}.

\bibitem[Jain \& Orlitsky(2020{\natexlab{b}})Jain and
  Orlitsky]{jain2020general}
Jain, A. and Orlitsky, A.
\newblock A general method for robust learning from batches.
\newblock In \emph{Conference on Neural Information Processing Systems
  (NeurIPS)}, 2020{\natexlab{b}}.

\bibitem[Kallus et~al.(2020)Kallus, Mao, and Zhou]{kallus2019assessing}
Kallus, N., Mao, X., and Zhou, A.
\newblock Assessing algorithmic fairness with unobserved protected class using
  data combination.
\newblock In \emph{Conference on Fairness, Accountability and Transparency
  (FAccT)}, 2020.

\bibitem[Kamiran \& Calders(2012)Kamiran and Calders]{kamiran2012data}
Kamiran, F. and Calders, T.
\newblock Data preprocessing techniques for classification without
  discrimination.
\newblock \emph{Knowledge and Information Systems (KAIS)}, 2012.

\bibitem[Kamishima et~al.(2012)Kamishima, Akaho, Asoh, and
  Sakuma]{kamishima2012fairness}
Kamishima, T., Akaho, S., Asoh, H., and Sakuma, J.
\newblock Fairness-aware classifier with prejudice remover regularizer.
\newblock In \emph{European Conference on Machine Learning and Data Mining
  (ECML PKDD)}, 2012.

\bibitem[Kearns \& Li(1993)Kearns and Li]{kearns1993learning}
Kearns, M. and Li, M.
\newblock Learning in the presence of malicious errors.
\newblock \emph{SIAM Journal on Computing}, 1993.

\bibitem[Khetan et~al.(2018)Khetan, Lipton, and Anandkumar]{khetan2017learning}
Khetan, A., Lipton, Z.~C., and Anandkumar, A.
\newblock Learning from noisy singly-labeled data.
\newblock In \emph{International Conference on Learning Representations
  (ICLR)}, 2018.

\bibitem[Kifer et~al.(2004)Kifer, Ben-David, and Gehrke]{kifer2004detecting}
Kifer, D., Ben-David, S., and Gehrke, J.
\newblock Detecting change in data streams.
\newblock In \emph{International Conference on Very Large Data Bases (VLDB)},
  2004.

\bibitem[Konstantinov \& Lampert(2019)Konstantinov and
  Lampert]{konstantinov2019robust}
Konstantinov, N. and Lampert, C.~H.
\newblock Robust learning from untrusted sources.
\newblock In \emph{International Conference on Machine Learing (ICML)}, 2019.

\bibitem[Konstantinov \& Lampert(2022)Konstantinov and
  Lampert]{konstantinov2022theotherone}
Konstantinov, N. and Lampert, C.~H.
\newblock Fairness-aware learning from corrupted data.
\newblock \emph{Journal of Machine Learning Research (JMLR)}, 2022.

\bibitem[Konstantinov et~al.(2020)Konstantinov, Frantar, Alistarh, and
  Lampert]{konstantinov2020sample}
Konstantinov, N., Frantar, E., Alistarh, D., and Lampert, C.
\newblock On the sample complexity of adversarial multi-source {PAC} learning.
\newblock In \emph{International Conference on Machine Learing (ICML)}, 2020.

\bibitem[Kullback \& Leibler(1951)Kullback and
  Leibler]{kullback1951information}
Kullback, S. and Leibler, R.~A.
\newblock On information and sufficiency.
\newblock \emph{The Annals of Mathematical Statistics}, 1951.

\bibitem[Lahoti et~al.(2020)Lahoti, Beutel, Chen, Lee, Prost, Thain, Wang, and
  Chi]{Lahoti2020AdversarialReweighting}
Lahoti, P., Beutel, A., Chen, J., Lee, K., Prost, F., Thain, N., Wang, X., and
  Chi, E.
\newblock Fairness without demographics through adversarially reweighted
  learning.
\newblock In \emph{Conference on Neural Information Processing Systems
  (NeurIPS)}, 2020.

\bibitem[Lamy et~al.(2019)Lamy, Zhong, Menon, and Verma]{lamy2019noise}
Lamy, A., Zhong, Z., Menon, A.~K., and Verma, N.
\newblock Noise-tolerant fair classification.
\newblock In \emph{Conference on Neural Information Processing Systems
  (NeurIPS)}, 2019.

\bibitem[Li et~al.(2021{\natexlab{a}})Li, Beirami, Sanjabi, and
  Smith]{li2021tilted}
Li, T., Beirami, A., Sanjabi, M., and Smith, V.
\newblock Tilted empirical risk minimization.
\newblock In \emph{International Conference on Learning Representations
  (ICLR)}, 2021{\natexlab{a}}.

\bibitem[Li et~al.(2021{\natexlab{b}})Li, Hu, Beirami, and Smith]{li2021ditto}
Li, T., Hu, S., Beirami, A., and Smith, V.
\newblock Ditto: Fair and robust federated learning through personalization.
\newblock In \emph{International Conference on Machine Learing (ICML)},
  2021{\natexlab{b}}.

\bibitem[Mandal et~al.(2020)Mandal, Deng, Jana, Wing, and
  Hsu]{Mandal2020EnsuringFairness}
Mandal, D., Deng, S., Jana, S., Wing, J., and Hsu, D.~J.
\newblock Ensuring fairness beyond the training data.
\newblock In \emph{Conference on Neural Information Processing Systems
  (NeurIPS)}, 2020.

\bibitem[Marcotte \& Savard(1992)Marcotte and Savard]{marcotte1992novel}
Marcotte, P. and Savard, G.
\newblock Novel approaches to the discrimination problem.
\newblock \emph{Zeitschrift f{\"u}r Operations Research}, 1992.

\bibitem[Mehrabi et~al.(2021{\natexlab{a}})Mehrabi, Morstatter, Saxena, Lerman,
  and Galstyan]{mehrabi2021survey}
Mehrabi, N., Morstatter, F., Saxena, N., Lerman, K., and Galstyan, A.
\newblock A survey on bias and fairness in machine learning.
\newblock \emph{ACM Computing Surveys (CSUR)}, 2021{\natexlab{a}}.

\bibitem[Mehrabi et~al.(2021{\natexlab{b}})Mehrabi, Naveed, Morstatter, and
  Galstyan]{mehrabi2020exacerbating}
Mehrabi, N., Naveed, M., Morstatter, F., and Galstyan, A.
\newblock Exacerbating algorithmic bias through fairness attacks.
\newblock In \emph{Conference on Artificial Intelligence (AAAI)},
  2021{\natexlab{b}}.

\bibitem[Mehrotra \& Celis(2021)Mehrotra and Celis]{mehrotra2021setselection}
Mehrotra, A. and Celis, L.~E.
\newblock Mitigating bias in set selection with noisy protected attributes.
\newblock In \emph{Conference on Fairness, Accountability and Transparency
  (FAccT)}, 2021.

\bibitem[Mohri \& Medina(2012)Mohri and Medina]{mohri2012new}
Mohri, M. and Medina, A.~M.
\newblock New analysis and algorithm for learning with drifting distributions.
\newblock In \emph{Algorithmic Learning Theory (ALT)}, 2012.

\bibitem[Natarajan et~al.(2013)Natarajan, Dhillon, Ravikumar, and
  Tewari]{natarajan2013learning}
Natarajan, N., Dhillon, I.~S., Ravikumar, P., and Tewari, A.
\newblock Learning with noisy labels.
\newblock In \emph{Conference on Neural Information Processing Systems
  (NeurIPS)}, 2013.

\bibitem[Qiao \& Valiant(2018)Qiao and Valiant]{qiao2018learning}
Qiao, M. and Valiant, G.
\newblock Learning discrete distributions from untrusted batches.
\newblock In \emph{Innovations in Theoretical Computer Science Conference
  (ITCS)}, 2018.

\bibitem[Rezaei et~al.(2020)Rezaei, Fathony, Memarrast, and
  Ziebart]{rezaei2020fairness}
Rezaei, A., Fathony, R., Memarrast, O., and Ziebart, B.
\newblock Fairness for robust log loss classification.
\newblock In \emph{Conference on Artificial Intelligence (AAAI)}, 2020.

\bibitem[Roh et~al.(2020)Roh, Lee, Whang, and Suh]{roh20FRtrain}
Roh, Y., Lee, K., Whang, S., and Suh, C.
\newblock {FR}-train: A mutual information-based approach to fair and robust
  training.
\newblock In \emph{International Conference on Machine Learing (ICML)}, 2020.

\bibitem[Roh et~al.(2021)Roh, Lee, Whang, and Suh]{roh2021sample}
Roh, Y., Lee, K., Whang, S., and Suh, C.
\newblock Sample selection for fair and robust training.
\newblock In \emph{Conference on Neural Information Processing Systems
  (NeurIPS)}, 2021.

\bibitem[Russakovsky et~al.(2015)Russakovsky, Deng, Su, Krause, Satheesh, Ma,
  Huang, Karpathy, Khosla, Bernstein, Berg, and Fei-Fei]{imagenet}
Russakovsky, O., Deng, J., Su, H., Krause, J., Satheesh, S., Ma, S., Huang, Z.,
  Karpathy, A., Khosla, A., Bernstein, M., Berg, A.~C., and Fei-Fei, L.
\newblock {ImageNet Large Scale Visual Recognition Challenge}.
\newblock \emph{International Journal of Computer Vision (IJCV)}, 2015.

\bibitem[Shalev-Shwartz \& Ben-David(2014)Shalev-Shwartz and
  Ben-David]{shwartz2014understanding}
Shalev-Shwartz, S. and Ben-David, S.
\newblock \emph{Understanding machine learning: from theory to algorithms}.
\newblock Cambridge University Press, 2014.

\bibitem[Smith \& Martinez(2018)Smith and Martinez]{smith2018robustness}
Smith, M.~R. and Martinez, T.~R.
\newblock The robustness of majority voting compared to filtering misclassified
  instances in supervised classification tasks.
\newblock \emph{Artificial Intelligence Review}, 2018.

\bibitem[Student(1908)]{student1908probable}
Student.
\newblock The probable error of a mean.
\newblock \emph{Biometrika}, 1908.

\bibitem[Tan et~al.(2020)Tan, Yeom, Fredrikson, and Talwalkar]{tan2020learning}
Tan, Z., Yeom, S., Fredrikson, M., and Talwalkar, A.
\newblock Learning fair representations for kernel models.
\newblock In \emph{Conference on Uncertainty in Artificial Intelligence
  (AISTATS)}, 2020.

\bibitem[Ting \& Low(1997)Ting and Low]{ting1997model}
Ting, K.~M. and Low, B.~T.
\newblock Model combination in the multiple-data-batches scenario.
\newblock In \emph{European Conference on Marchine Learning (ECML)}, 1997.

\bibitem[Tsybakov(2009)]{Tsybakov:1315296}
Tsybakov, A.~B.
\newblock \emph{{Introduction to Nonparametric Estimation}}.
\newblock Springer series in statistics. Springer, 2009.

\bibitem[Vapnik(2013)]{vapnik2013nature}
Vapnik, V.
\newblock \emph{The nature of statistical learning theory}.
\newblock Statistics for Engineering and Information Science. Springer, 2013.

\bibitem[Villani(2009)]{villani2009optimal}
Villani, C.
\newblock \emph{Optimal transport: old and new}.
\newblock Springer, 2009.

\bibitem[Wadsworth et~al.(2018)Wadsworth, Vera, and Piech]{wadsworth2018}
Wadsworth, C., Vera, F., and Piech, C.
\newblock Achieving fairness through adversarial learning: an application to
  recidivism prediction.
\newblock In \emph{Conference on Fairness, Accountability and Transparency
  (FAccT)}, 2018.

\bibitem[Wang et~al.(2019)Wang, Ustun, and Calmon]{wang2019repairing}
Wang, H., Ustun, B., and Calmon, F.
\newblock Repairing without retraining: Avoiding disparate impact with
  counterfactual distributions.
\newblock In \emph{International Conference on Machine Learing (ICML)}, 2019.

\bibitem[Wang et~al.(2020)Wang, Guo, Narasimhan, Cotter, Gupta, and
  Jordan]{wang2020robustfairnessnoise}
Wang, S., Guo, W., Narasimhan, H., Cotter, A., Gupta, M., and Jordan, M.
\newblock Robust optimization for fairness with noisy protected groups.
\newblock In \emph{Conference on Neural Information Processing Systems
  (NeurIPS)}, 2020.

\bibitem[Woodworth et~al.(2017)Woodworth, Gunasekar, Ohannessian, and
  Srebro]{woodworth17learning}
Woodworth, B., Gunasekar, S., Ohannessian, M.~I., and Srebro, N.
\newblock Learning non-discriminatory predictors.
\newblock In \emph{Workshop on Computational Learning Theory (COLT)}, 2017.

\bibitem[Zafar et~al.(2017{\natexlab{a}})Zafar, Valera, Gomez~Rodriguez, and
  Gummadi]{zafar2017fairness}
Zafar, M.~B., Valera, I., Gomez~Rodriguez, M., and Gummadi, K.~P.
\newblock Fairness beyond disparate treatment \& disparate impact: Learning
  classification without disparate mistreatment.
\newblock In \emph{International World Wide Web Conference (WWW)},
  2017{\natexlab{a}}.

\bibitem[Zafar et~al.(2017{\natexlab{b}})Zafar, Valera, Rogriguez, and
  Gummadi]{zafar17FairnessConstraints}
Zafar, M.~B., Valera, I., Rogriguez, M.~G., and Gummadi, K.~P.
\newblock Fairness constraints: Mechanisms for fair classification.
\newblock In \emph{Conference on Uncertainty in Artificial Intelligence
  (AISTATS)}, 2017{\natexlab{b}}.

\bibitem[Zemel et~al.(2013)Zemel, Wu, Swersky, Pitassi, and
  Dwork]{zemel2013learning}
Zemel, R., Wu, Y., Swersky, K., Pitassi, T., and Dwork, C.
\newblock Learning fair representations.
\newblock In \emph{International Conference on Machine Learing (ICML)}, 2013.

\bibitem[Zhang et~al.(2018)Zhang, Lemoine, and
  Mitchell]{zhang2018MitigatingBiasAdversarial}
Zhang, B.~H., Lemoine, B., and Mitchell, M.
\newblock Mitigating unwanted biases with adversarial learning.
\newblock In \emph{Conference on AI, Ethics, and Society (AIES)}, 2018.

\end{thebibliography}

\clearpage
\appendix

\onecolumn

\centerline{\LARGE\textbf{Appendix}}

\bigskip
\paragraph{Table of Contents}
\begin{itemize}
\item \textbf{A: Experimental Setup}
\item \textbf{B: Detailed Algorithm for Estimating $\disc$ and $\disp$}
\item \textbf{C: Detailed Experimental Results}
\item \textbf{D: Discussion of the Role of $\disb$, $\disc$ and $\disp$ and Ablation Study}
\item \textbf{E: Complete Formulation and Proof of Theorem 1}
\end{itemize}

\section{Experimental setup}

\subsection{Dataset preparation}

The datasets we use are publicly available and frequently 
used to evaluate fair classification methods.

The \compas dataset was introduced by ProPublica. It contains 
data from the US criminal justice system and was obtained by 
a public records request.
The dataset contains personal information. To mitigate negative
side effects, we delete the \emph{name}, \emph{first}, \emph{last} 
and \emph{dob} (date of birth) entries from the dataset before 
processing it further.
We then exclude entries that do not fit the problem setting of 
predicting two year recidivism, following the steps of the 
original analysis.\footnote{\url{https://github.com/propublica/compas-analysis}}
Specifically, this means keeping only cases from Broward county, Florida, 
for which data has been entered within 30 days of the arrest. 
Traffic offenses and cases with insufficient information are 
also excluded. 
This steps leave 6171 examples out of the original 7214 cases.
The categorical features and numerical features that we extract 
from the data are provided in Table~\ref{tab:info-compas}. 

\adult, \german{}, and \drugs{} are available in the UCI data 
repository as well as multiple other online sources.\footnote{
\adult: {\scriptsize\url{https://archive.ics.uci.edu/ml/datasets/adult}},%

\german:{\scriptsize\url{https://github.com/praisan/hello-world/blob/master/german_credit_data.csv}}, 

\drugs: {\scriptsize\url{https://raw.githubusercontent.com/deepak525/Drug-Consumption/master/drug_consumption.csv}}%
}
We use them in unmodified form, except for binning some of the 
feature values; see Tables~\ref{tab:datasetinfo} and \ref{tab:datasetinfo2}.

\begin{table}[H]\small
\caption{Dataset information}\label{tab:datasetinfo}
\begin{subfigure}[b]{.95\textwidth}\centering
\caption{\adult}\label{tab:info-adult}
\begin{tabular}{p{.18\textwidth}|p{.17\textwidth}|p{.63\textwidth}}
dataset size & 48842 
\\\hline
categorical features & \emph{workclass} &       federal-gov, local-gov, never-worked, private, self-emp-inc, self-emp-not-inc, state-gov, without-pay, unknown\\
                     & \emph{education} &       1st-4th, 5th-6th, 7th-8th, 9th, 10th, 11th, 12th, Assoc-acdm, Assoc-voc, Bachelors, Doctorate, HS-grad, Masters, Preschool, Prof-school, Some-college \\
                     & \emph{hours-per-week} &  $\leq$ 19, 20--29, 30--39, $\geq$ 40 \\
                     & \emph{age}            &  $\leq$ 24, 25--34, 35--44, 45--54, 55--64, $\geq$ 65 \\
                     & \emph{native-country} &  United States, other\\  
                     & \emph{race} &            Amer-Indian-Eskimo, Asian-Pac-Islander, Black, White, other \\\hline
numerical features & --- 
\\\hline
protected attribute  & \emph{gender} &          values: female (33.2\%), male (66.8\%)
\\\hline
target variable & \textit{income} & $\leq 50K$ (76.1\%), $>50K$ (33.9\%) 
\end{tabular}
\end{subfigure}
\end{table}

\begin{table}[H]\small
\caption{Dataset information (continued)}\label{tab:datasetinfo2}
\begin{subfigure}[b]{.95\textwidth}\centering
\caption{\compas}\label{tab:info-compas}
\begin{tabular}{p{.18\textwidth}|p{.3\textwidth}|p{.5\textwidth}}
dataset size & 6171 (7214 before filtering)
\\\hline
categorical features & \emph{c-charge-degree} & values: F (felony), M (misconduct)   \\
                     & \emph{age-cat} & values: $<$25, 25--45, $>$45 \\
                     & \emph{race}  & values: African-American, Caucasian, Hispanic, Other 
\\\hline
numerical features & \emph{priors-count} & 
\\\hline
protected attribute &\textit{sex} & Female (19.0\%), Male (81.0\%)
\\\hline
target variable & \textit{two-year-recid} & 0 (54.9\%), 1 (45.1\%)
\end{tabular}
\end{subfigure}\\[2.\baselineskip]
\begin{subfigure}[b]{.95\textwidth}\centering
\caption{\drugs}\label{tab:info-drugs}
\begin{tabular}{p{.18\textwidth}|p{.3\textwidth}|p{.5\textwidth}}
dataset size & 1885
\\\hline
categorical features & ---
\\\hline
numerical features & \emph{Age}, \emph{Gender}, \emph{Education}, \emph{Country}, \emph{Ethnicity}, \emph{Nscore}, \emph{Escore}, \emph{Oscore}, \emph{Ascore}, \emph{Cscore}, \emph{Impulsive}, \emph{SS}
& (precomputed numeric values in dataset)
\\\hline
protected attribute &\textit{Gender} & female (31.0\%), male (69.0\%)
\\\hline
target variable & \textit{Coke} & never used (55.1\%), used (44.9\%)
\end{tabular}
\end{subfigure}\\[2.\baselineskip]
\begin{subfigure}[b]{.95\textwidth}\centering
\caption{\german}\label{tab:info-german}
\begin{tabular}{p{.18\textwidth}|p{.3\textwidth}|p{.5\textwidth}}
dataset size & 1000
\\\hline
categorical features & \emph{Age} &  values: $\leq$ 24, 25--34, 35--44, 45--54, 55--64, $\geq$ 65 \\
                     & \emph{Saving accounts} & little, moderate, quite rich, rich \\
                     & \emph{Checking account} & little, moderate, rich
\\\hline
numerical features & \emph{Duration}, \emph{Credit amount}
\\\hline
protected attribute &\textit{Sex} & female (31.0\%), male (69.0\%)
\\\hline
target variable & \textit{Risk} & bad (30\%), good (70\%)
\end{tabular}
\end{subfigure}\\[2.\baselineskip]
\begin{subfigure}[b]{.95\textwidth}\centering
\caption{\folktables}\label{tab:info-folktables}
\begin{tabular}{p{.18\textwidth}|p{.3\textwidth}|p{.5\textwidth}}
dataset size & 255078
\\\hline
categorical features & \emph{AGE} (age; binned) &  values: $\leq$ 14, 15--24, 25--34, 35--44, 45--54, 55--64, $\geq$ 65 \\
                     & \emph{COW} (class of worker) & values: $1,\dots,9$\\
                     & \emph{SCHL} (education) &  values: $1,\dots,24$\\
                     & \emph{MAR} (marital status) & values: married, widowed, divorced, separated, never married\\
                     & \emph{OCCP} (occupation code) & values: $0,1,\dots,9$\\
                     & \emph{POBP} (place of birth) & values: USA, other \\
                     & \emph{RELP} (relationship in household) &  values: $0,1,\dots,17$\\
                     & \emph{WKHP} (weekly working hours; binned) & values: $\leq$ 19, 20-29, 30-39, $\geq$ 40 \\
                     & \emph{RAC1P} (race code) & values: $1,\dots,9$
\\\hline
numerical features & --- 
\\\hline
protected attribute &\textit{SEX} & female (52.1\%), male (47.9\%)
\\\hline
target variable & \textit{income} & $\leq 50K$ (64.8\%), $>50K$ (35.2\%) 
\\[.5\baselineskip]
\multicolumn{3}{l}{For details of the numeric codes, see \tiny \url{https://www2.census.gov/programs-surveys/acs/tech_docs/pums/data_dict/PUMS_Data_Dictionary_2018.pdf}}
\end{tabular}
\end{subfigure}
\end{table}

\subsection{Training objectives}\label{subsec:training_objectives}
All training objectives are derived from logistic regression classifiers.
For data $S=\{(x_1,y_1),\dots,(x_n,y_n)\}\subset\R^{d}\times\{\pm 1\}$ we 
learn a prediction function $g(x)=w^\top x + b$ by solving 
\begin{align}
\min_{w\in\mathbb{R}^d,b\in\mathbb{R}} & \mathcal{L}_S(w,b) + \lambda\|w\|^2
\intertext{with}
\mathcal{L}_S(w,b) &= \frac{1}{|S|}\sum_{(x,y)\in S} y\log(1+e^{-g(x)}) + (1-y)\log(1+e^{g(x)})
\end{align}
We use the \texttt{LogisticRegression} routine of the \texttt{sklearn} package
for this, which runs a LBFGS optimizer for up to 500 iterations. 
By default, we do not use a regularizer, \ie $\lambda=0$. 
From $g(x)$ we obtain classification decisions as $f(x)=\operatorname{sign}g(x)$ 
and probability estimates as $\sigma(x;w,b)=p(y=1|x)=\frac{1}{1+e^{-g(x)}}$,
where we clip the output of $g$ to the interval $[-20,20]$ to avoid
numeric issues. 

To train with fairness regularization, we solve the optimization problem
\begin{align}
\min_{w\in\mathbb{R}^d,b\in\mathbb{R}}  &\mathcal{L}_S(w,b) + \eta |\Gamma_{S}(w,b)|_{\epsilon}  \label{eq:logreg-fairnessreg}
\intertext{with}
\Gamma_{S}(w,b) &= \frac{1}{|S^{a=0}|}\!\sum_{x\in S^{a=0}}\sigma(x;w,b) - \frac{1}{|S^{a=1}|}\!\sum_{x\in S^{a=1}}\sigma(x;w,b),
\end{align}
where for reasons of numeric stability, we use $|t|_{\epsilon}=\sqrt{\frac{t^2}{t^2+\epsilon}}$ with $\epsilon=10^{-8}$. 
To do so, we use the \texttt{scipy.minimize} routine with \texttt{bfgs}
optimizer for up to 500 iterations. 
The necessary gradients are computed automatically using \emph{jax}.\footnote{\url{https://github.com/google/jax} \textcolor{black}{(version 0.3.14)}}
To initialize $(w,b)$, we use the result of training a (fairness-unaware) 
logistic regression with $\lambda=1$, where the regularization is meant 
to ensure that the parameters do not take too extreme values. 
When estimating the disparity, we use the same objective, but 
with different datasets, $S_1,S_2$ for the two terms in \eqref{eq:logreg-fairnessreg},
with the protected attributes as target labels for $S_1$,
and the inverse of the protected attributes as target labels for $S_2$.

To train with adversarial regularization, we parameterize an adversary 
$g':\R\to\R$ as $g'(x')=w'x'+b'$ and solve the optimization problem 
\begin{align}
\min_{w\in\mathbb{R}^d,b\in\mathbb{R}}
\max_{w'\in\mathbb{R},b'\in\mathbb{R}}
  &\mathcal{L}_S(w,b) - \eta \mathcal{L'}_{S}(w',b')  \label{eq:logreg-fairnessadv}
\intertext{with}
\mathcal{L'}_{S}(w,b,w',b') &= \frac{1}{|S|}\sum_{(x,a)\in S} a\log(1+e^{-g'(g(x))}) + (1-a)\log(1+e^{g'(g(x))})
\end{align}
To do so, we use the \texttt{optax} package with gradient updates 
by the Adam rule for up to 1000 steps. The learning rates for classifier
and adversary are $0.001$. The gradients are again computed using \emph{jax}. 
We initialize $(w,b)$ the same way as for~\eqref{eq:logreg-fairnessreg}.
$(w',b')$ we simply initialize with zeros.

To perform score postprocessing, we evaluate the linear prediction 
function on the training set and determine the thresholds that result 
in a fraction of $r\in\{0,0.01,\dots,0.99,1\}$ positive decision 
separately for each protected group. 
For each $r$ we then compute the overall accuracy of the classifier 
that results from using these group-specific thresholds and select 
the value for $r$ that leads to the highest accuracy.
We then modify the classifier to use the corresponding thresholds 
for each group by adjusting the classifier weights of the protected 
attributes. 

\subsection{Baselines}
In this section, we provide more details about the baselines.

\paragraph{Robust ensemble}
For this baseline, we train $N$ classifiers, one per data source, 
using the respective base learner. 
For prediction, we compute the median value of the predicted 
probabilities and threshold it at $0.5$ to obtain a binary 
label. 
Since in our experiments the number of sources is always odd, 
this is also equivalent to classifying using the majority vote 
rule.

\paragraph{Filtering method from \citet{konstantinov2020sample}}
The method proposed in \citet{konstantinov2020sample} uses a 
filtering step to suppress unreliable sources, like we do, 
but that differs from \method's in two main aspects: 
it uses only the discrepancy score for its decisions, 
and its decision criterion is threshold-based, not 
quantile-based.

For its implementation, one first computes the 
pairwise discrepancy scores, $\disc(S_i,S_j)$, 
between all sources. 
Then, one determines a threshold, 
$t=\sqrt{\frac{8d\log(2en/d) +8\log(8N/\delta)}{n}}$,
where $d$ is the VC dimension of the hypothesis class 
(for us: the dimensionality of the feature vectors plus 1). 
$\delta$ is a freely choosable confidence parameter. In 
the limited data regime of our experiments, its value has little 
influence on the threshold, so we leave it at a default of $\delta=0.1$. 
Finally, for each source, $S_i$, we check for how many 
other sources, $S_j$, their pairwise discrepancy to $S_i$ 
is less than $t$ (\ie $\sum_{j\neq i}\bbmone\{\disc(S_i,S_j)<t\}$). 
If the number of such sources is at least $K-1$, the 
source $S_i$ is made part of the overall training set, 
otherwise is it discarded.

One can check that in the setting of our experiments, 
only for the \adult\ dataset one obtains values for $t$ 
substantially below $1$. Therefore, only for this dataset, 
the filtering step can have a non-trivial effect. 

\paragraph{DRO method from \citet{wang2020robustfairnessnoise}}
The DRO method was proposed originally for the \emph{equal opportunity}
or \emph{equalized odds} fairness measures. We adapt it to 
\emph{demographic parity} by imposing constraints on the 
fraction of positive decisions instead of the true and 
false positive rates. 

Our implementation follows the publicly available github repository,\footnote{\url{https://github.com/wenshuoguo/robust-fairness-code}}
which implements an approximate version of the method described
in the publication. 
The main step is learning a classifier with fairness constraints.
This is implemented by deriving a Lagrangian objective and performing 
simultaneous gradient descent on the classifier parameters
and gradient ascent on the Lagrange multipliers.
This construction has one hyperparameter, $\xi$, the permitted 
slack up to which the constraints have to be fulfilled.
We set this adaptively, starting with a small value $\xi=0.01$,
but then doubling $\xi$ until the optimization results in a 
non-degenerate solution (\ie not a constant classifier).

Additionally, the constraint term of the objective is optimized 
in a distributionally robust (DRO) way. For this, sample weights are 
introduced, and the Lagrangian term is maximized also with respect 
to these weights, subject to $L^1$-ball constraints around uniform 
weights, and $L^1$-simplex constraints to ensure that the weights 
encode a discrete probability distribution. 
Following the original code, we use a projected gradient algorithm
for the ball constraint, while the simplex constraint is approximated
by implicit renormalization.
The DRO also has one hyperparameter, $s$, the radius of the $L^1$-ball.
Following the derivation in the original work, we set this to twice
the maximal total variation distance between the data distribution 
of the protected attribute in the original data and in the 
manipulated data, which in our case is $s=2(1-\alpha)$.

Additional hyperparameters are the learning rates for the classifier
itself, for the Lagrangian multipliers, and for the sample weights. 
After some initial sanity checks we keep these at the values that worked
best in the original publication, which is $0.01$ in all three cases.

\paragraph{hTERM method from \citet{li2021tilted}} 
\emph{TERM (tilted empirical risk minimization)} 
learns a classifier by minimizing an exponentially weighted loss, 
$\frac{1}{t}\log\big(\frac{1}{|S|}\sum_{(x,y)\in S} e^{t \ell(y,f(x))}\big)$,
instead of the standard uniform average of losses over all samples.
For negative values of $t$, this expression acts as a \emph{softmin}, 
thereby encouraging \emph{robustness} in the sense that hard-to-classify 
outliers will be ignored.
For positive values of $t$, the effect is of a \emph{softmax}, 
which encourages \emph{fairness} in the sense that all loss
values should be comparably large.
For our experiments, we use TERM's hierarchical group-based extension (hTERM):
an outer \emph{softmin}-loss encourages robustness across 
sources, while an inner per-source \emph{softmax}-loss 
enforces \emph{fairness} across protected groups,
\begin{align}
\mathcal{L}(f) &= \frac{1}{t}\log\big(\frac{1}{N}\sum_{i=1}^N n_i e^{t R_i(f)}\big)
\quad\text{with}\quad
R_i(f) = \frac{1}{\tau}\log\big(\frac{1}{2}\!\!\!\sum_{z\in\{0,1\}} e^{\tau R^{a=z}_i(f)}\big),
\intertext{where}
R^{a=z}_i(f) &= \frac{1}{|S^{a=z}_i|}\!\!\sum_{(x,y)\in S^{a=z}_i}\!\!\!\ell(y,f(x)).
\end{align}
Following the original manuscript, we use $t=-2$ and $\tau=2$. 
To numerically solve the resulting optimization problem, 
we use the \emph{binary cross-entropy} as the loss 
function, $\ell$, and we call \texttt{sklearn}'s 
\texttt{minimize} routine with LGFBS optimization.

\subsection{Computing resources}\label{subsec:computing_resources}

All experiments were run on CPU-only compute servers. 
For each train/test split of each dataset and each 
adversary, one experimental run across all baseline 
learning methods takes between 3 minutes and 3 hours 
on two CPU cores, depending on the number of sources, 
the size of the data sets, and the CPU architecture. 
The time needed for each row in the ablation study is 
similar, except for the \folktables data, which each 
took 4-6 hours. 
The combined time for all reported experiments with linear 
classifiers (5 datasets, 12 adversaries, 10 train-test splits, 
5 base learners) is approximately 1800 core hours. 
The experiments with nonlinear classifiers required approximately 
500 times longer per setting, most of which is spent on 
cross-validation of the hyperparameters. 

For the baselines we are able to reuse many already computed 
parts. If implemented individually, we'd estimate that the 
robust ensemble would be the fastest to train, but it is
slower than the other methods at prediction time. 
\textcolor{black}{hTERM would also be efficient to train, 
as it only requires learning one classifier on the combined 
training data.} The training time for \citet{konstantinov2020sample} 
and the DRO method would be comparable to \method's.

\subsection{Hyperparameters}
We avoid hyperparameter tuning as far as possible.
We do not use $L^2$-regularization (hyperparameter $\lambda$) 
except to create initializers, where we found the value 
used to hardly matter. 
For the fairness-regularizer and fairness-adversary we 
use fixed values of $\eta=\frac12$.
We found these to result in generally fair classifiers 
for unperturbed data without causing classifiers to 
degenerate (\ie become constant).
Hence we, did not tune these values on a 
case-by-case basis. 
When estimating the disparity, we use $\eta=1$ 
to be consistent with the theory. 

As learning rate for the adversarial fairness 
training, $\text{lr}_{\text{adv}}=0.001$ was found 
by trial and error to ensure convergence at a
reasonable speed. 
Once we identified a reliably working setting, 
we did not try to tune it further. 

\subsection{Adversaries}\label{subsec:adversariesdetails}
In this section, we describe the adversaries and
their motivation in more detail.

\begin{itemize}
\item \emph{flip protected (FP)}: the adversary flips the 
value of protected attribute.

This is a straightforward attack on fairness.
FP inverts the correlation between the protected 
attribute and the rest of the data %
After the sources have been combined, 
the correlation is therefore weakened, which makes 
the training data look "less unfair". On the one hand, 
this can cause fairness-enforcing mechanisms as used, 
\eg, in postprocessing fairness, to erroneously believe 
that little or no compensation for dataset unfairness is 
required. Consequently, the resulting classifier is actually 
unfair when applied to future unmanipulated data. 
On the other hand, it is possible that the training process 
actually learns to ignore the protected attribute during 
training, because it is uncorrelated with the target labels.
This could make the classifier more fair, \eg when used 
with fairness-unaware training.

Our detailed experimental results (Fig.~\ref{extrafig:summaryresults_first}
--\ref{extrafig:summaryresults_folktables_last}) show that both of these effect do, in 
fact, occur. FP typically increases unfairness when regularization-based 
or postprocessing-based base learners are used, but it has the opposite
effect for the fairness-unaware base learner.

\item \emph{flip label (FL)}: the adversary flips the 
value of the  label.

This is a straightforward attack on accuracy. Following an analog 
reasoning as above, FL reduces the correlation between the 
target label and all other data, which makes it harder for 
the learner to identify a strong classifier. 

Indeed, the experiments shows that the FL adversary often succeeds 
in reducing the accuracy, while the fairness is relatively unaffected. 
The adverse effect is small for the large datasets (\adult, \compas), 
and larger for the small ones (\drugs, \german), presumably because 
having more data increases the robustness of the learners against 
mislabeled data. 

\item \emph{flip both (FB)}: the adversary flips the value 
of the protected attribute and the label.

This attack influences fairness and accuracy at the same time. 
It preserves the correlation between the protected attribute 
and the labels, but reduces the correlation between these
two and all the other features. Consequently, the learned classifier
might rely heavily on the protected attribute to predict the
label, which would make it maximally unfair, but potentially 
also less accurate.

Our experiments show that this is, indeed, 
often the observed effect, though the exact amount depends 
strongly on the dataset and the base learner. 

\item \emph{shuffle protected (SP)}: the adversary shuffles 
the protected attribute entries of each batch it modifies, \ie each
example gets assigned the protected attribute of another example
that has been chosen at random (without replacement).

This adversary is similar to FP in that is reduces the overall 
correlation between the protected attribute and the other data. 
Its effect is weaker, since it does not explicitly introduce
anti-correlation in the manipulated sources. However, its manipulations
are less likely to be detected by automatic or manual inspection,
since it does not change the marginal statistics of 
the data, \ie even after the manipulation, the statistical 
distribution of each feature dimension, including the protected
attribute, is the same as for clean sources. 

In experimental results, SP indeed performs similarly to FP
for the fairness-aware base learners, and its effect are 
somewhat weaker for the fairness-unaware base learner. 

\item \emph{overwrite protected (OP)}:
the adversary overwrites the protected attribute of 
each sample in the affected batch by its label.

This manipulation creates a strong artificial correlation 
between the protected attribute and the target label. 
In fact, the maximally unfair classifier that predicts the 
label directly from the protected attribute will have 
perfect accuracy on the manipulated data, and still a 
much higher accuracy than what would be correct on the 
overall training data. 
Consequently, the learned classifier might make strong 
use of the protected attribute, which leads to unfair 
and potentially incorrect decisions on clean data. 

Our experiments show that OP indeed often leads to 
large increases in unfairness. However, there are also 
cases where the unfairness is actually reduced, but 
then typically this is accompanied by loss of accuracy.

\item \emph{overwrite label (OL)}: the adversary overwrites 
the label of each sample in the affected batch by its 
protected attribute.

Like the OP adversary, this manipulation leads to a 
perfect correlation between the target labels and the 
protected attributes. However, it achieves this 
without changing the marginal distribution of the 
protected attribute, instead influencing
the statistics of the labels. Depending on the specific
situation, it might be easier or harder to detect 
from automatic or manual inspection.
OL is also more likely to negatively affect the 
accuracy, since the classifier will try to 
predict incorrect labels.

The experiments show that OL indeed almost always
reduces the accuracy, while at the same time often
increasing unfairness.

\item \emph{resample protected (RP)}: the adversary 
samples new batches of data in the following ways: 
all original samples of protected group $a=0$ with 
labels $y=1$ are replaced by data samples from other 
sources which also have $a=0$, but $y=0$. 
Analogously, all samples of group $a=1$ with labels 
$y=0$ are replaced by data samples from other sources 
with $a=1$ and $y=1$. 

Like OL and OP, RP results in a perfect correlation 
between protected attributes and labels, thereby 
facilitating unfairness and reducing accuracy. 
It does so in a more subtle and harder-to-detect way, 
however, as it achieves the effect using original data 
samples.

Indeed, in our experimental results RP influences 
fairness and accuracy in similar ways as the other
two methods. 

\item \emph{random anchor (RA0/RA1)}: 
these adversaries follow the protocol introduced in \citet{mehrabi2020exacerbating}.
RA0 first picks a random \emph{anchor}
example $x_{\text{target}}^-$ of group $a=1$ with label $y=0$
from the target source. 
It then creates a group, $\mathcal{G}_+$, of poisoned data by 
constructed new examples within a feasible set that also have $a=1$ 
and are close to $x_{\text{target}}^-$, but that have label $y=1$. 
The number of samples in $\mathcal{G}_+$ matches the number of 
samples in the target source with $a=1$. 
Subsequently, the adversary repeats the above procedure for group 
$a=0$, but with the opposite label values, resulting in a second 
group of poisoned samples $\mathcal{G}_-$. 
Both poisoned sets are then merged to yield a manipulated source 
that is meant to influence the decision boundary near the 
anchor points in a maximally unfair way. 
The adversary RA1 performs the same construction 
as RA0, but with the roles of $a=0$ and $a=1$ 
exchanged.

Given that our data sources mostly have categorical 
features, it is not possible to create realistic-looking 
new samples simply by small random perturbations. 
Instead, we define as feasible set the set of all 
samples that occur in any of the original training 
sources. 
As newly 'constructed' samples we then take those
examples with smallest Euclidean distance to the 
anchors. 

\item \emph{random (RND)}: the adversary randomly picks 
one of the strategies above (except ID) for each source.

This adversary reflects the observation that 
different sources might be manipulated in different
ways. One reason for this could be that in a real-world 
system, multiple adversaries exists who manipulate 
individual data sources without coordinating their 
actions.
Alternatively, there might be just one adversary who 
manipulates all sources, but chooses to manipulate 
them in different ways, \eg to avoid easy detection.

The experimental results show that this strategy does, 
indeed, work to some extent, with RND often having an effect 
where some of the other methods do not, but the effect is weaker.

\item \emph{identity (ID)}: the adversary makes 
no changes to the data. 

The ID adversary serves as a useful check that FLEA does not 
damage the learning process in the case that all data 
is actually clean.
It also reflects the fact that even though 
the adversary has the power to manipulate the data 
it does not have to.
Ideally, the learning method will notice this and
achieve even better results in presence of the ID 
adversary than for the \emph{oracle}. 

In the experimental results, this is effect is 
only rarely visible for any method, though.
\end{itemize}

Note that even though we introduced the adversaries above 
as intentional manipulations, many of them could also occur 
accidentally when data from different sources is collected, 
\eg as problems during data entering or numeric encoding. 

\section{Detailed algorithm for estimating $\disc$ and $\disp$}\label{sec:detailedalgo}

In this section we provide pseudocode for estimating the pairwise 
discrepancy and disparity between two data sources.
In both cases we approximate a classifier $f$ that maximizes 
a continuous relaxation of the relevant metric, and then 
estimate the actual quantity of interest from it. 
The $\disc$-maximizing classifier is trained by flipping the 
labels of one of the two sources, combining the sources into 
a single dataset, and then training a classifier to predict 
the (new) label.
The $\disp$-maximizing classifier is trained as a classifier 
that comes as close as possible to predicting the protected 
attribute $a$ on one data source, and $1-a$ on the other, while
balancing the loss from each subgroup.

\begin{minipage}{0.99\textwidth}
\begin{algorithm}[H]
\caption*{\textsc{Empirical Discrepancy Estimation}}%
\begin{algorithmic}[1]
\REQUIRE datasets $S_1,S_2$
\STATE $f \leftarrow \min_{f} \Big( \frac{1}{|S_1|}\sum\limits_{(x, y) \in S_1} \textsc{CrossEntropy}(f(x),y) + \frac{1}{|S_2|} \sum\limits_{(x, y) \in S_2} \textsc{CrossEntropy}(f(x),(1-y)) \Big)$
\STATE disc $\leftarrow \Big | \frac{1}{|S_1|}\sum\limits_{x\in S_1} \bbmone\{ (f(x) \geq 0.5)\neq y\} - \frac{1}{|S_2|}\sum\limits_{x\in S_2} \bbmone\{ (f(x) \geq 0.5) = y\} \Big | $
\ENSURE Empirical discrepancy estimate disc $\in \mathbb{R}^+$
\end{algorithmic}
\end{algorithm}
\end{minipage}
\vskip-.5\baselineskip
\begin{minipage}{0.99\textwidth}
\begin{algorithm}[H]
\caption*{\textsc{Empirical Disparity Estimation} (Demographic Parity)}%
\begin{algorithmic}[1]
\REQUIRE datasets $S_1,S_2$
\STATE $f \leftarrow \min_{f} 
\Big( \frac{1}{|S^{a=0}_1|}\sum\limits_{x\in S^{a=0}_1} \textsc{CrossEntropy}(f(x),0) 
         + \frac{1}{|S^{a=1}_1|} \sum\limits_{x\in S^{a=1}_1} \textsc{CrossEntropy}(f(x),1) \newline
         \hspace*{2cm} + \frac{1}{|S^{a=0}_2|} \sum\limits_{x\in S^{a=0}_2} \textsc{CrossEntropy}(f(x),1)
         + \frac{1}{|S^{a=1}_2|} \sum\limits_{x\in S^{a=1}_2} \textsc{CrossEntropy}(f(x),0)\Big)$
\STATE disp $\leftarrow \Big |\frac{1}{|S^{a=0}_1|}\sum\limits_{x\in S^{a=0}_1}\bbmone\{ (f(x) \geq 0.5)\}
          - \frac{1}{|S^{a=1}_1|}\sum\limits_{x\in S^{a=1}_1}\bbmone\{ (f(x) \geq 0.5)\} \newline
          \hspace*{2cm}- \frac{1}{|S^{a=0}_2|}\sum\limits_{x\in S^{a=0}_2}\bbmone\{ (f(x) \geq 0.5)\}
          + \frac{1}{|S^{a=1}_2|}\sum\limits_{x\in S^{a=1}_2}\bbmone\{ (f(x) \geq 0.5)\}\Big|$
\ENSURE Empirical discrepancy estimate disp $\in \mathbb{R}^+$
\end{algorithmic}
\end{algorithm}
\end{minipage}

\section{Detailed experimental results}\label{sec:extendedresults}
In addition to the experiments with a regularization-based base learner 
that were reported in the main manuscript, we also run experiments 
with postprocessing-based fairness, preprocessing-based fairness,
adversarial fairness, and fairness-unaware learning.
The results are depicted in Fig.~\ref{extrafig:summaryresults_first}--\ref{extrafig:summaryresults_last}
for the homogeneous setting and in 
Fig.~\ref{extrafig:summaryresults_folktables_first}--\ref{extrafig:summaryresults_folktables_last} 
for the heterogeneous setting. 
Also included are results for two of the baselines, 
robust ensemble and \citet{konstantinov2020sample}.
The DRO~\citep{wang2020robustfairnessnoise} and hTERM~\citep{li2021tilted} 
require specific learning procedures and therefore cannot be 
combined with arbitrary base learners. We report them together 
with results for the regularization-based base learners. %

The format of the figures is as follows: for each datasets
and method, we report the accuracy and fairness results 
for different adversaries. 
Each panel contains 12 bars. The left-most one in each diagram 
shows the result of the hypothetical \emph{oracle} setting, 
where the learning algorithm trains only on the clean data 
sources, \ie the ones which the adversary cannot modify. 
The remaining bars correspond to the outcome when different 
adversaries have perturbed the data.
An ideal robust method should achieve results approximately as 
good as the oracle result, as this would indicate that the 
adversary was indeed not able to negatively affect the 
learning process.

From the results, one can see that \method works almost 
perfectly in the homogeneous setting with a lot of 
data (\adult) and still quite well when the amount 
of data is limited (\compas, \drugs and \german
data). 
In the latter cases, for some adversaries \method does 
not always exactly match the \emph{oracle} results, but 
it still performs better than the baselines.
In the heterogeneous case, \method works reliably in 
all settings, except for the fairness measure when
$N-K=25$, as we had already discussed in the main manuscript.
The results also show that different base learners 
achieve different accuracy/fairness trade-offs, 
but \method is effective with each of them. 
\textcolor{black}{In a few cases, \method's results appear 
to even improve over the ones of the oracle. However, we 
do not believe this to be a systematic effect, but rather 
a case in which the adversarial perturbation were largely 
benign, and \method chooses a subset of sources that by 
random chance yields a better classifier than when using 
exactly the clean sources.}

\begin{figure*}[t]
\caption{\adult\ dataset, $N=5, N-K=2$}
\label{extrafig:summaryresults_first}
\begin{subfigure}[b]{\textwidth}\centering
\caption{regularization-based fairness}\includegraphics[width=\textwidth]{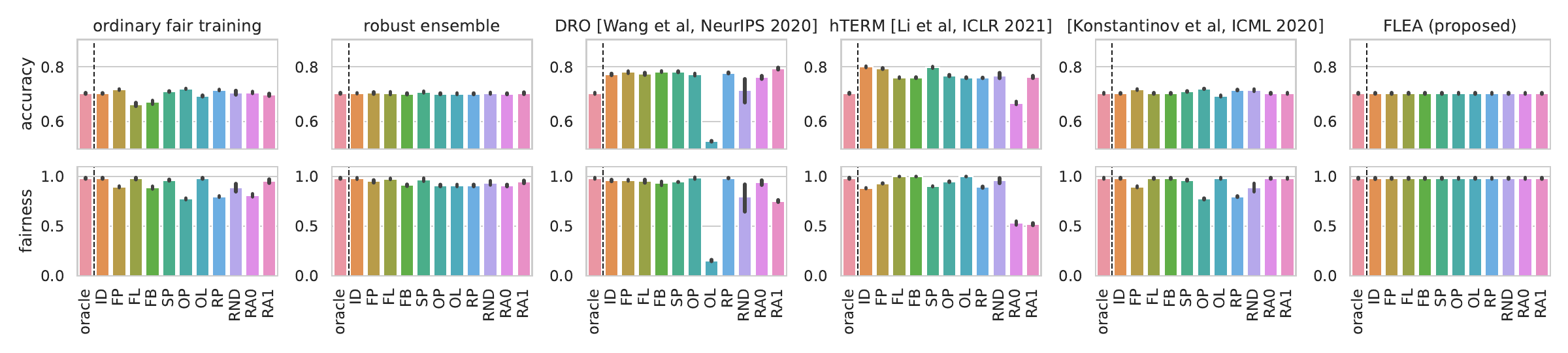}
\end{subfigure}
\begin{subfigure}[b]{\textwidth}\centering
\caption{preprocessing-based fairness}\includegraphics[width=\textwidth]{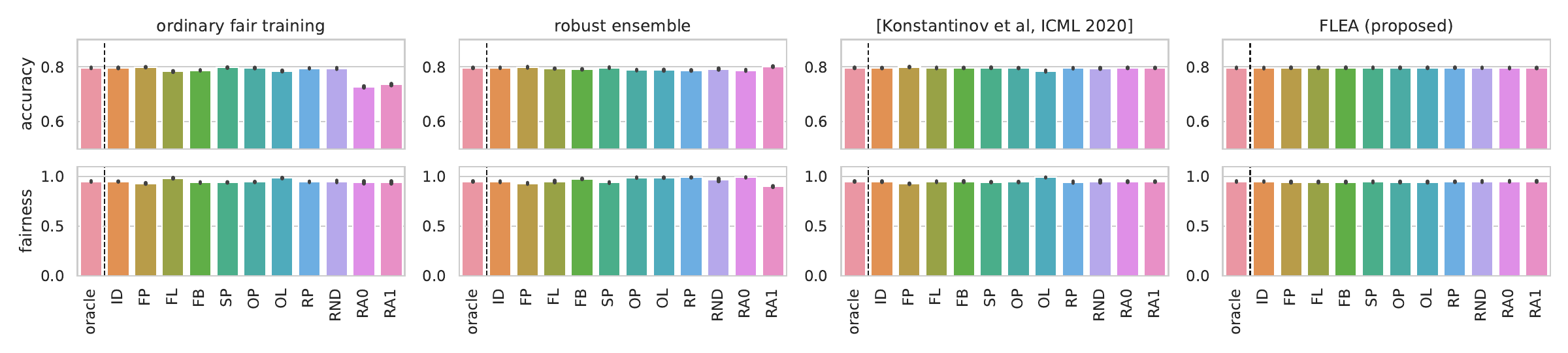}
\end{subfigure}
\begin{subfigure}[b]{\textwidth}\centering
\caption{postprocessing-based fairness}\includegraphics[width=\textwidth]{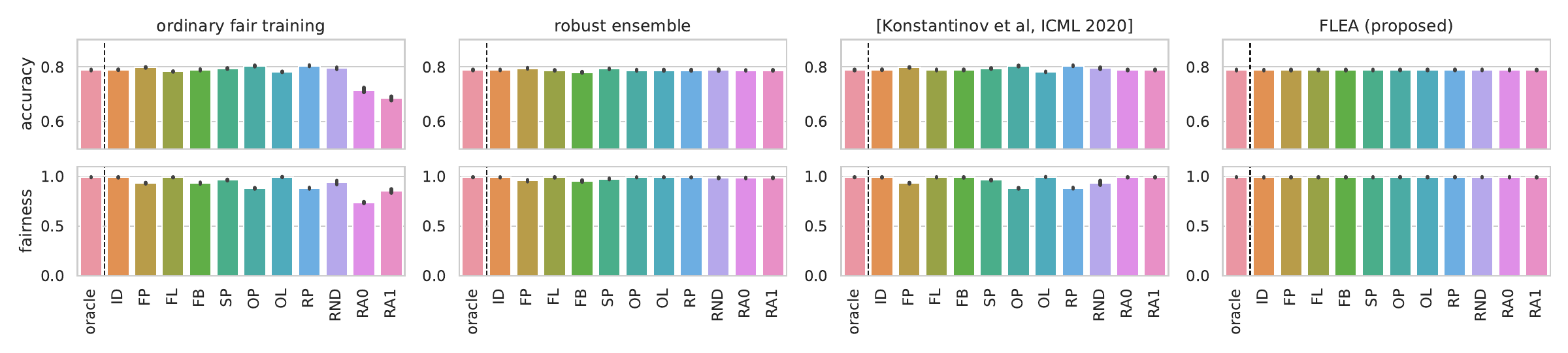}
\end{subfigure}
\begin{subfigure}[b]{\textwidth}\centering
\caption{adversarial fairness}\includegraphics[width=\textwidth]{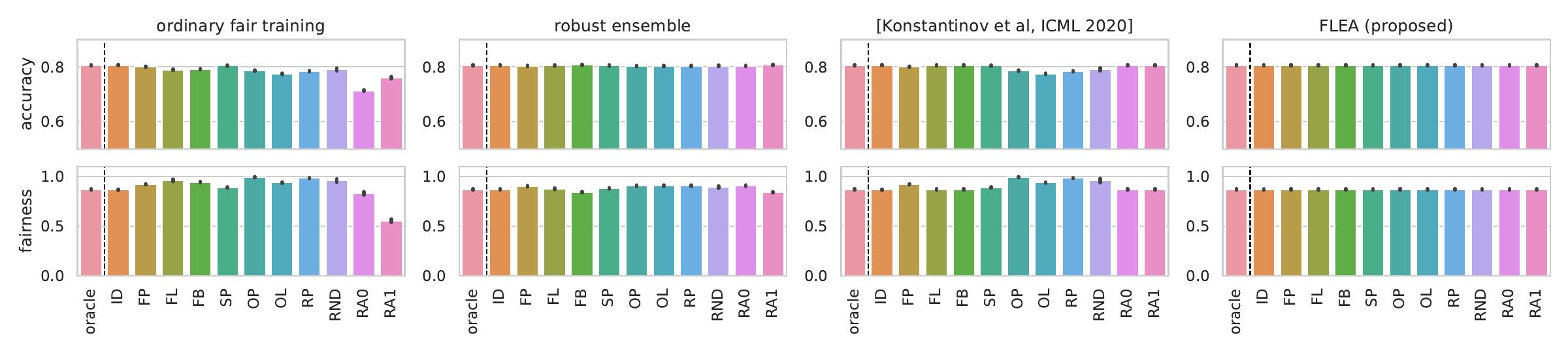}
\end{subfigure}
\begin{subfigure}[b]{\textwidth}\centering
\caption{fairness-unaware}\includegraphics[width=\textwidth]{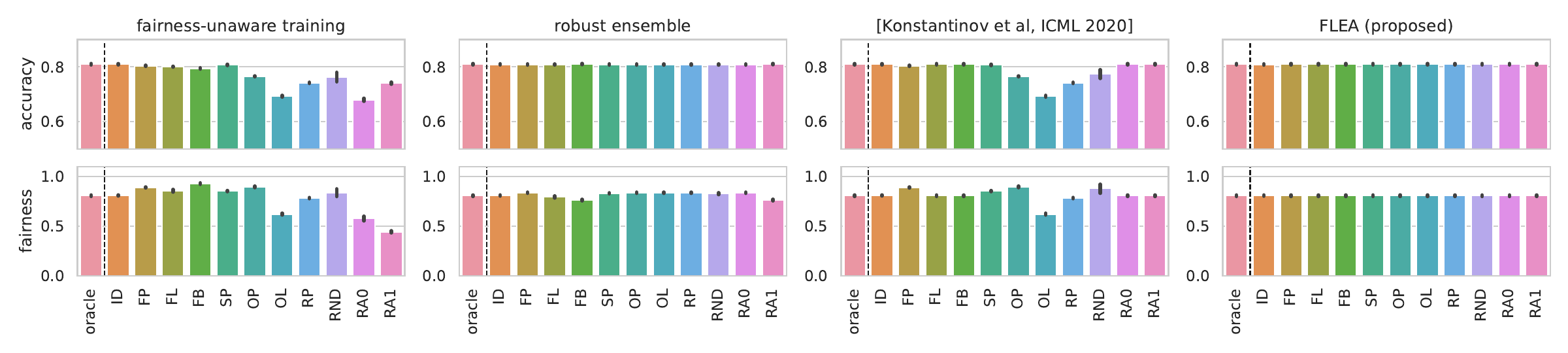}
\end{subfigure}
\end{figure*}

\begin{figure*}[t]
\caption{\compas\ dataset, $N=5, N-K=2$}
\begin{subfigure}[b]{\textwidth}\centering
\caption{regularization-based fairness}\includegraphics[width=\textwidth]{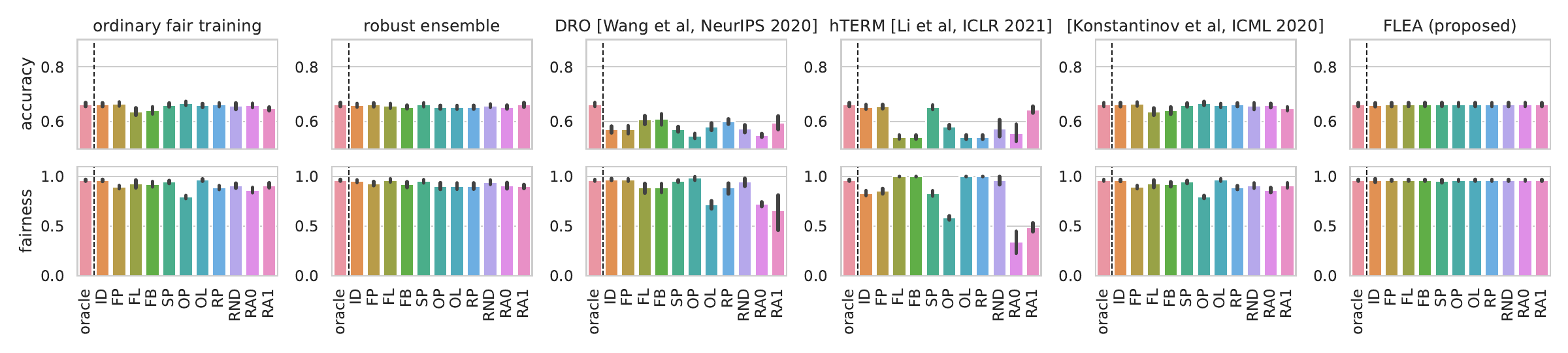}
\end{subfigure}
\begin{subfigure}[b]{\textwidth}\centering
\caption{preprocessing-based fairness}\includegraphics[width=\textwidth]{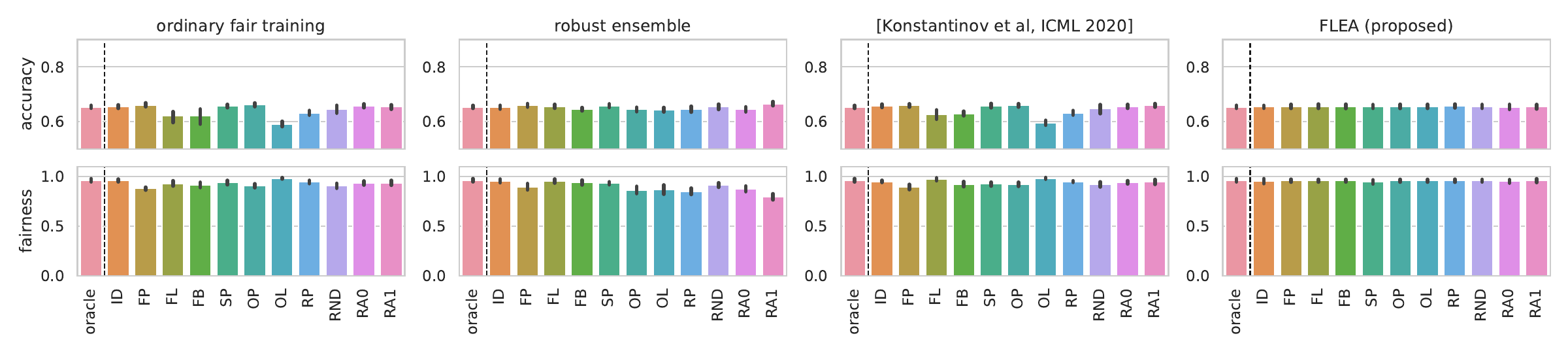}
\end{subfigure}
\begin{subfigure}[b]{\textwidth}\centering
\caption{postprocessing-based fairness}\includegraphics[width=\textwidth]{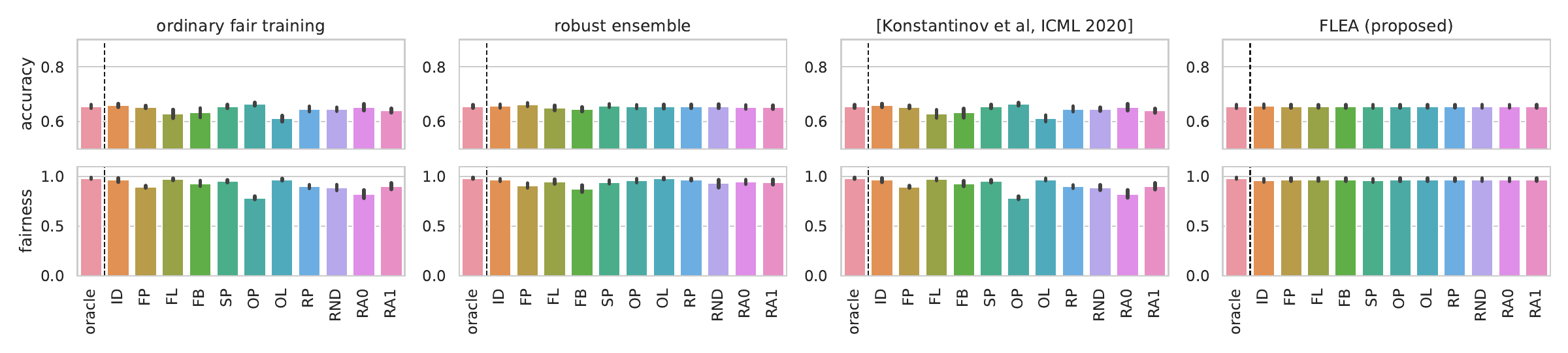}
\end{subfigure}
\begin{subfigure}[b]{\textwidth}\centering
\caption{adversarial fairness}\includegraphics[width=\textwidth]{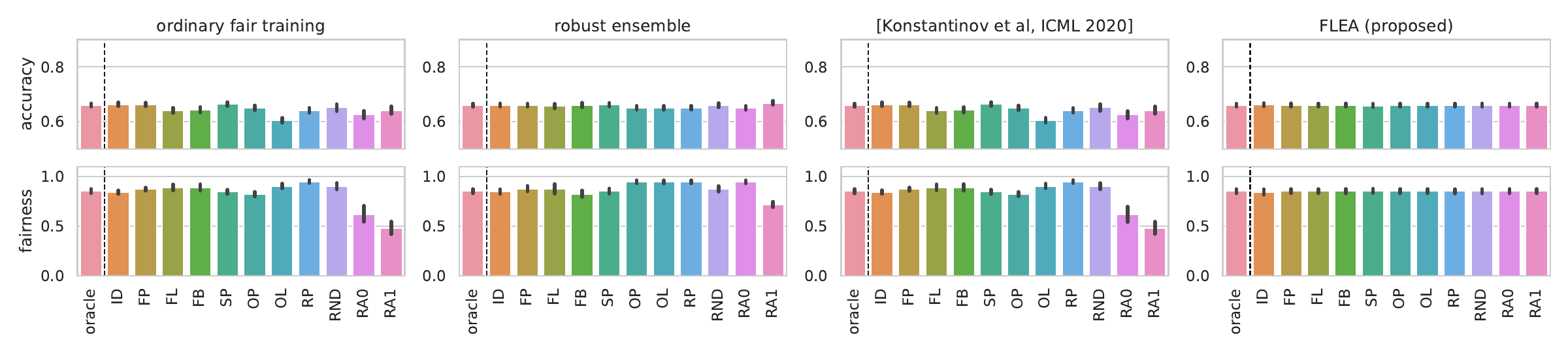}
\end{subfigure}
\begin{subfigure}[b]{\textwidth}\centering
\caption{fairness-unaware}\includegraphics[width=\textwidth]{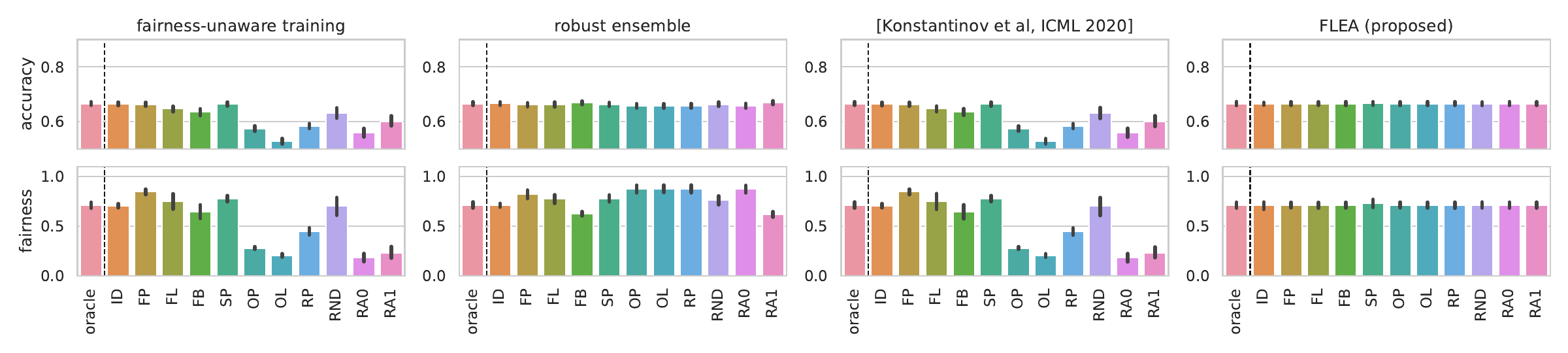}
\end{subfigure}
\end{figure*}

\begin{figure*}[t]
\caption{\drugs\ dataset, $N=5, N-K=2$}
\begin{subfigure}[b]{\textwidth}\centering
\caption{regularization-based fairness}\includegraphics[width=\textwidth]{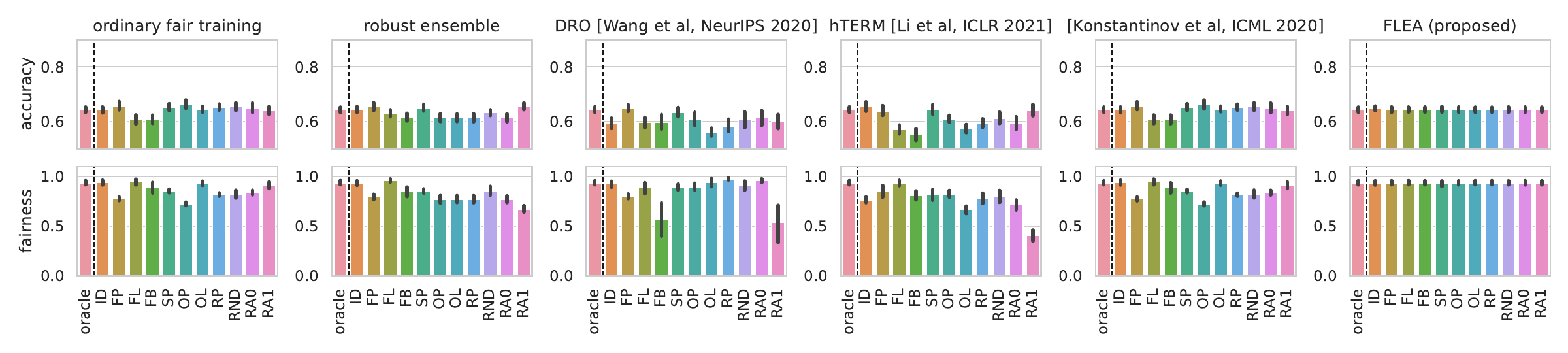}
\end{subfigure}
\begin{subfigure}[b]{\textwidth}\centering
\caption{preprocessing-based fairness}\includegraphics[width=\textwidth]{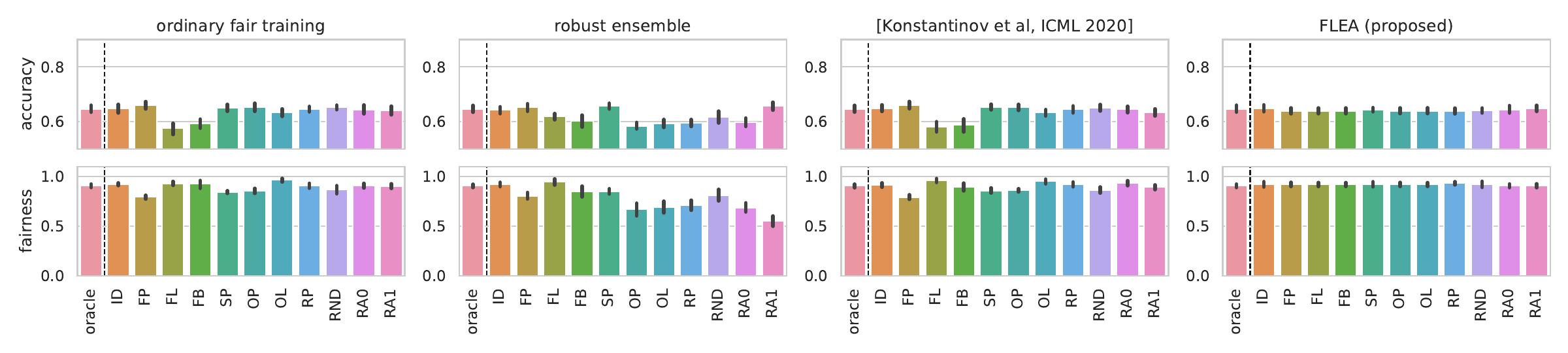}
\end{subfigure}
\begin{subfigure}[b]{\textwidth}\centering
\caption{postprocessing-based fairness}\includegraphics[width=\textwidth]{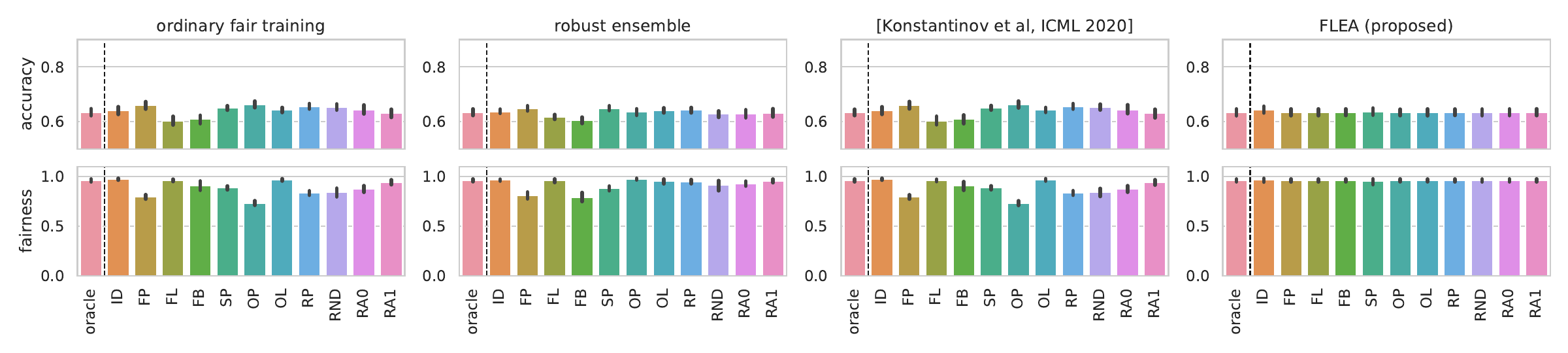}
\end{subfigure}
\begin{subfigure}[b]{\textwidth}\centering
\caption{adversarial fairness}\includegraphics[width=\textwidth]{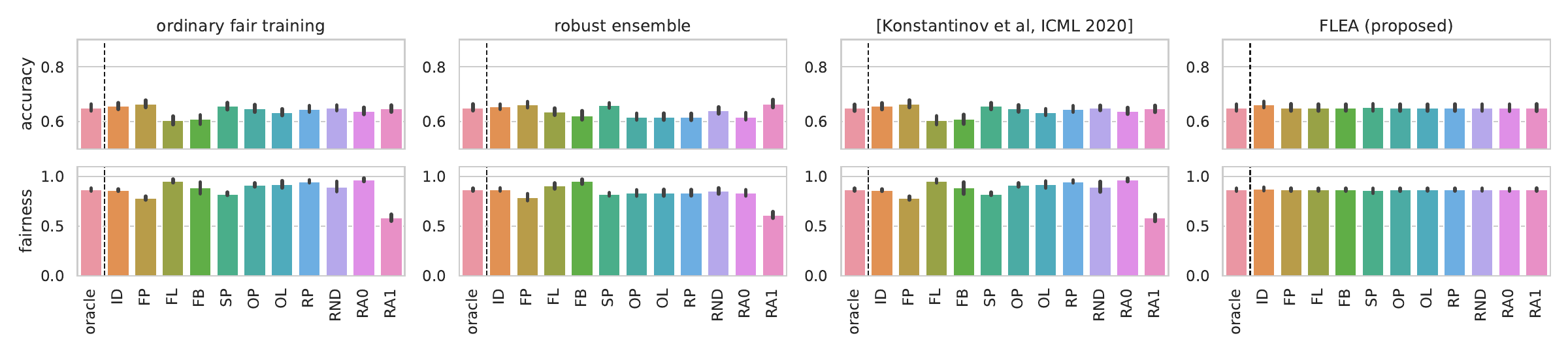}
\end{subfigure}
\begin{subfigure}[b]{\textwidth}\centering
\caption{fairness-unaware}\includegraphics[width=\textwidth]{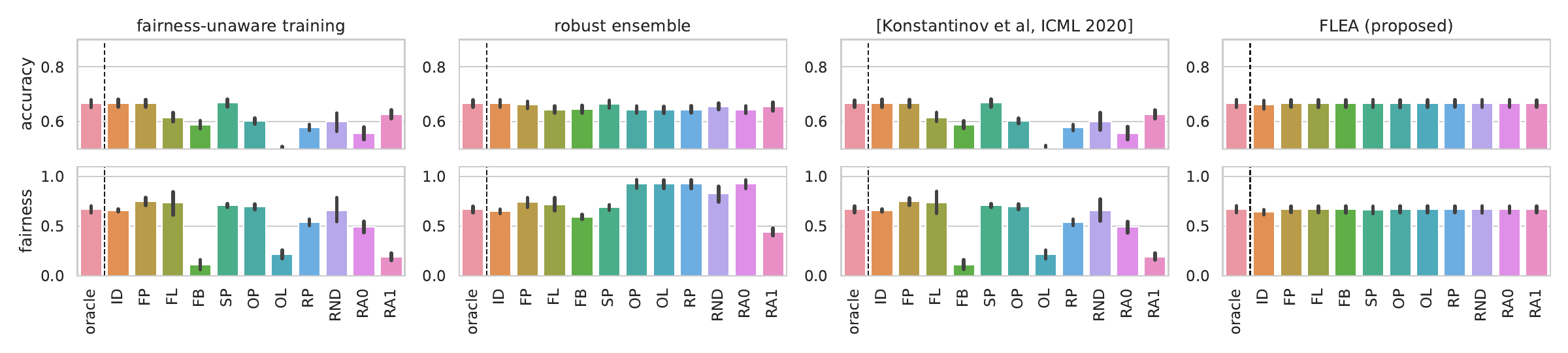}
\end{subfigure}
\end{figure*}

\begin{figure*}[t]
\caption{\german\ dataset, $N=5, N-K=2$}\label{extrafig:summaryresults_last}
\begin{subfigure}[b]{\textwidth}\centering
\caption{regularization-based fairness}\includegraphics[width=\textwidth]{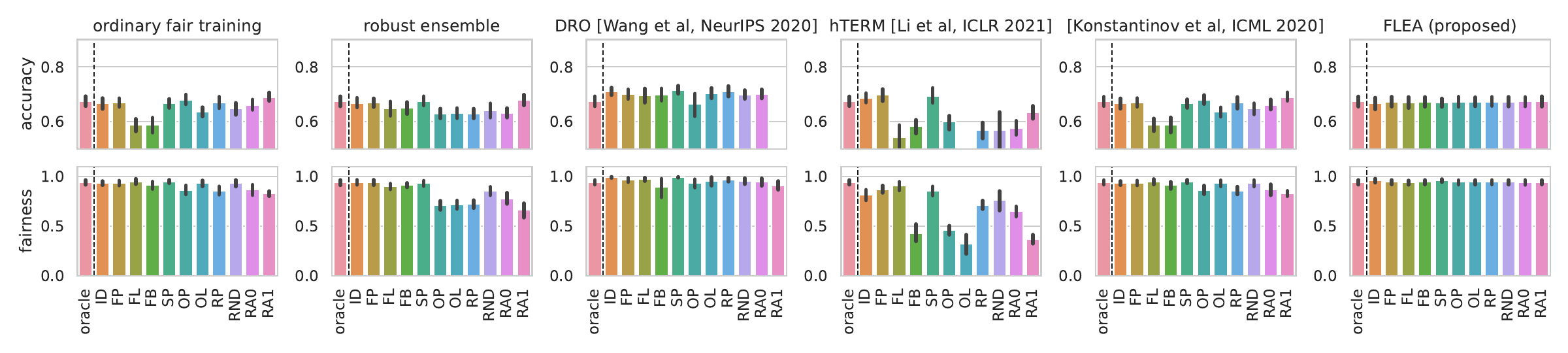}
\end{subfigure}
\begin{subfigure}[b]{\textwidth}\centering
\caption{preprocessing-based fairness}\includegraphics[width=\textwidth]{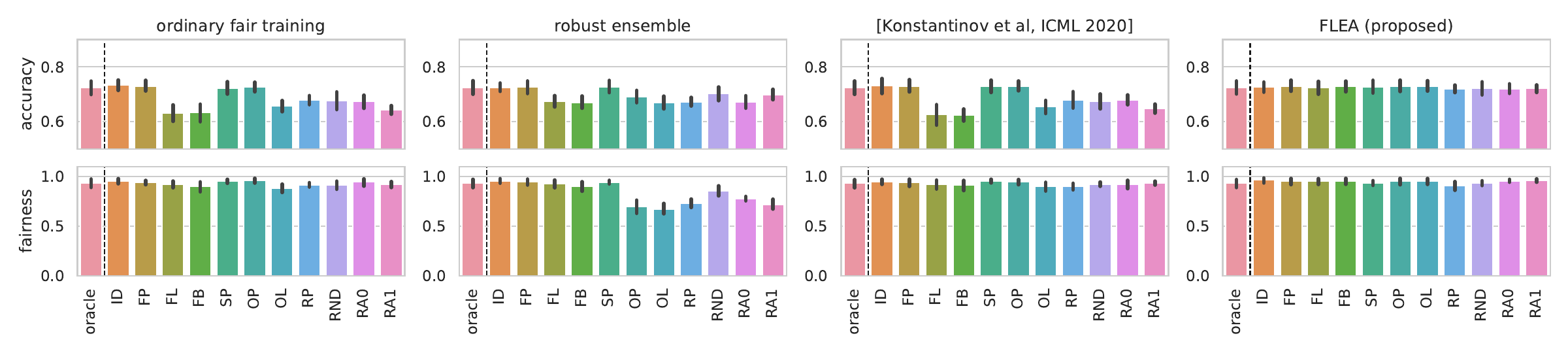}
\end{subfigure}
\begin{subfigure}[b]{\textwidth}\centering
\caption{postprocessing-based fairness}\includegraphics[width=\textwidth]{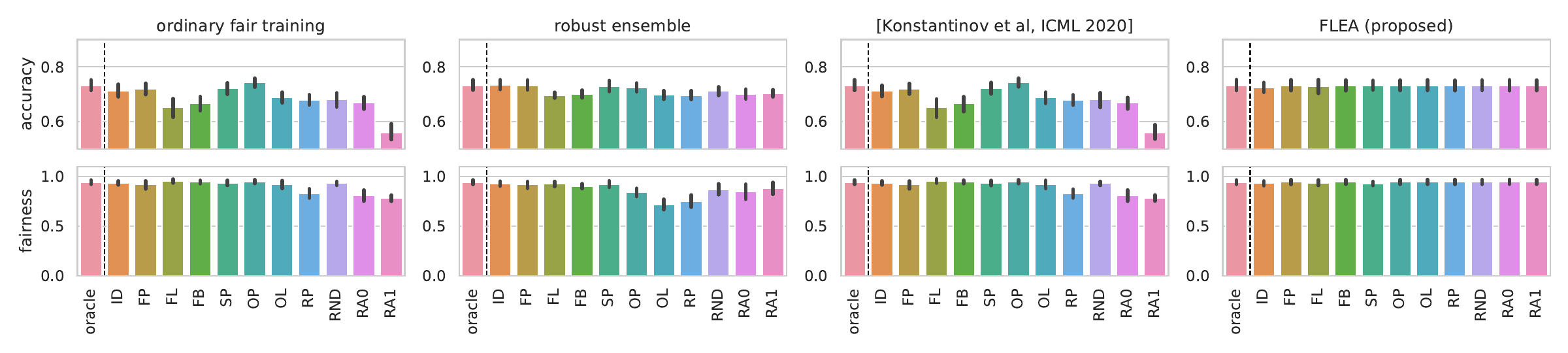}
\end{subfigure}
\begin{subfigure}[b]{\textwidth}\centering
\caption{adversarial fairness}\includegraphics[width=\textwidth]{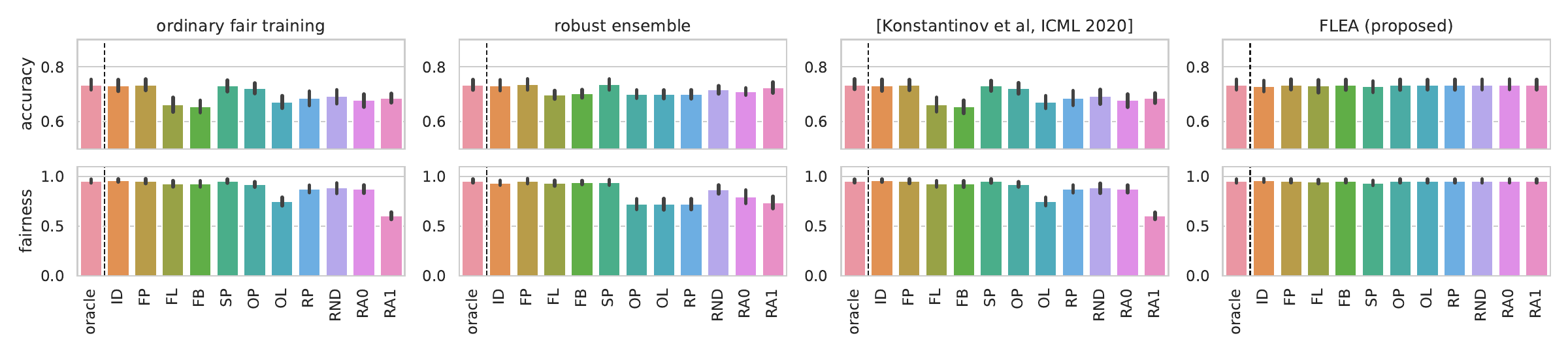}
\end{subfigure}
\begin{subfigure}[b]{\textwidth}\centering
\caption{fairness-unaware}\includegraphics[width=\textwidth]{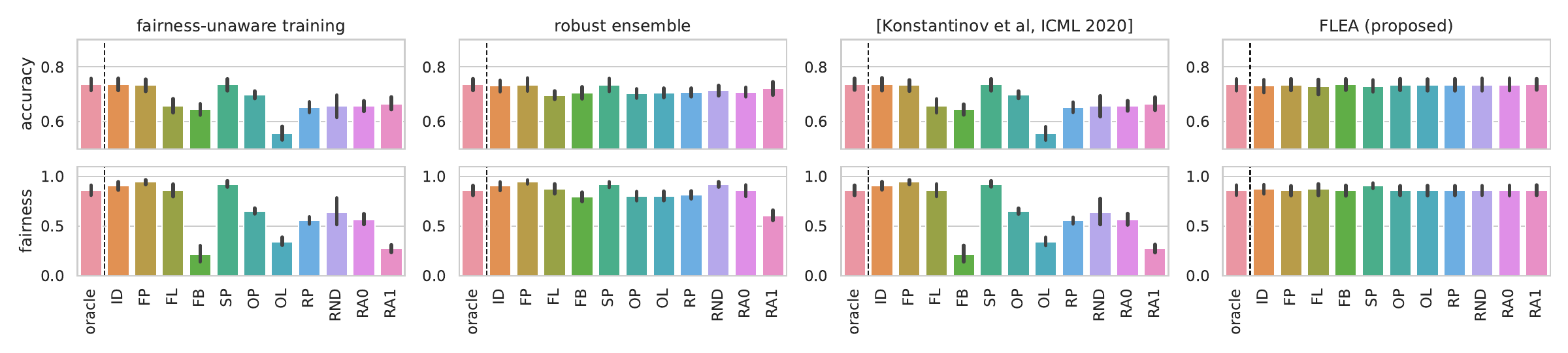}
\end{subfigure}
\end{figure*}

\begin{figure*}[t]
\caption{\folktables\ dataset, regularization-based fairness}
\label{extrafig:summaryresults_folktables_first} 
\begin{subfigure}[b]{\textwidth}\centering
\caption{$N=51, N-K=5$}\includegraphics[width=\textwidth]{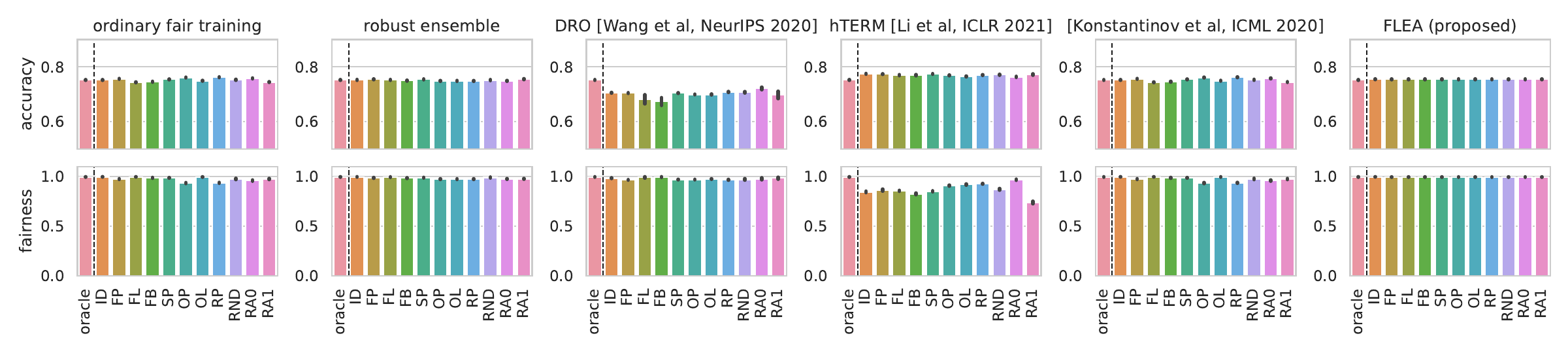}
\end{subfigure}
\begin{subfigure}[b]{\textwidth}\centering
\caption{$N=51, N-K=10$}\includegraphics[width=\textwidth]{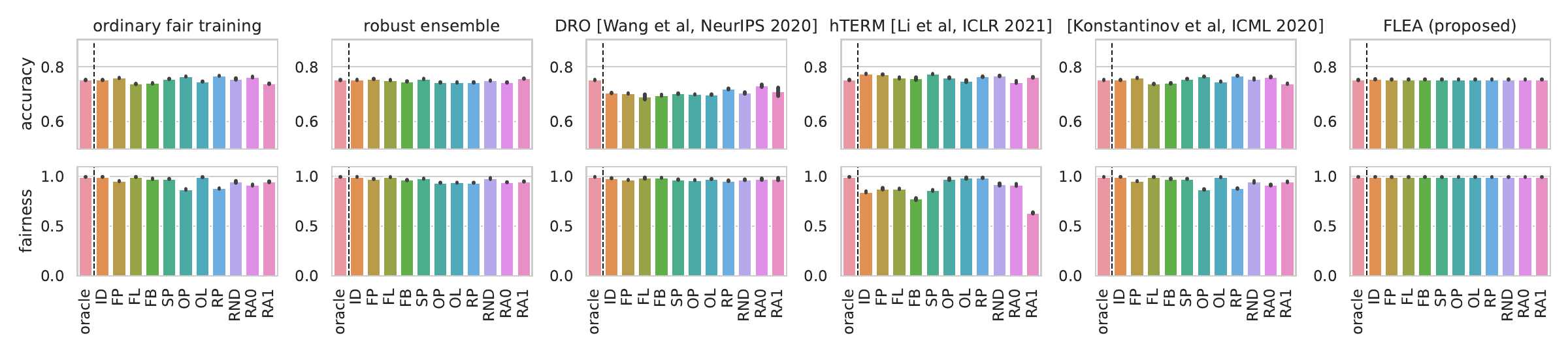}
\end{subfigure}
\begin{subfigure}[b]{\textwidth}\centering
\caption{$N=51, N-K=15$}\includegraphics[width=\textwidth]{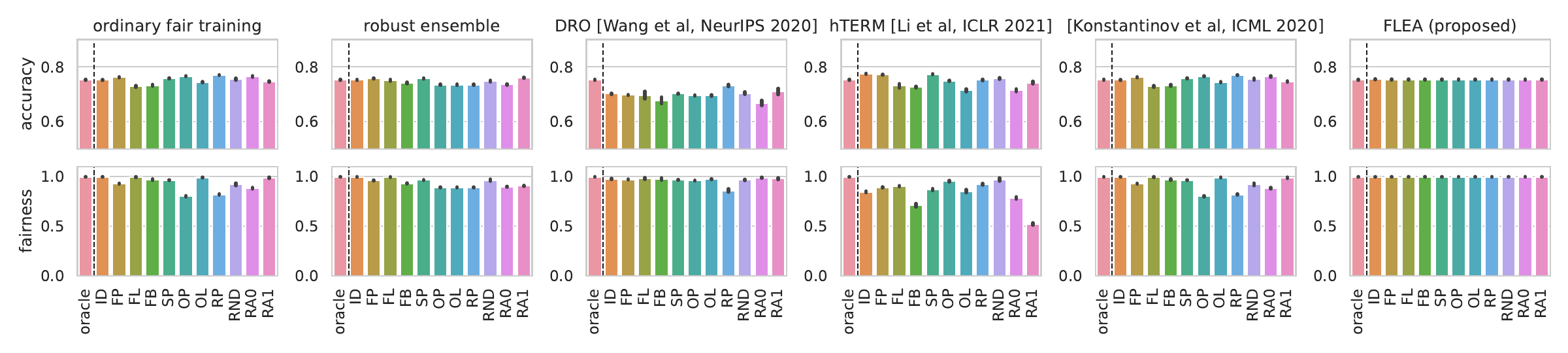}
\end{subfigure}
\begin{subfigure}[b]{\textwidth}\centering
\caption{$N=51, N-K=20$}\includegraphics[width=\textwidth]{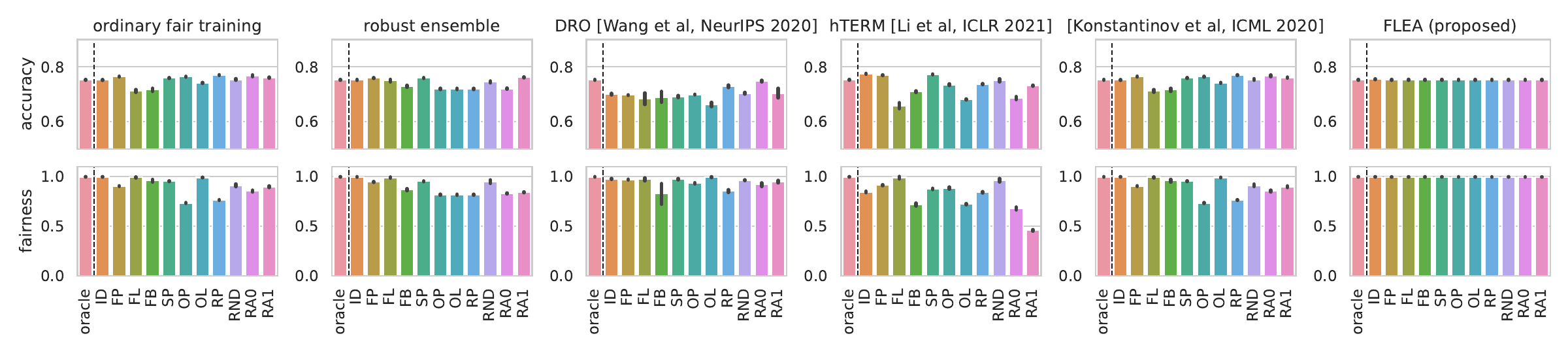}
\end{subfigure}
\begin{subfigure}[b]{\textwidth}\centering
\caption{$N=51, N-K=25$}\includegraphics[width=\textwidth]{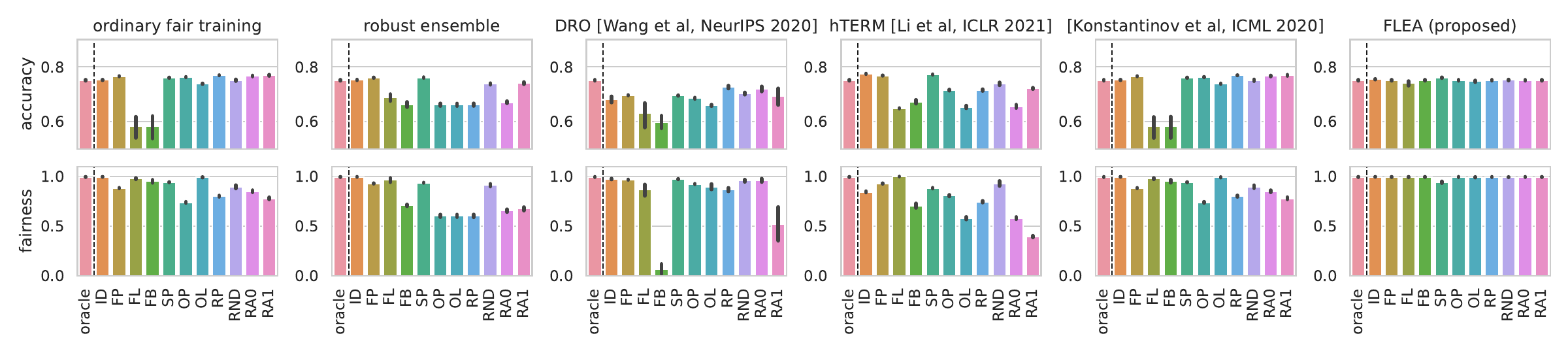}
\end{subfigure}
\end{figure*}

\begin{figure*}[t]
\caption{\folktables\ dataset, preprocessing-based fairness}
\begin{subfigure}[b]{\textwidth}\centering
\caption{$N=51, N-K=5$}\includegraphics[width=\textwidth]{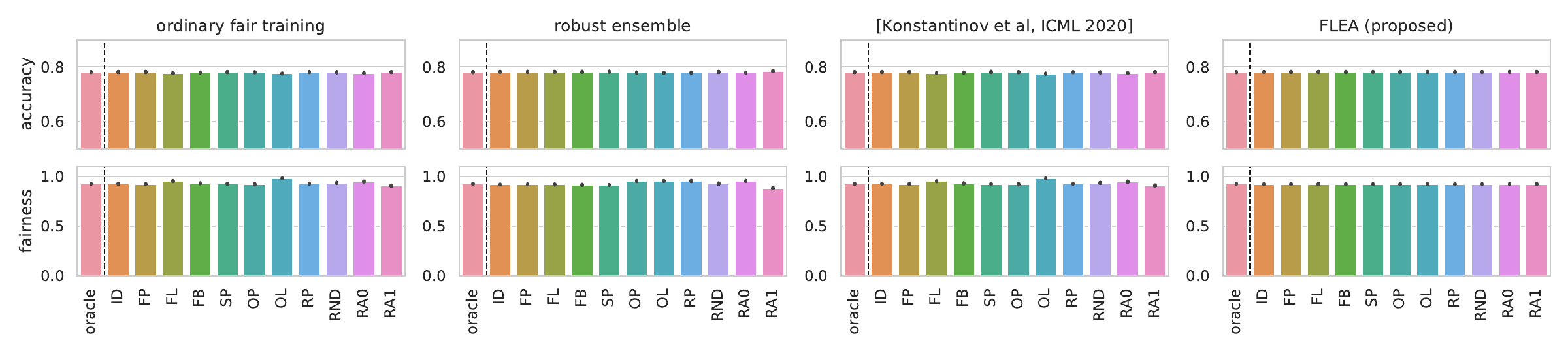}
\end{subfigure}
\begin{subfigure}[b]{\textwidth}\centering
\caption{$N=51, N-K=10$}\includegraphics[width=\textwidth]{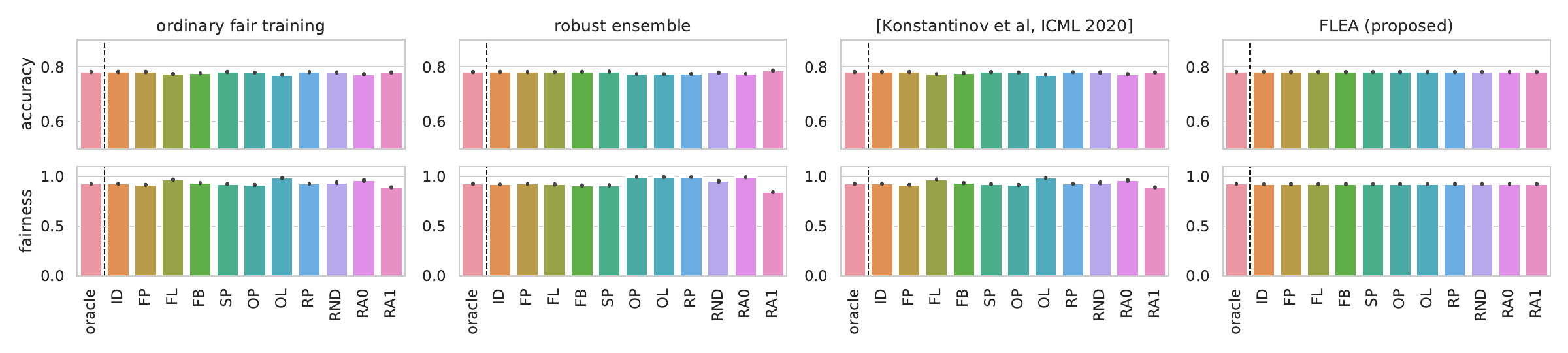}
\end{subfigure}
\begin{subfigure}[b]{\textwidth}\centering
\caption{$N=51, N-K=15$}\includegraphics[width=\textwidth]{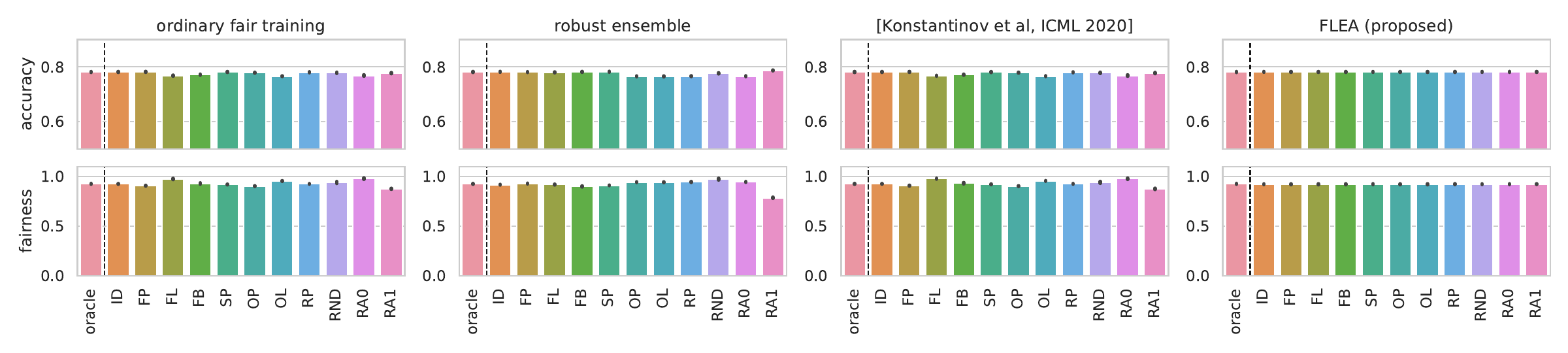}
\end{subfigure}
\begin{subfigure}[b]{\textwidth}\centering
\caption{$N=51, N-K=20$}\includegraphics[width=\textwidth]{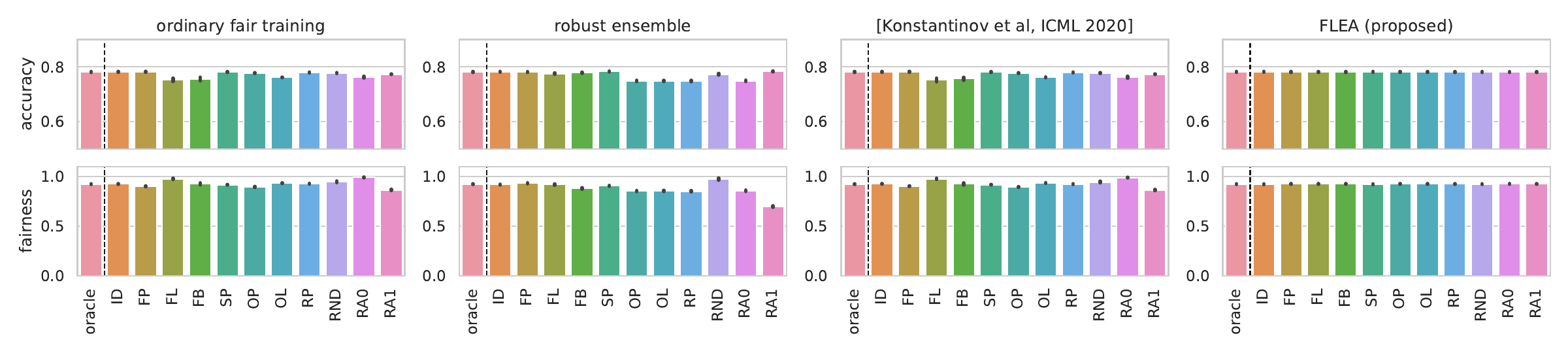}
\end{subfigure}
\begin{subfigure}[b]{\textwidth}\centering
\caption{$N=51, N-K=25$}\includegraphics[width=\textwidth]{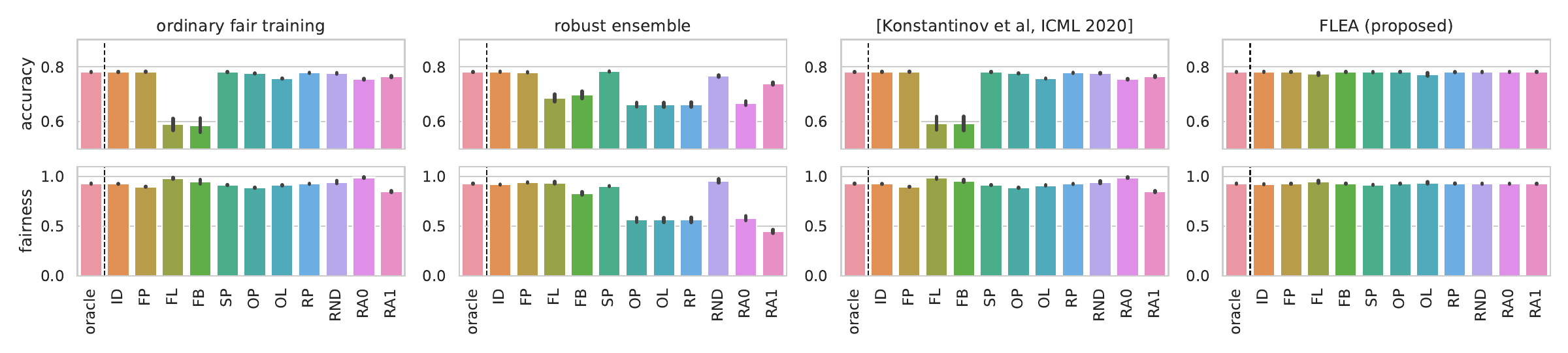}
\end{subfigure}
\end{figure*}

\begin{figure*}[t]
\caption{\folktables\ dataset, postprocessing-based fairness}
\begin{subfigure}[b]{\textwidth}\centering
\caption{$N=51, N-K=5$}\includegraphics[width=\textwidth]{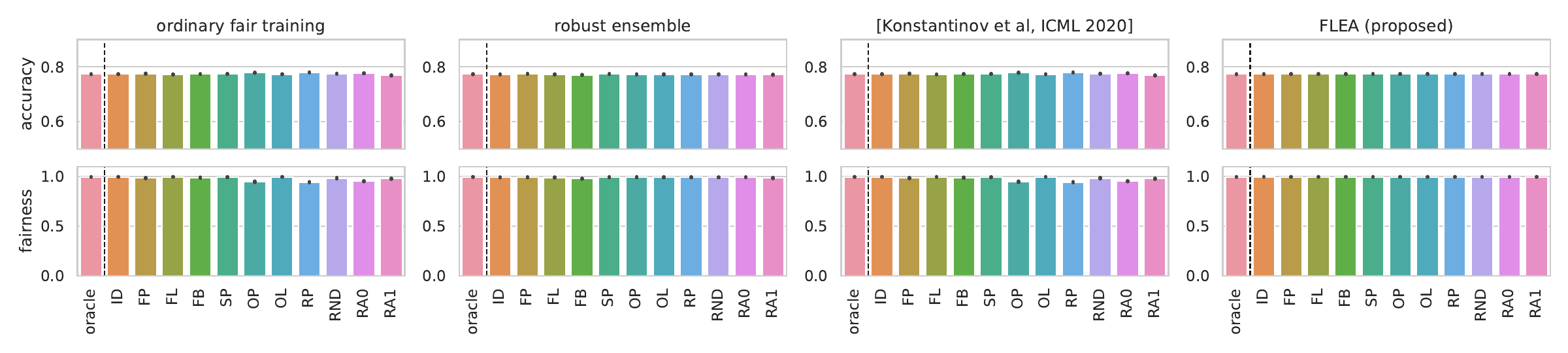}
\end{subfigure}
\begin{subfigure}[b]{\textwidth}\centering
\caption{$N=51, N-K=10$}\includegraphics[width=\textwidth]{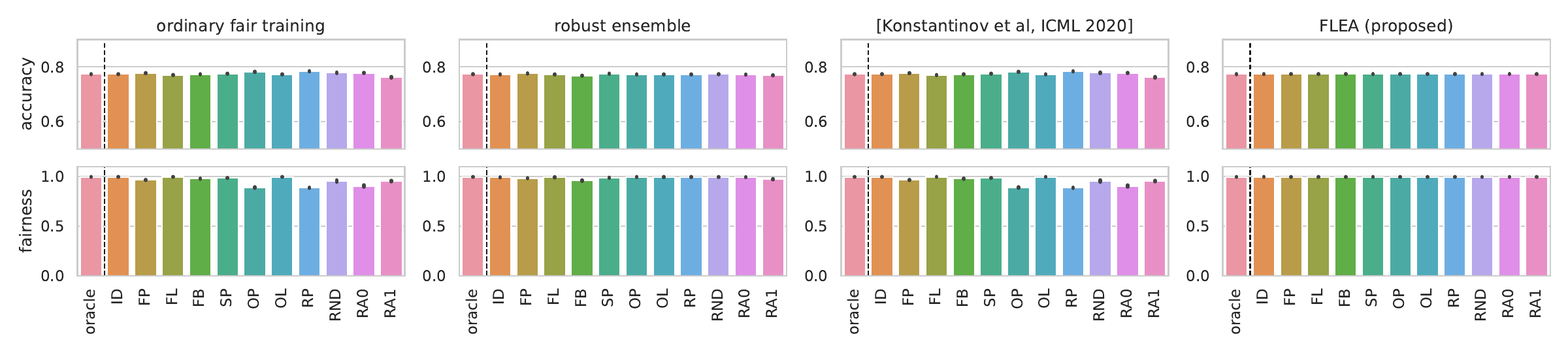}
\end{subfigure}
\begin{subfigure}[b]{\textwidth}\centering
\caption{$N=51, N-K=15$}\includegraphics[width=\textwidth]{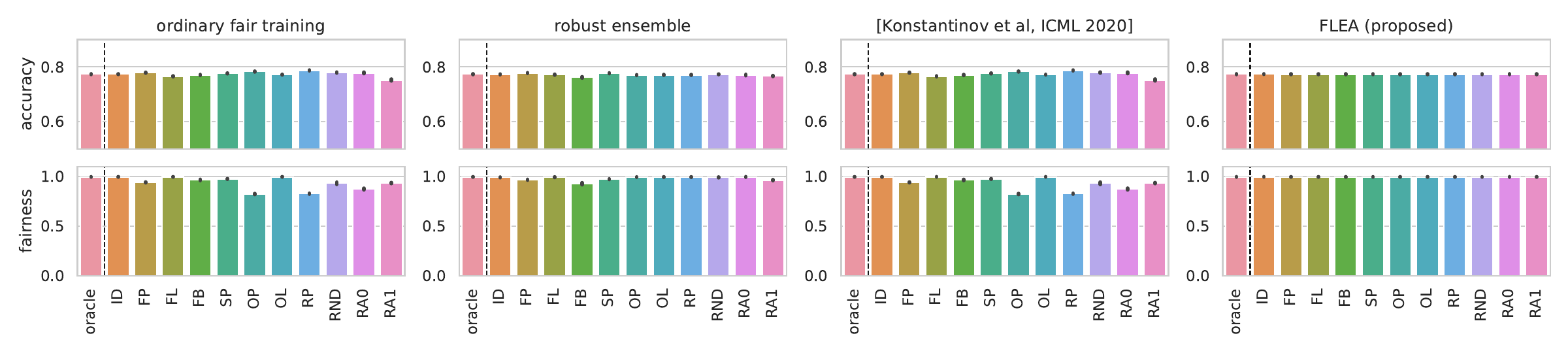}
\end{subfigure}
\begin{subfigure}[b]{\textwidth}\centering
\caption{$N=51, N-K=20$}\includegraphics[width=\textwidth]{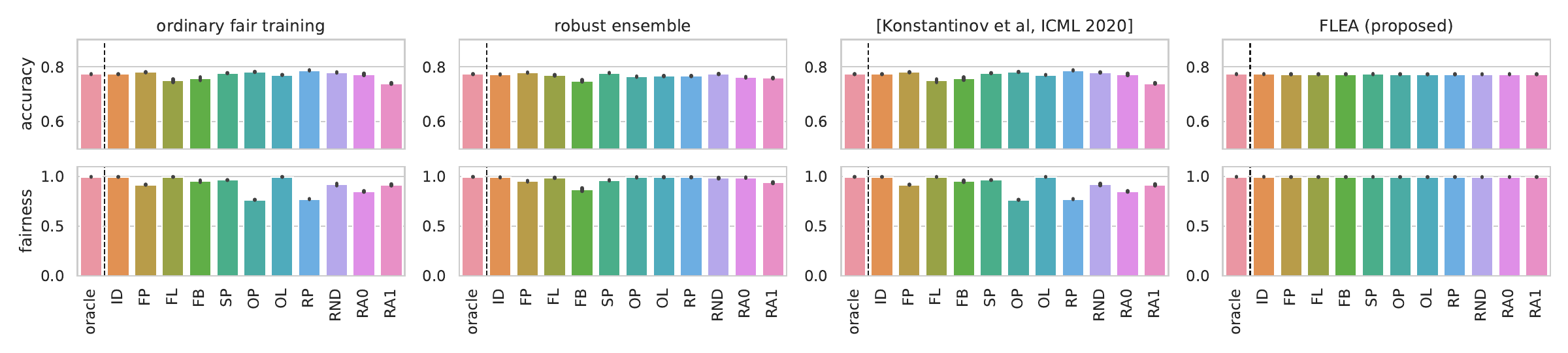}
\end{subfigure}
\begin{subfigure}[b]{\textwidth}\centering
\caption{$N=51, N-K=25$}\includegraphics[width=\textwidth]{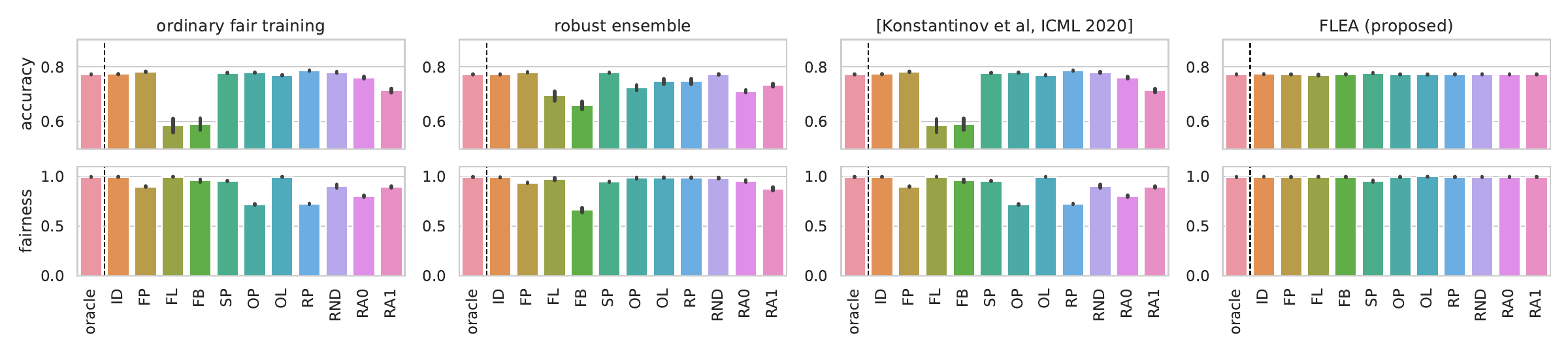}
\end{subfigure}
\end{figure*}

\begin{figure*}[t]
\caption{\folktables\ dataset, adversarial fairness}
\begin{subfigure}[b]{\textwidth}\centering
\caption{$N=51, N-K=5$}\includegraphics[width=\textwidth]{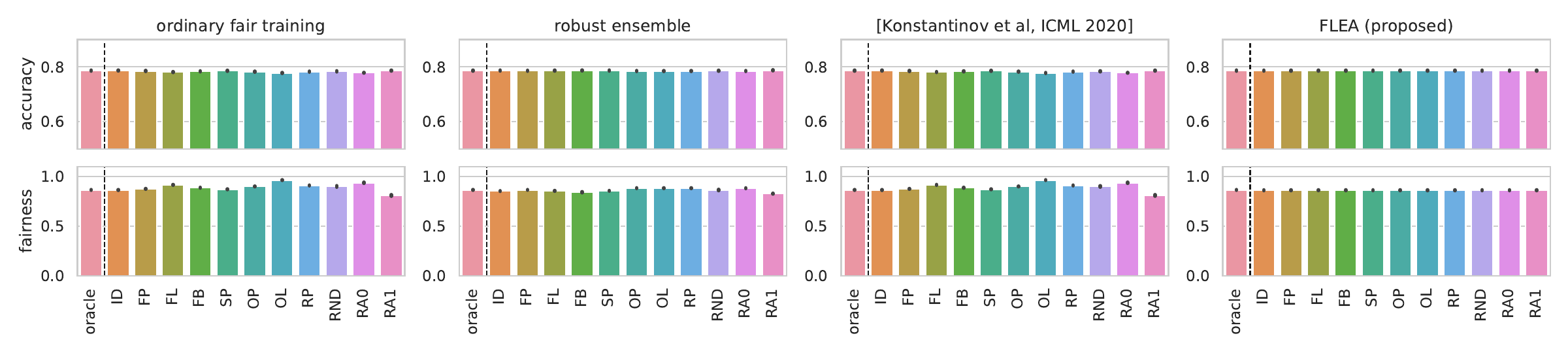}
\end{subfigure}
\begin{subfigure}[b]{\textwidth}\centering
\caption{$N=51, N-K=10$}\includegraphics[width=\textwidth]{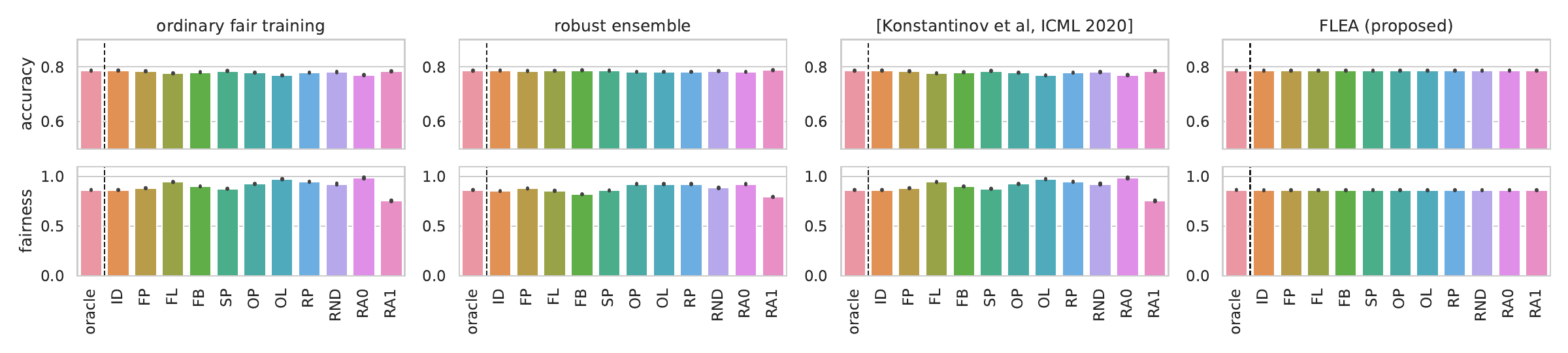}
\end{subfigure}
\begin{subfigure}[b]{\textwidth}\centering
\caption{$N=51, N-K=15$}\includegraphics[width=\textwidth]{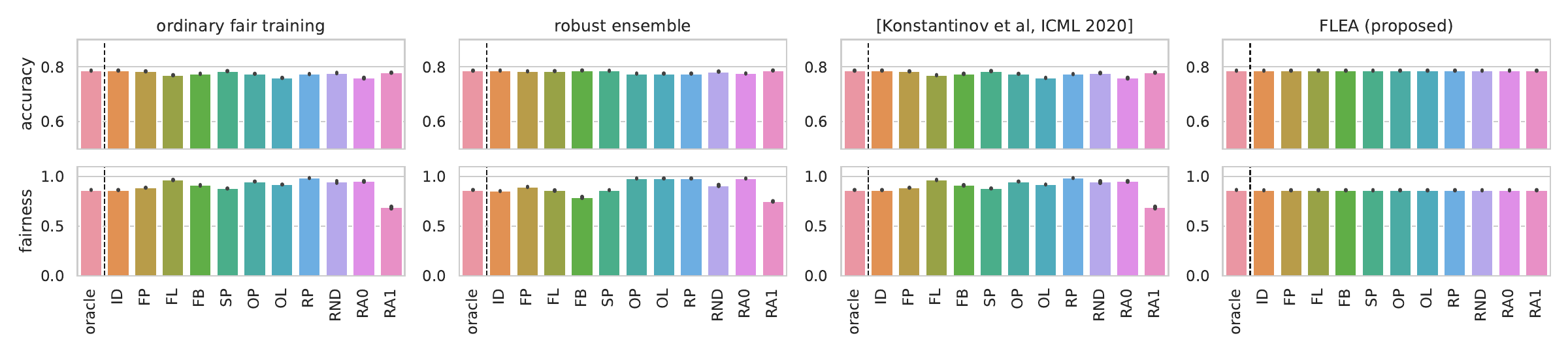}
\end{subfigure}
\begin{subfigure}[b]{\textwidth}\centering
\caption{$N=51, N-K=20$}\includegraphics[width=\textwidth]{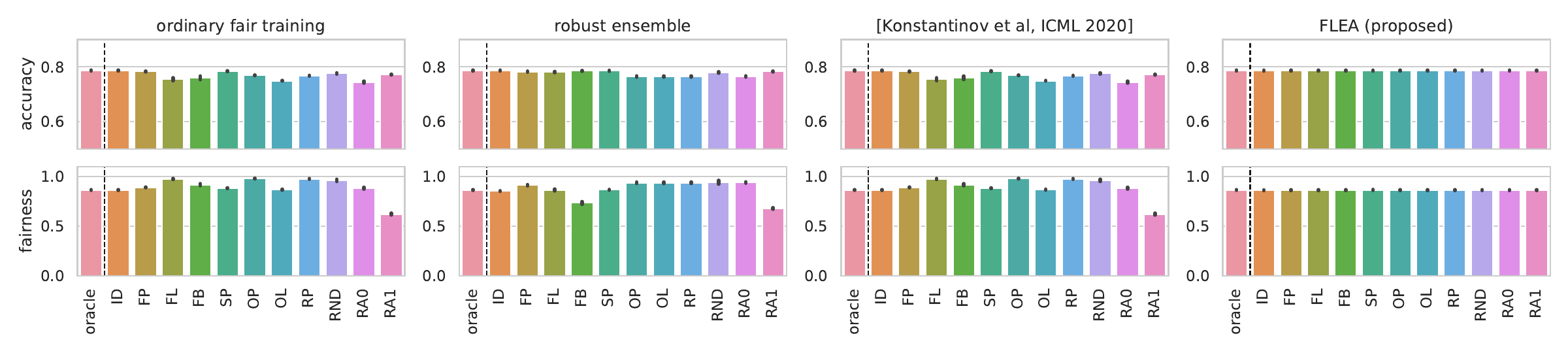}
\end{subfigure}
\begin{subfigure}[b]{\textwidth}\centering
\caption{$N=51, N-K=25$}\includegraphics[width=\textwidth]{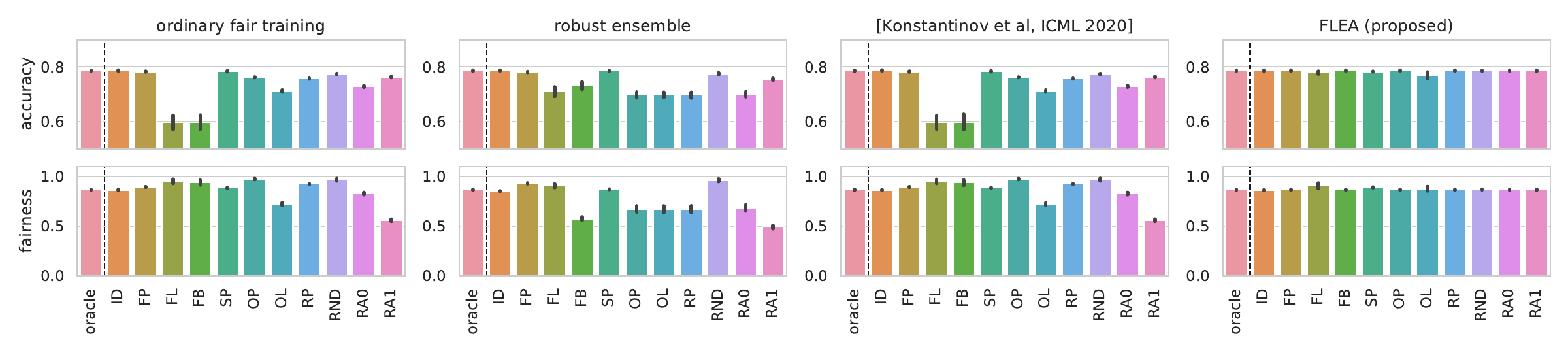}
\end{subfigure}
\end{figure*}

\begin{figure*}[t]
\caption{\folktables\ dataset, fairness-unaware}
\label{extrafig:summaryresults_folktables_last} 
\begin{subfigure}[b]{\textwidth}\centering
\caption{$N=51, N-K=5$}\includegraphics[width=\textwidth]{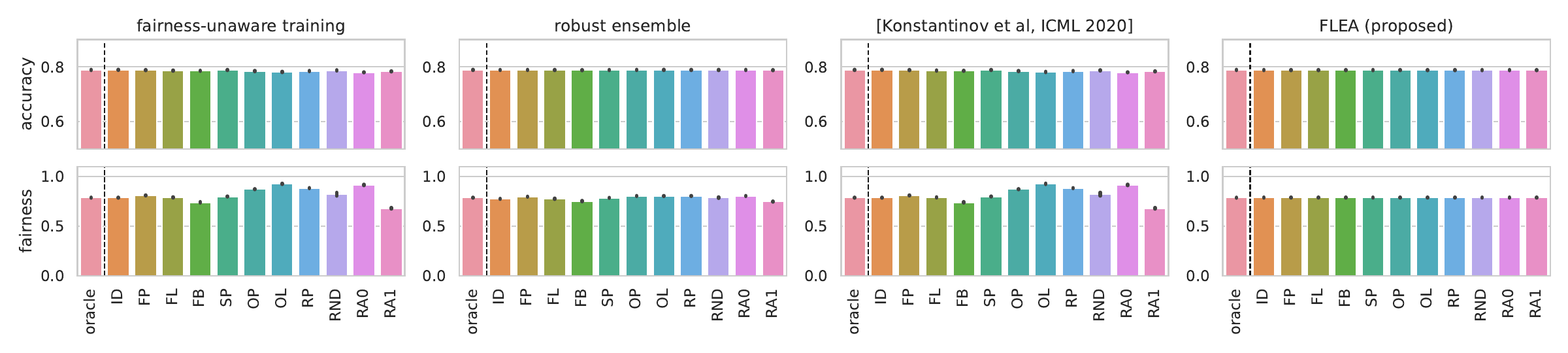}
\end{subfigure}
\begin{subfigure}[b]{\textwidth}\centering
\caption{$N=51, N-K=10$}\includegraphics[width=\textwidth]{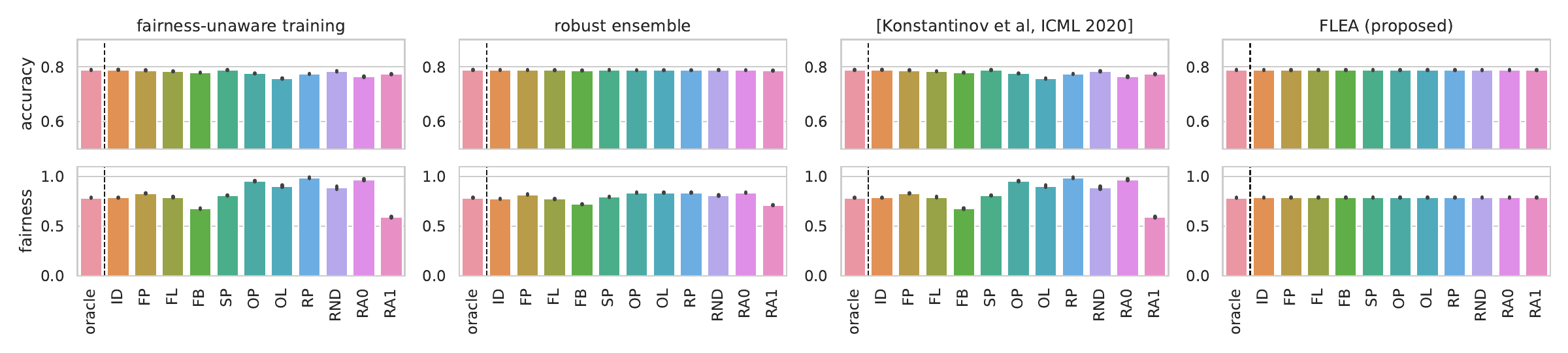}
\end{subfigure}
\begin{subfigure}[b]{\textwidth}\centering
\caption{$N=51, N-K=15$}\includegraphics[width=\textwidth]{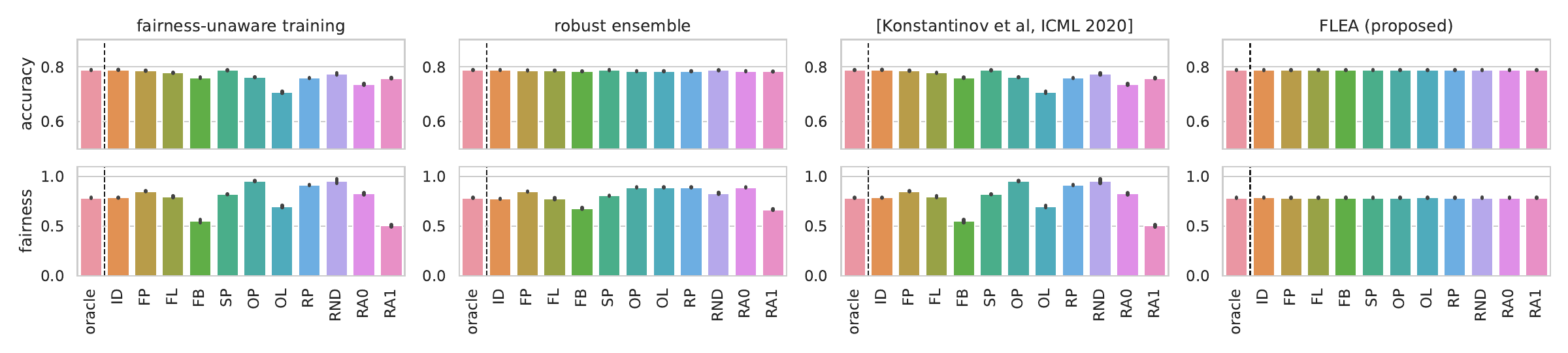}
\end{subfigure}
\begin{subfigure}[b]{\textwidth}\centering
\caption{$N=51, N-K=20$}\includegraphics[width=\textwidth]{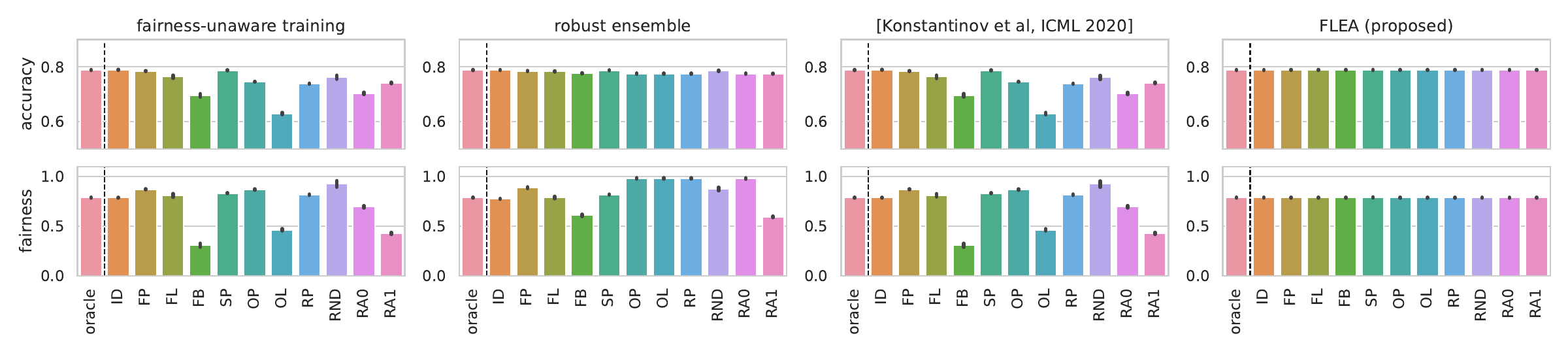}
\end{subfigure}
\begin{subfigure}[b]{\textwidth}\centering
\caption{$N=51, N-K=25$}\includegraphics[width=\textwidth]{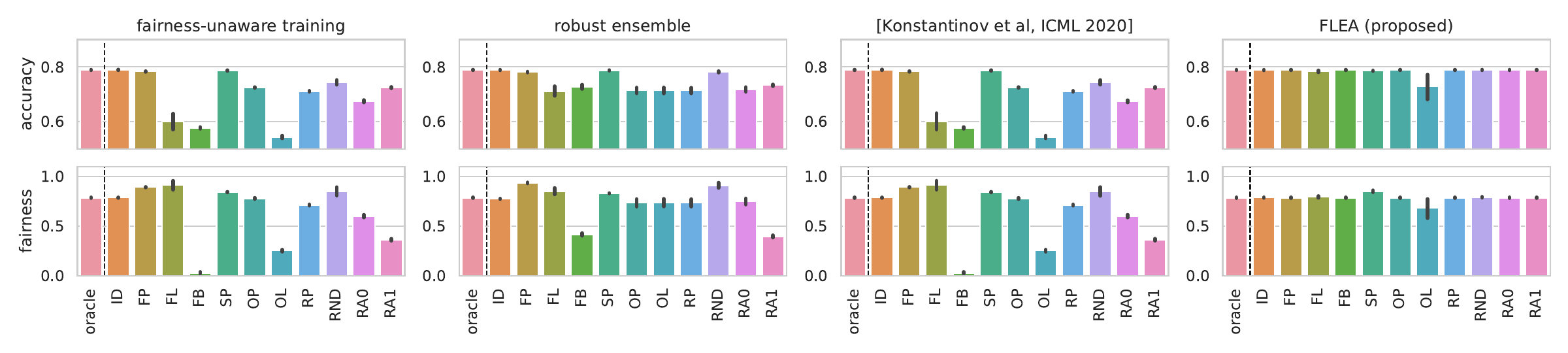}
\end{subfigure}
\end{figure*}

\clearpage
\section{Additional Results}
While our experimental evaluation in this work focussed on the 
setting of linear classification, \method is also applicable 
in combination with nonlinear classifiers. However, this 
leads to increased computational cost, and also the size of the 
training sources would have to be larger to effectively estimate 
the $\disc$ and $\disp$ measures. %
A compromise is to perform \method's filtering step 
with respect to linear classifiers but afterwards train a
non-linear classifier on the resulting combined training set. 
We observed this setup to work well in practice, even 
though the theoretical guarantees do not hold.

\begin{table*}[t]\centering\footnotesize
\caption{Results of \method and baselines for robust fairness-aware multisource learning using 
preprocessing fairness and a non-linear (gradient boosted decision trees) classifer. 
See Table~\ref{tab:results} for an explanation of setting and entries.}
\label{tab:results-nonlinear}
\begin{subfigure}{\textwidth}\centering
\caption{\adult, \compas, \drugs and \german\ datasets with 
$5$ sources of which $2$ are unreliable.}\label{tab:results-nonlinear-homogeneous}
\begin{tabular}{l|ll|ll|ll|ll}
& \multicolumn{2}{c}{\adult} & \multicolumn{2}{c}{\compas} & \multicolumn{2}{c}{\drugs} & \multicolumn{2}{c}{\german}
\\
\textbf{method} & accuracy & fairness & accuracy & fairness & accuracy & fairness & accuracy & fairness 
\\\hline
naive	         & $66.2_{\pm 0.7}$ & $86.9_{\pm 1.8}$
                 & $55.6_{\pm 2.6}$ & $83.3_{\pm 4.7}$
                 & $60.0_{\pm 2.5}$ & $77.8_{\pm 3.7}$ 
                 & $53.9_{\pm 2.6}$ & $76.2_{\pm 3.8}$
\\\hline
robust ensemble  & $74.4_{\pm 0.7}$& $71.8_{\pm 2.3}$
                 & $56.3_{\pm 2.5}$& $41.2_{\pm 5.7}$ 
                 & $54.9_{\pm 2.3}$& $36.5_{\pm 5.2}$
                 & $58.8_{\pm 1.8}$& $48.0_{\pm 4.1}$
\\
\nolink{\citep{konstantinov2020sample}}& $77.5_{\pm 0.4}$ & $90.9_{\pm 1.7}$
                                       & $56.2_{\pm 3.1}$ & $81.8_{\pm 6.3}$
                                       & $52.7_{\pm 1.6}$ & $79.4_{\pm 2.2}$ 
                                       & $53.3_{\pm 3.0}$ & $79.0_{\pm 4.1}$
\\\hline
FLEA (proposed)             & $77.9_{\pm 0.4}$ & $92.1_{\pm 1.1}$
                            & $62.3_{\pm 0.9}$ & $87.8_{\pm 5.4}$
                            & $62.8_{\pm 1.6}$ & $80.6_{\pm 3.3}$ 
                            & $64.5_{\pm 6.6}$ & $84.4_{\pm 3.7}$
\\\hline
oracle	        & $78.5_{\pm 0.4}$ & $93.3_{\pm 1.5}$
                & $64.8_{\pm 1.4}$ & $93.0_{\pm 5.9}$
                & $64.1_{\pm 2.1}$ & $88.1_{\pm 4.6}$ 
                & $69.0_{\pm 3.8}$ & $92.8_{\pm 4.7}$
                
\end{tabular}
\end{subfigure}
\end{table*}

\begin{figure*}[t]
\caption{Nonlinear (gradient-boosted decision trees) classifier, preprocessing-based fairness}
\label{extrafig:summaryresults_nonlinear_first}
\begin{subfigure}[b]{\textwidth}\centering
\caption{\adult ($N=5, N-K=2$)}\includegraphics[width=\textwidth]{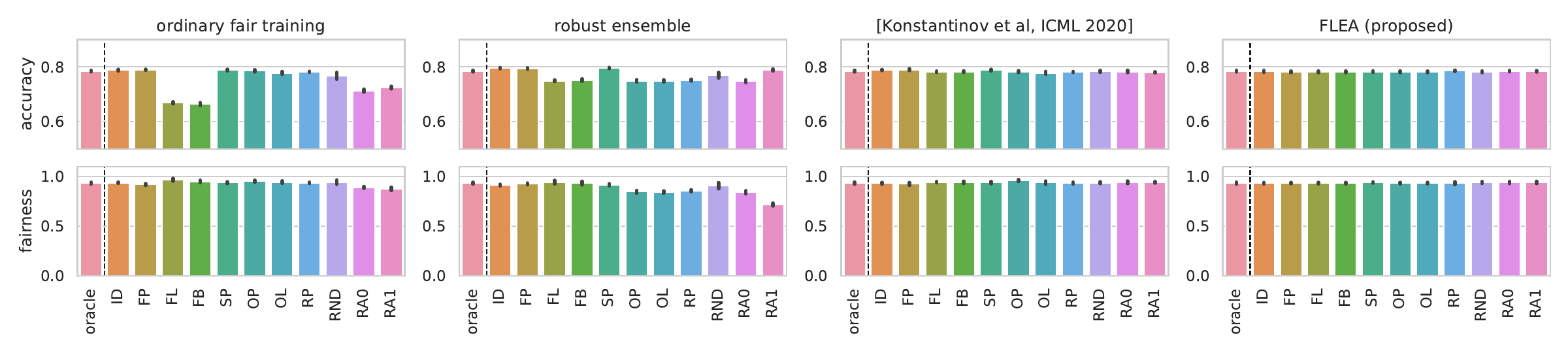}
\end{subfigure}
\begin{subfigure}[b]{\textwidth}\centering
\caption{\compas ($N=5, N-K=2$)}\includegraphics[width=\textwidth]{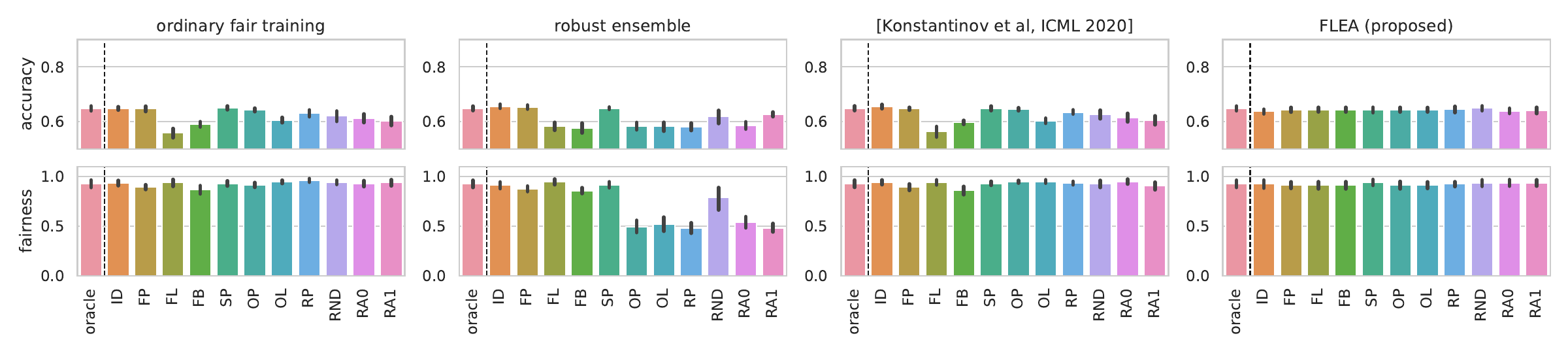}
\end{subfigure}
\begin{subfigure}[b]{\textwidth}\centering
\caption{\drugs ($N=5, N-K=2$)}\includegraphics[width=\textwidth]{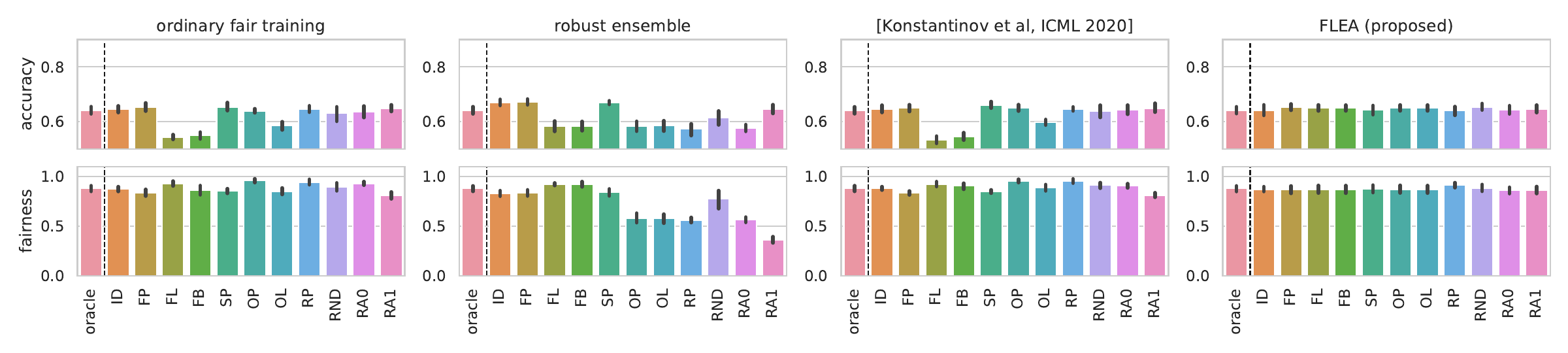}
\end{subfigure}
\begin{subfigure}[b]{\textwidth}\centering
\caption{\german ($N=5, N-K=2$)}\includegraphics[width=\textwidth]{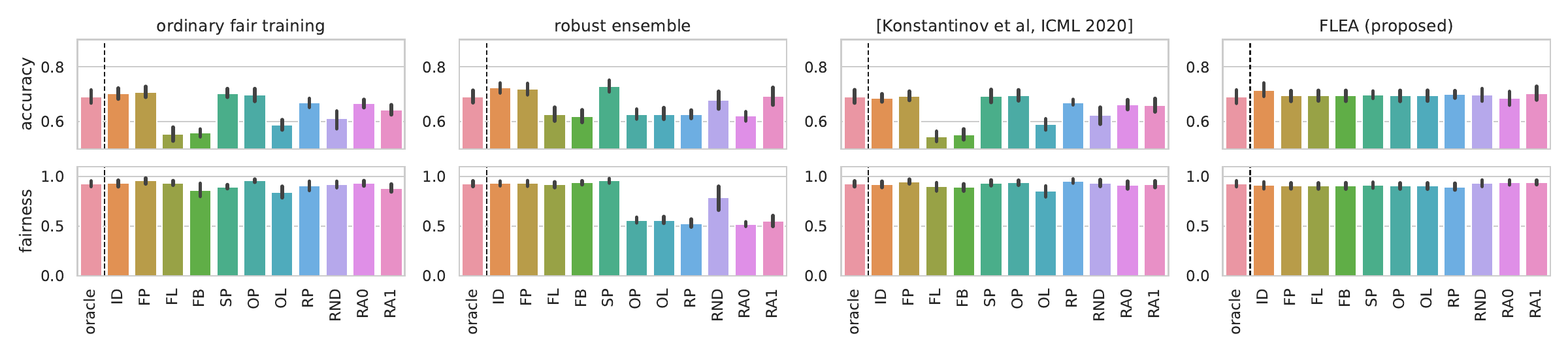}
\end{subfigure}
\end{figure*}

\begin{figure*}[t]
\caption{Nonlinear (gradient-boosted decision trees) classifier, fairness-unaware learning}
\label{extrafig:summaryresults_nonlinear_last}
\begin{subfigure}[b]{\textwidth}\centering
\caption{\adult ($N=5, N-K=2$)}\includegraphics[width=\textwidth]{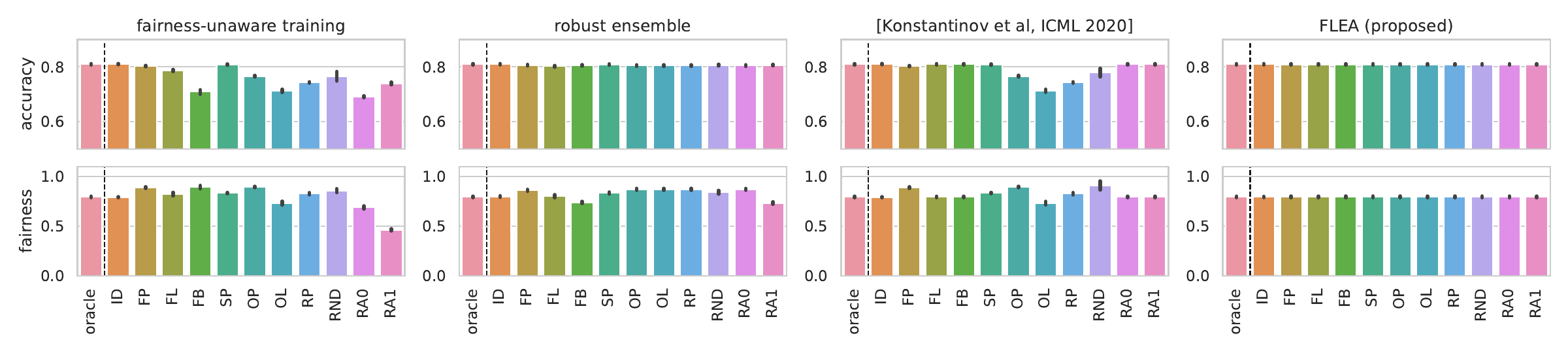}
\end{subfigure}
\begin{subfigure}[b]{\textwidth}\centering
\caption{\compas ($N=5, N-K=2$)}\includegraphics[width=\textwidth]{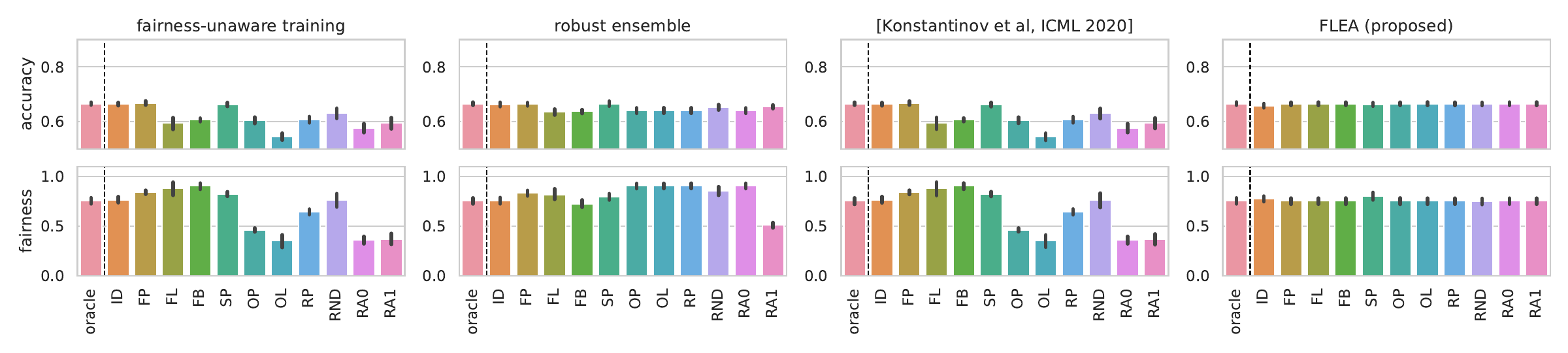}
\end{subfigure}
\begin{subfigure}[b]{\textwidth}\centering
\caption{\drugs ($N=5, N-K=2$)}\includegraphics[width=\textwidth]{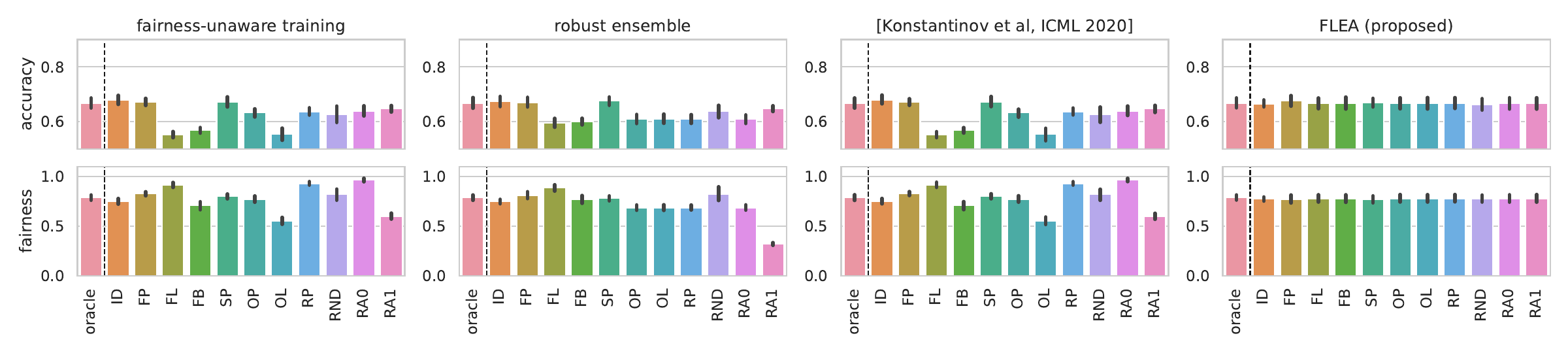}
\end{subfigure}
\begin{subfigure}[b]{\textwidth}\centering
\caption{\german ($N=5, N-K=2$)}\includegraphics[width=\textwidth]{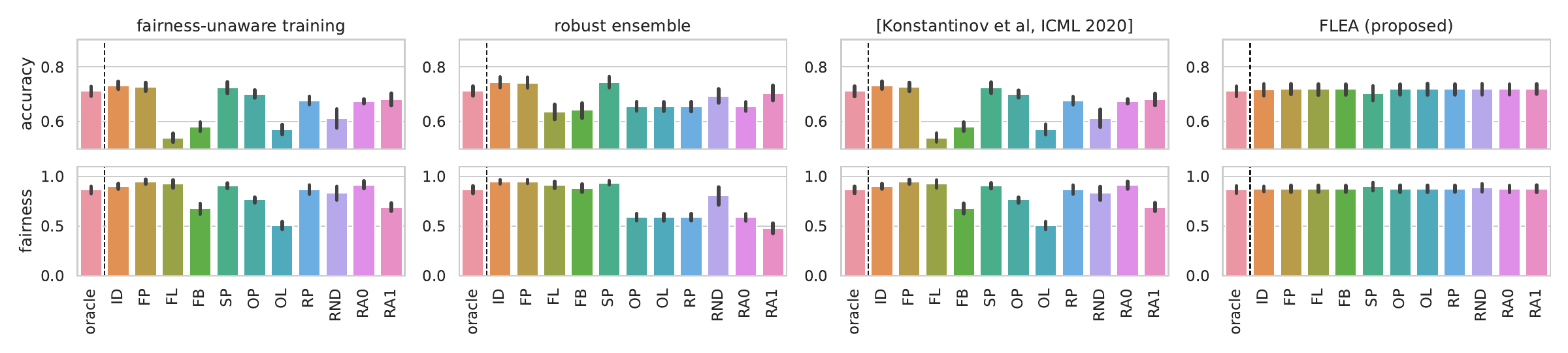}
\end{subfigure}
\end{figure*}

As exemplary setting, we perform experiments in a subset of the 
situations using \emph{gradient boosted decision trees}~\cite{friedman2001greedy} 
from the \texttt{xgboost} package\footnote{\url{https://xgboost.readthedocs.io/en/stable/python/index.html}} 
as nonlinear classifiers. 
We use 5-fold crossvalidation to select 
hyperparameters $\texttt{n\_estimators}\in\{100,200,300,400,500\}$
and $\texttt{max\_depth}\in\{3,5,7,9\}$.
To encourage fairness of the resulting classifiers we 
use the \emph{preprocessing} approach, as that requires 
no changes to the actual classifier training routines.
As baselines, we compare to the robust ensemble and 
the filtering approach of~\cite{konstantinov2020sample}.
The other baselines of Section~\ref{subsec:results} are not 
applicable, as they require modifications of the training 
process itself. 

Table~\ref{tab:results-nonlinear} reports the results in
tabular form. Figures~\ref{extrafig:summaryresults_nonlinear_first}
and Figure~\ref{extrafig:summaryresults_nonlinear_last}
visualize the results for each adversary. 

Overall, one can see the same trend as in the linear setting: 
naively merging the data sources leads to strong decreases in 
accuracy and fairness. 
The different robust methods overcome this to varying
degrees, with \method always achieving the best 
results, \ie closest to the hypothetical \emph{oracle}.

A comparison of Table~\ref{tab:results-nonlinear} to Table~\ref{tab:results} 
shows that the use of a nonlinear instead of linear classifier generally 
does not lead to more accurate nor more fair classifiers in 
the tested setting. Presumably, this is because in the chosen 
categorical representation, the bottleneck for prediction quality 
is not a lack of expressibility of the hypothesis class, but rather 
the datasets' intrinsic noise. 
This view is supported by the fact that accuracy and fairness are 
not much increased even for the \emph{oracle} approach, which is 
not affected by the data manipulations.

\section{Discussion of the role of $\disb$, $\disc$ and $\disp$ and ablation study}
\method relies on the combination of three dissimilarity measures: $\disc$, $\disp$ and $\disb$. In these section we discuss the importance of each of them and report on an ablation study to verify their practical significance.

Of the three measures, $\disc$ is indispensable to ensure classifier accuracy, as it is the only measure that depends on the label values.
At the same time, $\disc$ is blind to changes in the protected attributes whenever those are not part of the feature set. 

Even in the case when the protected attribute is among the features, the $\disc$ measure may not detect changes in the data that may harm fairness. For example,
if one of the protected groups is much more rare than the other,
changing even a small number of data points (\eg the points from that group) can cause a large change in the conditional distributions of the data given the value of the 
protected attribute. At the same time, the discrepancy will remain largely unaffected, since only a few points
have been changed in total. Therefore, filtering only based on $\disc$ is insufficient - sources with a different conditional distribution may get in through the filtering step,
potentially causing unfair classifiers (on clean data) to appear fair on the (corrupted) dataset.

FLEA avoids such issues by additionally adopting the $\disp$ measure, 
which can reliably detect changes in the conditional distributions of 
the data given the value of the protected attribute, thereby ensuring 
reliable fairness estimates based on the sources that are returned by 
the filtering procedure.
However, $\disb$ is not sensitive to manipulations in the size of the
protected attributes, for example, an adversary who selectively drops 
examples of one protected group. %
However, the $\disb$ measure would detect such a manipulation. Thereby, 
it ensures that the disparity of the union of multiple sources is close 
to their average individual disparity. 
This is important because a fairness-aware learner works on top of 
\textsc{FilterSources} by merging the data from the sources returned by 
the filtering algorithm. This aspect becomes apparent only in the proof 
of Theorem~\ref{thm:method}, see Section~\ref{sec:proof}.

\subsection{Ablation study}
To understand the respective contributions of $\disb$, $\disc$,
and $\disp$ on real-world data, we performed an ablation study that 
runs variants of \method in which any subsets of the three measures 
are used to compute the $D$-scores. The variant with all measures 
active is identical to \method. The variant with all measures 
inactive randomly chooses subsets to train on. 

The results are presented in Table~\ref{tab:ablation_first}.
One can see that for all datasets, not using $\disc$ (column 3) 
for the filtering step has the most noticeable effect.
This makes sense, because several of the adversaries make 
large changes to the labels and features, and $\disc$ is 
well suited to identify these.

Not using $\disp$ (column 2) has usually less of an effect, but 
in some situations it does lead to a noticeable drop in fairness, 
see \eg \compas, and \drugs, as well as \folktables with
$N-K\in\{5,10,15,20\}$. 
This is also consistent with our expectations, as the measure is 
specifically able to detect even subtle manipulation that would 
negatively affect fairness. 
However, for \folktables with $N-K=25$, where the amount of 
manipulated data is very close to half, not using $\disp$ 
would actually be beneficial for the system. 
We attribute this to the fact that as a difference of ratios,
$\disp$ is harder to estimate from small sample sets than
the other two measures. A noisy estimate, however, can lead
to clean sources to be suppressed, and manipulated ones
to be selected. This explanation is also supported by the 
fact that for small datasets the variability of results is 
bigger when $\disp$ is included than when it is not.

The effect of not using $\disb$ (column 4) is small on real 
data. It never exceeds the standard deviation of the estimates.
This, again, is expected, as the other two measures are
typically able to ensure accuracy and fairness, whereas
the role of $\disb$ is mainly to handle corner cases
that are unlikely to occur in real data. 

Dropping two measures from the filtering step only 
makes sense, if the remaining measure is $\disc$
(column 6). Even then, a decrease in accuracy and/or 
fairness is quite common, or at least an increase
in variability.

\begin{table}[t]\setlength{\tabcolsep}{5pt}
\caption{Performance of \method with different combination of $\disb$, $\disc$, and $\disp$ activated or deactived (crossed-out). Reported results are in the same format at Table~\ref{tab:results}: minimal accuracy (A) and fairness (F) against any of the tested adversaries}
\begin{subfigure}[b]{\textwidth}\centering\small
\caption{regularization-based fairness}\label{tab:ablation_first}
\begin{tabular}{rc@{~~}rrrrrrrr}
	&& \multicolumn{1}{c}{disb} & \multicolumn{1}{c}{disb} & \multicolumn{1}{c}{disb} & \multicolumn{1}{c}{\xcancel{disb}} & \multicolumn{1}{c}{disb} & \multicolumn{1}{c}{\xcancel{disb}} & \multicolumn{1}{c}{\xcancel{disb}} & \multicolumn{1}{c}{\xcancel{disb}}
\\
	&& \multicolumn{1}{c}{disc} & \multicolumn{1}{c}{disc} & \multicolumn{1}{c}{\xcancel{disc}} & \multicolumn{1}{c}{disc} & \multicolumn{1}{c}{\xcancel{disc}} & \multicolumn{1}{c}{disc} & \multicolumn{1}{c}{\xcancel{disc}} & \multicolumn{1}{c}{\xcancel{disc}}
\\
	&& \multicolumn{1}{c}{disp} & \multicolumn{1}{c}{\xcancel{disp}} & \multicolumn{1}{c}{disp} & \multicolumn{1}{c}{disp} & \multicolumn{1}{c}{\xcancel{disp}} & \multicolumn{1}{c}{\xcancel{disp}} & \multicolumn{1}{c}{disp} & \multicolumn{1}{c}{\xcancel{disp}}
\\\hline
adult & A:& $70.3_{\pm 0.4}$& $70.3_{\pm 0.4}$& $63.7_{\pm 9.9}$& $70.3_{\pm 0.4}$& $53.3_{\pm 16.6}$& $70.3_{\pm 0.4}$& $64.3_{\pm 10.1}$& $57.4_{\pm 16.5}$\\
$N=5$, $N-K=2$ & F:& $98.2_{\pm 1.1}$& $98.2_{\pm 1.1}$& $91.2_{\pm 18.4}$& $98.2_{\pm 1.1}$& $64.0_{\pm 22.5}$& $89.2_{\pm 5.5}$& $91.8_{\pm 18.6}$& $68.8_{\pm 22.4}$\\\hline
compas & A:& $66.2_{\pm 1.1}$& $66.1_{\pm 1.2}$& $61.9_{\pm 8.9}$& $66.2_{\pm 1.1}$& $43.8_{\pm 13.0}$& $66.0_{\pm 1.1}$& $61.7_{\pm 9.4}$& $56.0_{\pm 13.0}$\\
$N=5$, $N-K=2$ & F:& $95.4_{\pm 2.6}$& $93.1_{\pm 1.8}$& $94.3_{\pm 2.6}$& $95.3_{\pm 2.5}$& $86.2_{\pm 7.5}$& $90.2_{\pm 1.9}$& $94.3_{\pm 2.4}$& $68.7_{\pm 26.2}$\\\hline
drugs & A:& $64.4_{\pm 1.5}$& $64.3_{\pm 1.5}$& $52.1_{\pm 11.8}$& $64.3_{\pm 1.4}$& $55.4_{\pm 9.2}$& $64.3_{\pm 1.5}$& $58.1_{\pm 9.8}$& $52.9_{\pm 11.1}$\\
$N=5$, $N-K=2$ & F:& $93.0_{\pm 4.5}$& $88.9_{\pm 5.2}$& $72.6_{\pm 24.7}$& $92.5_{\pm 4.9}$& $69.6_{\pm 12.3}$& $86.4_{\pm 6.2}$& $82.9_{\pm 18.1}$& $70.5_{\pm 21.0}$\\\hline
germancredit & A:& $66.5_{\pm 2.8}$& $66.0_{\pm 2.7}$& $60.2_{\pm 5.1}$& $66.8_{\pm 2.8}$& $59.7_{\pm 6.9}$& $66.2_{\pm 2.5}$& $55.9_{\pm 9.8}$& $56.6_{\pm 11.5}$\\
$N=5$, $N-K=2$ & F:& $93.8_{\pm 3.8}$& $92.6_{\pm 3.8}$& $86.2_{\pm 13.1}$& $93.9_{\pm 3.9}$& $87.1_{\pm 5.1}$& $92.5_{\pm 4.1}$& $81.0_{\pm 12.6}$& $75.4_{\pm 18.9}$\\\hline
folktables & A:& $75.5_{\pm 0.2}$& $75.5_{\pm 0.2}$& $74.6_{\pm 0.2}$& $75.4_{\pm 0.2}$& $74.4_{\pm 0.5}$& $75.4_{\pm 0.2}$& $74.6_{\pm 0.2}$& $74.2_{\pm 0.5}$\\
$N=51$, $N-K=5$ & F:& $99.5_{\pm 0.3}$& $99.0_{\pm 0.5}$& $99.3_{\pm 0.3}$& $99.5_{\pm 0.2}$& $93.0_{\pm 2.6}$& $99.1_{\pm 0.6}$& $99.4_{\pm 0.2}$& $92.7_{\pm 4.0}$\\\hline
folktables & A:& $75.4_{\pm 0.2}$& $75.4_{\pm 0.2}$& $73.9_{\pm 0.4}$& $75.4_{\pm 0.2}$& $73.7_{\pm 0.3}$& $75.3_{\pm 0.2}$& $73.9_{\pm 0.5}$& $73.0_{\pm 0.9}$\\
$N=51$, $N-K=10$ & F:& $99.5_{\pm 0.3}$& $98.1_{\pm 0.7}$& $99.1_{\pm 0.2}$& $99.5_{\pm 0.3}$& $88.3_{\pm 3.0}$& $98.2_{\pm 0.5}$& $99.1_{\pm 0.4}$& $86.7_{\pm 3.7}$\\\hline
folktables & A:& $75.4_{\pm 0.2}$& $75.4_{\pm 0.3}$& $72.7_{\pm 0.6}$& $75.4_{\pm 0.2}$& $72.3_{\pm 0.9}$& $75.3_{\pm 0.3}$& $72.7_{\pm 0.8}$& $68.8_{\pm 8.1}$\\
$N=51$, $N-K=15$ & F:& $99.7_{\pm 0.2}$& $97.0_{\pm 0.6}$& $98.9_{\pm 0.4}$& $99.6_{\pm 0.3}$& $82.4_{\pm 2.4}$& $97.1_{\pm 1.0}$& $98.9_{\pm 0.5}$& $78.9_{\pm 5.4}$\\\hline
folktables & A:& $75.3_{\pm 0.2}$& $75.3_{\pm 0.2}$& $69.9_{\pm 3.3}$& $75.3_{\pm 0.2}$& $70.5_{\pm 2.7}$& $75.3_{\pm 0.2}$& $68.4_{\pm 4.5}$& $64.8_{\pm 10.9}$\\
$N=51$, $N-K=20$ & F:& $99.5_{\pm 0.2}$& $95.6_{\pm 1.1}$& $98.4_{\pm 0.6}$& $99.5_{\pm 0.2}$& $79.1_{\pm 3.1}$& $93.9_{\pm 1.5}$& $98.3_{\pm 0.7}$& $74.9_{\pm 3.3}$\\\hline
folktables & A:& $74.0_{\pm 1.4}$& $75.0_{\pm 0.3}$& $53.9_{\pm 14.4}$& $74.1_{\pm 1.3}$& $41.6_{\pm 14.1}$& $75.1_{\pm 0.3}$& $59.3_{\pm 11.5}$& $56.5_{\pm 13.3}$\\
$N=51$, $N-K=25$ & F:& $94.2_{\pm 1.5}$& $89.9_{\pm 1.6}$& $95.2_{\pm 4.7}$& $94.5_{\pm 1.8}$& $79.4_{\pm 6.7}$& $89.3_{\pm 1.7}$& $96.9_{\pm 1.7}$& $72.6_{\pm 2.7}$\\\hline

\end{tabular}
\end{subfigure}
\end{table}

Another interesting ablation study would be to 
determine the success rate of the \method's filtering step, 
\ie how what fraction of the malignant sources it successfully 
suppresses. 
This, however, we cannot estimate, because we lack ground 
truth information which sources are malignant and which 
are not. 
From the experimental setup, only the information is available 
which sources have been manipulated and in what way. However, 
whether a manipulation is benign or malignant depends not only 
on the adversary's strategy, but also on the actual data and the 
later learning strategy. 
The proxy measure of determining what fraction of all manipulated 
sources were detected would not be very meaningful, as adversaries 
can easily create manipulated data sources that are indistinguishable 
from clean ones, \eg just shuffling the data point or not making
any changes at all, as realized in our experiments by the \emph{ID} 
adversary. 

\section{Complete formulation and proof of Theorem \ref{thm:method}}\label{sec:proof}

In this section we present the full proof of Theorem \ref{thm:method}. 
We begin by reminding the reader of our notation and formal assumption 
from in \ref{sec:proof_assumptions}. Next in Section 
\ref{sec:appendix_concentration_tools} we state a few standard 
concentration results that are used in our main proof. In Section 
\ref{sec:distances_distributions} we define the population counterparts 
of the empirical discrepancy, disparity and disbalance measures, as 
understanding how well these measures are estimated from finite data is 
crucial for the proof of our results.

Finally, we present the full proof of Theorem \ref{thm:method} in 
Section \ref{sec:proof_now_really} and the proof of the concentration lemmas from 
Section \ref{sec:appendix_concentration_tools} in Section 
\ref{sec:proof_lemma}.

\subsection{Assumptions and formal adversary model}
\label{sec:proof_assumptions}
For convenience of the reader, we repeat the formal notation and assumptions stated in \ref{sec:theory_assumptions}. Initially, there are $N$ datasets
$\tilde S_1,\dots,\tilde S_N$, with the $i$-th set of samples being drawn \iid from a distribution $p_i(x,y,a)$. We assume that all these distributions are clean, in the sense that they are close to the true target distribution $p$. Formally, we assume that each of the following conditions hold:
\begin{align}
\label{eqn:closeness_assumption}
\TV(p_i(x,y,a), p(x,y,a)) \leq \eta,\quad\text{and}\quad \max_{z\in\mathcal{A}}\big\{\TV(p_i(x,y|a=z), p(x,y|a=z))\big\} \leq \eta,
\end{align}
where $\TV(p, q) = \sup_{B\in \mathcal{B}(\mathcal{X}\times\mathcal{Y}\times\mathcal{A})} \left|p(B) - q(B)\right|$
with $\mathcal{B}(X)$ denoting the Borel $\sigma$-algebra on a topological 
space $X$.

Once the clean datasets $\tilde S_1,\dots,\tilde S_N$ are sampled, an \emph{adversary}
operates on them.
This results in new datasets, $S_1,\dots,S_N$, 
which the learning algorithm receives as input. 
The adversary is an arbitrary (deterministic or 
randomized) function  $\mathcal{F}:\prod_{i=1}^N\left(\mathcal{X}\times\mathcal{Y}\times\mathcal{A}\right)^n \rightarrow \prod_{i=1}^N\left(\mathcal{X}\times\mathcal{Y}\times\mathcal{A}\right)^{n}$,
with the only restriction that for a fixed subset 
of indices, $G\subset \{1,\dots,N\}$, the data remains unchanged. That is, $S_i = \tilde S_i$ 
for all $i\in G$, and $S_i$ is arbitrary for $i\not\in G$.

Note that the learner only observes the datasets $S_i$ and outputs a hypothesis based on them. Therefore, in the proof we will only work with the datasets $S_i$ and not with $\tilde S_i$, using that $S_i = \tilde S_i$ whenever $i \in G$, so that $S_i$ is \iid from $p_i$. For simplicity, we refer to a dataset $S_i$ or a source $i\in [N]$ as clean if $i\in G$.

We assume without loss of generality that $\tau = \mathbb{P}_{(X,Y,A)\sim p} (A = 0) \in \left(0, \frac{1}{2}\right]$. For technical reasons, we also assume that $18\eta < \tau = \mathbb{P}_{(X,Y,A)\sim p} (A = 0)$.

\subsection{Concentration tools and notation}
\label{sec:appendix_concentration_tools}
We first present the two lemmas which demonstrate uniform convergence of the empirical risk and the empirical fairness deviation measure respectively, for any hypothesis set $\HYPS$ with finite VC dimension. The first is just the classic VC generalization bound, as given in Chapter 28.1 of \citet{shwartz2014understanding}. The proof of the second lemma closely follows the proofs of similar results from \citet{woodworth17learning, agarwal2018reductions, konstantinov2022theotherone} and is presented in Section \ref{sec:proof_lemma} for completeness.

\begin{lemma}[Uniform Convergence for Binary Loss]
 \label{lemma:loss_convergence}
    Let $d$ be the VC-dimension of $\mathcal{H}$. Then for any dataset $S$ of size $n$ sampled \iid from a distribution $p$, for all $\delta \in (0,1)$,
    \[
        \mathbb{P}\Bigg(\sup_{h \in \H}|\mathcal{R}_S(h) - \mathcal{R}_p(h)| > 2 \sqrt{\frac{8d\log\left(\frac{en}{d}\right) + 2\log\left(\frac{4}{\delta}\right)}{n}} \Bigg) \leq \delta.
    \]
 \end{lemma}

\begin{lemma}[Uniform Convergence for demographic parity]
 \label{lemma:uniform_convergence}
    Let $p$ be a distribution on $\mathcal{X}\times\mathcal{A}\times\mathcal{Y}$. Let $d = \VC(\HYPS) \geq 1$ and let $\tau = \min_{a\in\{0,1\}}\mathbb{P}_{(X,Y,A)\sim p}(A = a)$ for some constant $\tau\in (0,0.5]$. Then for any dataset $S$ of size $n \geq \max \left\{\frac{8\log\left(\frac{8}{\delta}\right)}{\tau}, \frac{d}{2}\right\}$ sampled \iid from $p$, for all $\delta \in (0,1/2)$:
    \begin{equation}
        \mathbb{P}_S \Bigg( \sup_{h \in \H}\left|\unfairS(h) - \unfair{p}(h)\right|
        \geq  16\sqrt{2\frac{d \log\big(\frac{2en}{d}\big) + \log\big(\frac{24}{\delta}\big)}{n\tau}} \Bigg) \leq \delta
    \end{equation}
 \end{lemma} 
 
For the dataset $S_i$, denote by 
\begin{equation}
c_i \coloneqq \sum_{(x,y,a) \in S_i} \mathbbm{1}\{a = 0\} = |S_1^{a=0}|.
\end{equation}
Denote $\tau_i = \mathbb{P}_{(X,Y,A)\sim p_i}(A = 0)$. Then for a clean data source we have that $c_i \sim \Bin(n, \tau_i)$. Therefore, by the Hoeffding bound, for any $\delta > 0$:
\begin{equation}
\label{eqn:binomial_concentration}
\mathbb{P}\left(\left|c_i - n\tau_i \right| \geq n\sqrt{\frac{\log\left(\frac{2}{\delta}\right)}{2n}}\right) \leq 2\exp \left(-\frac{2\left(\sqrt{\frac{n}{2}\log\left(\frac{2}{\delta}\right)}\right)^2}{n}\right) = \delta.
\end{equation}
Because, by assumption, $\tau = \mathbb{P}_{(X,Y,A)\sim p_i}(A = 0) = \min_{a\in\{0,1\}}\mathbb{P}_{(X,Y,A)\sim p}(A = a)$ and 
$TV(p_i, p) \leq \eta$, for any clean dataset $S_i$, it holds that
$$\tau_i = \mathbb{P}_{(X,Y,A)\sim p_i}(A = 0) \geq \mathbb{P}_{(X,Y,A)\sim p}(A = 0) - \eta = \tau - \eta.$$ In addition, $$1 - \tau_i =  \mathbb{P}_{(X,Y,A)\sim p_i}(A = 1) \geq \mathbb{P}_{(X,Y,A)\sim p}(A = 1) - \eta \geq \mathbb{P}_{(X,Y,A)\sim p}(A = 0) - \eta = \tau - \eta.$$
Recall also that $\tau - \eta \geq \tau - 18\eta > 0$ by assumption. Denote by:
\begin{align}
\label{eqn:delta_definition}
\Delta(\delta) & = \max\left\{2 \sqrt{\frac{8d\log\left(\frac{en}{d}\right) + 2\log\left(\frac{4}{\delta}\right)}{n}}, 16\sqrt{2\frac{d \log\big(\frac{2en}{d}\big) + \log\big(\frac{24}{\delta}\big)}{n(\tau - \eta)}}, \sqrt{\frac{\log\left(\frac{2}{\delta}\right)}{2n}} \right\} \\ & = 16\sqrt{2\frac{d \log\big(\frac{2en}{d}\big) + \log\big(\frac{24}{\delta}\big)}{n(\tau - \eta)}}.
\end{align}
The lemmas above, as well as the observation that $\min \{\tau_i, 1 - \tau_i\} \geq \tau - \eta$ for any clean source $i$, readily imply that:
\begin{equation}
\label{eqn:bad_bound_risk}
\mathbb{P}\Bigg(\sup_{h \in \H}|\mathcal{R}_{S_i}(h) - \mathcal{R}_{p_i}(h)| \geq \Delta(\delta)\Bigg) \leq \delta,
\end{equation}
\begin{equation}
\label{eqn:bad_bound_fairness}
\mathbb{P}_S \Bigg( \sup_{h \in \H}\left|\unfair{S_i}(h) - \unfair{p_i}(h)\right| \geq  \Delta(\delta) \Bigg) \leq \delta
\end{equation}
and 
\begin{equation}
\label{eqn:bad_bound_binomial}
\mathbb{P}_S\left(\left|c_i - n\tau_i\right| \geq n\Delta(\delta) \right) \leq \delta,
\end{equation}  
for any clean $i$.

\subsection{Discrepancy, disparity and disbalance between distributions}
\label{sec:distances_distributions}
In our proof we will consider the population counterparts of the between-dataset distances that we defined in the main body of the text. In particular, the discrepancy distance between two distributions $p$ and $q$ is  
\begin{align}
\label{eqn:discrepancy_distributions}
\disc(p, q) = \sup_{h\in\mathcal{H}}\left|\mathcal{R}_p(h) - \mathcal{R}_q(h)\right|.
\end{align}
Similarly, the disparity is
\begin{align}
\label{eqn:disparity_distributions}
\disp(p, q) = \sup_{h\in\mathcal{H}}\left|\unfair{p}(h) - \unfair{q}(h)\right|,
\end{align}
where as an unfairness measure $\unfair{p}$ we will use:
$$\unfair{p}(h) = \mathbb{P}_{(X,Y,A)\sim p}\left(h(X) = 1 | A = 0\right) - \mathbb{P}_{(X,Y,A)\sim p}\left(h(X) = 1 | A = 1\right).$$
Finally, the disbalance is simply
\begin{align}
\label{eqn:disbalance_distributions}
\disb(p, q) = \left|\mathbb{P}_p(A = 0) - \mathbb{P}_q(A = 0)\right|.
\end{align}
Next we study these distances, between a distribution $p_i$ of a clean source and the true target distribution $p$. Recall our assumptions about the bounded TV distances from Section \ref{sec:proof_assumptions}. Clearly, we have that $\disb(p_i, p) \leq TV(p_i,p) \leq \eta$. Note also that:
\begin{align*}
\disc(p_i, p) = \sup_{h\in\mathcal{H}}\left|\mathcal{R}_{p_i}(h) - \mathcal{R}_p(h)\right| = \sup_{h\in\mathcal{H}}\left|\mathbb{P}_{(X,Y,A) \sim p_i}(h(X) \neq Y) - \mathbb{P}_{(X,Y,A) \sim p}(h(X) \neq Y)\right| \leq \eta,
\end{align*}
because any (measurable) classifier $h: \mathcal{X}\to \mathcal{Y}$ can be associated with a (Borel) set $S_h = \{(x,y,a) \in (\mathcal{X}\times\mathcal{Y}\times\mathcal{A}): h(x) \neq y\}$.
Finally, we bound the disparity in terms of $\eta$. Note that:

\newcommand{\cond}[1]{\mathbbm{1}\{#1\}}
\begin{align*}
\disp(p_i, p) & =  \sup_{h\in\mathcal{H}}\left|\mathbb{P}_{(X,Y,A)\sim p_i}\left(h(X) = 1 | A = 0\right) - \mathbb{P}_{(X,Y,A)\sim p_i}\left(h(X) = 1 | A = 1\right) \right. 
\\ 
& \qquad\left. - \mathbb{P}_{(X,Y,A)\sim p}\left(h(X) = 1 | A = 0\right) + \mathbb{P}_{(X,Y,A)\sim p}\left(h(X) = 1 | A = 1\right) \right| 
\\
& \leq \sup_{h\in\mathcal{H}}\left(\left|\mathbb{P}_{(X,Y,A)\sim p_i}\left(h(X) = 1 | A = 0\right) - \mathbb{P}_{(X,Y,A)\sim p}\left(h(X) = 1 | A = 0\right)\right| \right. 
\\ 
& \qquad\left. + \left|\mathbb{P}_{(X,Y,A)\sim p_i}\left(h(X) = 1 | A = 1\right) - \mathbb{P}_{(X,Y,A)\sim p}\left(h(X) = 1 | A = 1\right)\right|\right) 
\\
& \leq \sup_{h\in\mathcal{H}}\left|\mathbb{P}_{(X,Y,A)\sim p_i}\left(h(X) = 1 | A = 0\right) - \mathbb{P}_{(X,Y,A)\sim p}\left(h(X) = 1 | A = 0\right)\right| 
\\
& \quad + \sup_{h\in\mathcal{H}}\left|\mathbb{P}_{(X,Y,A)\sim p_i}\left(h(X) = 1 | A = 1\right) - \mathbb{P}_{(X,Y,A)\sim p}\left(h(X) = 1 | A = 1\right)\right| \\
& \leq 2\eta.
\end{align*}
\subsection{Proof}\label{sec:proof_now_really}

\addtocounter{theorem}{-1}

\begin{theorem}%
Assume that $\mathcal{H}$ has a finite VC-dimension $d \geq 1$. Let $p$ be an arbitrary target data distribution and without loss of generality let $\tau = p(a=0) \in \left(0, 0.5\right]$. Let $S_1,\dots,S_N$ be $N$ datasets, each consisting of $n$ samples, out of which $K>\frac{N}{2}$ 
are sampled $\iid$ from a data distribution $p_i$ that is $\eta$-close the distribution $p$ in the sense of Section \ref{sec:proof_assumptions}. Assume that $18\eta < \tau$.
For $\frac{1}{2} < \beta \leq \frac{K}{N}$ and 
$I=\textsc{FilterSources}(S_1,\dots,S_N;\beta)$ set $S=\bigcup_{i\in I} S_i$. 
Let $\delta>0$. Then there exists a constant $C = C(\delta, \tau, d, N, \eta)$, such that for any $n \geq C$, the following inequalities hold with 
probability at least $1-\delta$ uniformly over all $f\in\H$ and against any adversary:
\begin{align}
  \left|\unfairS(f) - \unfair{p}(f) \right| &\leq  \mathcal{O}\left(\eta\right) + \widetilde{\mathcal{O}}\left(\sqrt{\frac{1}{n}}\right),
\qquad\qquad
\left|\mathcal{R}_S(f) - \mathcal{R}_p(f)\right| \leq \mathcal{O}\left(\eta\right) + \widetilde{\mathcal{O}}\left(\sqrt{\frac{1}{n}}\right). %
\label{eqn:theory_results}
\end{align}
\end{theorem}

\begin{proof}
First, we characterize a set of values into which the empirical 
risks and empirical deviation measures of the clean data 
sources falls with probability at least $1-\delta$. 
Then we show that because the clean datasets cluster in such a way, 
any individual dataset that is accepted by the \filter 
algorithm provides good empirical estimates of the true risk 
and the unfairness measure.
Finally, we show that the same holds for the union of these 
sets, $S$, which implies the inequalities~\eqref{eqn:theory_results}.
For the risk, the last step is a straightforward
consequence of the second. For the fairness, however, 
a careful derivation is needed that crucially uses 
the disbalance measure as well.
 
\paragraph{Step 1} Let $G\subset [N]$ be the set of indexes $i$, such that $S_i$ was not modified by the adversary. By definition, $|G| = K$. Now consider the following events that, as we will show, describe the likely values of the studied quantities on the clean datasets. 

In particular, for all $i \in G$, let $\mathcal{E}^{\RISK}_i$ be the event that:
\begin{equation}
\label{eqn:good_event_risk}
\sup_{h \in \H} \big|\mathcal{R}_{S_i}(h) - \mathcal{R}_{p_i}(h)\big| \leq \Delta\left(\frac{\delta}{6N}\right),
\end{equation}
let $\mathcal{E}^{\Gamma}_i$ be the event that
\begin{equation}
\label{eqn:good_event_fairness}
\sup_{h \in \H} \big|\unfair{S_i}(h) - \unfair{p_i}(h)\big| \leq \Delta\left(\frac{\delta}{6N}\right),
\end{equation}
let $\mathcal{E}^{bin}_i$ be the event that
\begin{equation}
\label{eqn:good_event_binomial}
\left|c_i - n\tau_i\right| \leq n\Delta\left(\frac{\delta}{6N}\right)
\end{equation}
and finally, let $\mathcal{E}^{count}_i$ be the event that
\begin{equation}
\label{eqn:no_trivial_binomial}
0 < c_i < n.
\end{equation}
Denote by $(\mathcal{E}^{\RISK}_i)^c, (\mathcal{E}^{\Gamma}_i)^c$ and $(\mathcal{E}^{bin}_i)^c, \left(\mathcal{E}^{count}_i\right)^c$ the respective complements of these events. Then, by equations \eqref{eqn:bad_bound_risk} and \eqref{eqn:bad_bound_fairness}, \eqref{eqn:bad_bound_binomial}, we have:
$$ \mathbb{P}((\mathcal{E}^{\RISK}_i)^c) \leq \frac{\delta}{6N}, \quad \mathbb{P}((\mathcal{E}^{\Gamma}_i)^c) \leq \frac{\delta}{6N}, \quad \mathbb{P}((\mathcal{E}^{bin}_i)^c) \leq \frac{\delta}{6N}, \quad \forall i \in G.$$

To bound the probability of the complement of $\mathcal{E}^{count}_i$, note that for any $i \in G$
$$1 - \tau_i = \mathbb{P}_{(X,Y,A)\sim p_i}(A = 1) \leq \mathbb{P}_{(X,Y,A)\sim p}(A = 1) + \eta = 1 - \tau + \eta$$
and that $1 - \tau + \eta < 1$ because of the assumption that $\eta < \tau$.
Similarly, 
$$\tau_i =  \mathbb{P}_{(X,Y,A)\sim p_i}(A = 0) \leq \mathbb{P}_{(X,Y,A)\sim p}(A = 0) + \eta = \tau + \eta \leq 1 - \tau + \eta.$$

Now, for any $i\in G$, whenever $n \geq C_1 (\delta, \tau, d, N) = \frac{\log\left(\frac{4N}{\delta}\right)}{\log\left(\frac{1}{1 - \tau + \eta}\right)} \geq \max \left\{\frac{\log\left(\frac{4N}{\delta}\right)}{\log\left(\frac{1}{1-\tau_i}\right)}, \frac{\log\left(\frac{4N}{\delta}\right)}{\log\left(\frac{1}{\tau_i}\right)}\right\}$, we have that 
\begin{align*}
\mathbb{P}\left(\left(\mathcal{E}^{count}_i\right)^c\right) & = (1-\tau_i)^n + \tau_i^n \\
& \leq \exp\left(-n\log\left(\frac{1}{1-\tau_i}\right)\right) + \exp\left(-n\log\left(\frac{1}{\tau_i}\right)\right) \\
& \leq \frac{\delta}{4N} + \frac{\delta}{4N} = \frac{\delta}{2N}.
\end{align*}

Therefore, setting $\mathcal{E} := (\land_{i \in G} \mathcal{E}^{\RISK}_i) \land (\land_{i \in G} \mathcal{E}^{\Gamma}_i) \land (\land_{i \in G} \mathcal{E}^{bin}_i) \land  (\land_{i\in G} \mathcal{E}^{count}_i)$ then by the union bound the probability of $\mathbb{P}\left(\mathcal{E}^{c}\right) \leq K\frac{\delta}{6N} + K\frac{\delta}{6N} + K\frac{\delta}{6N} + K\frac{\delta}{2N} \leq 3\frac{\delta}{6} + \frac{\delta}{2} =\delta$.

Hence the probability of the event $\mathcal{E}$ that all of \eqref{eqn:good_event_risk}, \eqref{eqn:good_event_fairness}, \eqref{eqn:good_event_binomial}, and \eqref{eqn:no_trivial_binomial} hold is at least $1 - \delta$.

\paragraph{Step 2} Now we show that under the event $\mathcal{E}$, the inequalities in \eqref{eqn:theory_results} are fulfilled. Indeed, assume that $\mathcal{E}$ holds. Fix any adversary $\mathcal{A}$ and any $h\in \mathcal{H}$.

For any pair of clean sources $i,j\in [N]$ the triangle law and the derivations in Section \ref{sec:distances_distributions} give:
\begin{align*}
\disc(S_i, S_j) & = \sup_{h\in\H}|\mathcal{R}_{S_i}(h) - \mathcal{R}_{S_j}(h)| \\
& \leq \sup_{h\in\H}|\mathcal{R}_{S_i}(h) - \mathcal{R}_{p_i}(h)| + \sup_{h\in\H}|\mathcal{R}_{p_i}(h) - \mathcal{R}_{p}(h)| + \sup_{h\in\H}|\mathcal{R}_{p}(h) - \mathcal{R}_{p_j}(h)| + \sup_{h\in\H}|\mathcal{R}_{p_j}(h) - \mathcal{R}_{S_j}(h)| \\
& \leq 2\eta + 2\Delta\left(\frac{\delta}{6N}\right).
\end{align*}
Similarly,
$$\disp(S_i, S_j) = \sup_{h\in\H}|\unfair{S_i}(h) - \unfair{S_j}(h)| \leq 4\eta + 2\Delta\left(\frac{\delta}{6N}\right)$$ and $$\disb(S_i, S_j) =\left|\frac{c_i}{n} - \frac{c_j}{n}\right| \leq 2\eta + 2\Delta\left(\frac{\delta}{6N}\right).$$
Therefore, for any pair of clean sources $i,j\in [N]$:
\begin{align}
\disc(S_i, S_j) + \disp(S_i, S_j) + \disb(S_i, S_j) \leq 8\eta + 6\Delta\left(\frac{\delta}{6N}\right).
\end{align}
It follows that, under $\mathcal{E}$, we have that $q_i \leq 8\eta + 6\Delta\left(\frac{\delta}{6N}\right)$ for any clean $i\in [N]$. Since the fraction of clean sources is $\frac{K}{N} \geq \beta$, it follows that also $q \leq 8\eta + 6\Delta\left(\frac{\delta}{6N}\right)$, where $q$ is the $\beta$-th quantile of the $q_i$'s.

Denote by $I = \textsc{FilterSources}(S_1,\dots,S_N; \beta)$ the result of the filtering algorithm. Now for any $i \in I$, we have that $q_i \leq q \leq 8\eta + 6\Delta\left(\frac{\delta}{6N}\right)$. In addition, by the definition of $q_i$, $\disc(S_i, S_j) \leq \disc(S_i, S_j) + \disp(S_i, S_j) + \disb(S_i, S_j) \leq q_i$ for at least $|I| = \beta N > \frac{N}{2}$ values of $j\in [N]$. Since $K > \frac{N}{2}$, this means that $\disc(S_i, S_j) \leq q_i \leq 8\eta + 6\Delta\left(\frac{\delta}{6N}\right)$ for at least $1$ value $j\in G$. Therefore, we have:
\begin{align}
\label{eqn:risk_bound_trusted_source}
\sup_{h\in\H}\left|\mathcal{R}_{S_i}(h) - \mathcal{R}_p(h)\right| & \leq \sup_{h\in\H}\left|\mathcal{R}_{S_i}(h) - \mathcal{R}_{S_j}(h)\right| + \sup_{h\in\H}\left|\mathcal{R}_{S_j}(h) - \mathcal{R}_{p_j}(h)\right| +  \sup_{h\in\H}\left|\mathcal{R}_{p_j}(h) - \mathcal{R}_{p}(h)\right| \\
& \leq 8\eta + 6\Delta\left(\frac{\delta}{6N}\right) + \Delta\left(\frac{\delta}{6N}\right) + \eta \\
& = 9\eta + 7\Delta\left(\frac{\delta}{6N}\right)
\end{align}
because $\mathcal{E}$ holds. Similarly,
\begin{align}\label{eqn:fairness_bound_trusted_source}
\sup_{h\in\H}\left|\unfair{S_i}(h) - \unfair{p}(h)\right| \leq 10\eta + 7\Delta\left(\frac{\delta}{6N}\right)
\end{align}
and
\begin{align}
\label{eqn:balance_bound_trusted_source}
\left|c_i - n\tau\right| \leq 9\eta n + 7n\Delta\left(\frac{\delta}{6N}\right).
\end{align}

\paragraph{Step 3} Finally, we study the risk and disparity measures based on all filtered data $S = \cup_{i\in I} S_i$.

Denote by $\mathcal{R}_S(h)$ the empirical risk across the entire trusted dataset $I$:
\begin{equation}
\label{eqn:rsik_trusted_estimate}
\mathcal{R}_S(h) := \frac{1}{|I|}\sum_{i\in I}\mathcal{R}_{S_i}(h).
\end{equation}
Then the triangle law gives:
\begin{align*}
|\mathcal{R}_S(h) - \mathcal{R}_p(h)| = \left|\frac{1}{|I|}\left(\sum_{i\in I} \mathcal{R}_{S_i}(h) - \mathcal{R}_p (h)\right)\right| \leq \frac{1}{|I|}\sum_{i\in I}\left|\mathcal{R}_{S_i}(h) - \mathcal{R}_p(h)\right| = 9\eta + 7\Delta\left(\frac{\delta}{6N}\right)
\end{align*}
Since
\begin{equation}
\label{eqn:delta_rate}
3\Delta\left(\frac{\delta}{6N}\right) = 112\sqrt{2\frac{d \log\big(\frac{2en}{d}\big) + \log\big(\frac{144N}{\delta}\big)}{(\tau-\eta) n}} = \widetilde{\mathcal{O}}\left(\sqrt{\frac{d}{(\tau-\eta) n}}\right),
\end{equation}
the bound on the risk follows.

Denote by $\unfair{S}(h)$ the empirical estimate of demographic parity across the entire trusted dataset $I$:
\begin{equation}
\label{eqn:fairness_trusted_estimate}
\unfair{S}(h) :=  \frac{\sum_{j\in I} \sum_{i=1}^n \mathbbm{1}\{h(x^{(j)}_i) = 1, a^{(j)}_i =0\}}{\sum_{j\in I} \sum_{i=1}^n \mathbbm{1} \{a^{(j)}_i = 0\}} - \frac{\sum_{j\in I} \sum_{i=1}^n \mathbbm{1}\{h(x^{(j)}_i) = 1, a^{(j)}_i = 1\}}{\sum_{j\in I} \sum_{i=1}^n \mathbbm{1} \{a^{(j)}_i = 1\}} .
\end{equation}
For convenience, denote $v_j = v_j(h) = \sum_{i=1}^n \mathbbm{1}\{h(x^{(j)}_i) = 1, a^{(j)}_i =0\}$ and $w_j = w_j(h) = \mathbbm{1}\{h(x^{(j)}_i) = 1, a^{(j)}_i = 1\}$, so that:
\begin{align*}
\unfair{S}(h) = \frac{\sum_{j\in I} v_j}{\sum_{j\in I} c_j} - \frac{\sum_{j\in I} w_j}{\sum_{j\in I} (n - c_j)}.
\end{align*}

Our goal is to bound the difference $\left|\unfair{S}-\frac{1}{|I|}\sum_{i \in I}\unfair{S_i}\right|$, and the difference $\left|\frac{1}{\left|I\right|}\sum_{i \in I}\unfair{S_i}-\unfair{p}\right|$, and use these two bounds to bound  $\left|\unfair{S}-\unfair{p}\right|$. The second bound follows directly from \eqref{eqn:fairness_bound_trusted_source}:
\begin{equation}\label{eqn:fairness_bound_avg_of_selected}
  \Big|\unfair{p}-\frac{1}{\left|I\right|}\sum_{i \in I}\unfair{S_i}\Big| \leq  10\eta + 7\Delta\left(\frac{\delta}{6N}\right) 
\end{equation}

To compute the first bound, we first build on \eqref{eqn:balance_bound_trusted_source} to note
\begin{align*}
  \left|\frac{c_i}{\jenc}-1\right| &= \frac{\left|c_i-\jenc\right|}{\jenc}\leq\frac{\jenepsilon}{\jenc} \leq \jene 
  \\
  \left|\frac{\jenc}{c_i}-1\right| &= \frac{\left|\jenc-c_i\right|}{c_i}\leq\frac{\jenepsilon}{c_i} \leq \jene
\end{align*}
and therefore,
\begin{align}
  1-\jene \leq \frac{c_i}{\jenc} &\leq 1+\jene \label{eqn:fairness_ci_c} \\
  1-\jene \leq \frac{\jenc}{c_i} &\leq 1+\jene \label{eqn:fairness_c_ci} 
\end{align}
Applying the same logic to $n-c_i$: 
\begin{align}
  1-\jenf \leq \frac{n-c_i}{n-\jenc} &\leq 1+\jenf \label{eqn:fairness_nci_nc}  \\
  1-\jenf \leq \frac{n-\jenc}{n-c_i} &\leq 1+\jenf \label{eqn:fairness_nc_nci}  
\end{align}

Now consider,
\begin{align*}
 \frac{1}{\left|I\right|}\sum_{j \in I}\unfair{S_j} 
 &= \frac{1}{\left|I\right|}\sum_{j\in I}\frac{v_j}{c_j} - \frac{1}{\left|I\right|}\sum_{j\in I}\frac{w_j}{ (n - c_j)} 
 \\
 &\leq \frac{1}{\left|I\right|}\sum_{j\in I}\frac{v_j}{c}\left(1+\jene\right) - \frac{1}{\left|I\right|}\sum_{j\in I}\frac{w_j}{ (n - c)}\left(1-\jenf\right) 
 \\
 &= \frac{\sum_{j\in I} v_j}{\sum_{j\in I} c}\left(1+\jene\right) - \frac{\sum_{j\in I} w_j}{\sum_{j\in I} (n - c)}\left(1-\jenf\right) 
 \\
 &\leq \frac{\sum_{j\in I} v_j}{\sum_{j\in I} c_j\left(1-\jene\right)}\left(1+\jene\right) 
 \\
 &\qquad - \frac{\sum_{j\in I} w_j}{\sum_{j\in I} (n - c_j)\left(1+\jenf\right)}\left(1-\jenf\right) 
 \\
 &= \left(\frac{\sum_{j\in I} v_j}{\sum_{j\in I} c_j}\right)\frac{1+\jene}{1-\jene} - \left(\frac{\sum_{j\in I} w_j}{\sum_{j\in I} (n - c_j)}\right)\frac{1-\jenf}{1+\jenf}
 \\
 &= \frac{\sum_{j\in I} v_j}{\sum_{j\in I} c_j}-\frac{\sum_{j\in I} w_j}{\sum_{j\in I} (n - c_j)} 
+ 2\left(\frac{\sum_{j\in I} v_j}{\sum_{j\in I} c_j}\right)\frac{\jene}{1-\jene} + 2\left(\frac{\sum_{j\in I} w_j}{\sum_{j\in I} (n - c_j)}\right)\frac{\jenf}{1+\jenf},
 \intertext{where we have used that $\frac{1+t}{1-t}=1+\frac{2t}{1-t}$ and $\frac{1-t}{1+t}=1-\frac{2t}{1+t}$. Now, using $v_j\leq c_j$ and $w_j\leq n-c_j$ and simplifying the fractions further, we obtain}
 &\leq \frac{\sum_{j\in I} v_j}{\sum_{j\in I} c_j}-\frac{\sum_{j\in I} w_j}{\sum_{j\in I} (n - c_j)} + \frac{2\left(\jenepsilonwithoutn\right))}{\jencwithoutn-2\left(\jenepsilonwithoutn\right)} + \frac{2\left(\jenepsilonwithoutn\right)}{1-\jencwithoutn} 
 \\
 &= \unfair{S} + \frac{2\left(\jenepsilonwithoutn\right))}{\jencwithoutn-2\left(\jenepsilonwithoutn\right)} + \frac{2\left(\jenepsilonwithoutn\right)}{1-\jencwithoutn} 
\end{align*}

Using analogue steps, we can show that
\begin{equation*}
  -\frac{1}{\left|I\right|}\sum_{i \in I}\unfair{S_i} \leq -\unfair{S} + \frac{2\left(\jenepsilonwithoutn\right))}{\jencwithoutn} + \frac{2\left(\jenepsilonwithoutn\right)}{1-\jencwithoutn-2\left(\jenepsilonwithoutn\right)} 
 \end{equation*}

Combining these two bounds:
\begin{equation}\label{eqn:fairness_union_sum_boundz}
  \left|\frac{1}{\left|I\right|}\sum_{i \in I}\unfair{S_i} - \unfair{S}\right| \leq \frac{2\left(\jenepsilonwithoutn\right)}{\jencwithoutn-2\left(\jenepsilonwithoutn\right)} + \frac{2\left(\jenepsilonwithoutn\right)}{1-\jencwithoutn-2\left(\jenepsilonwithoutn\right)} 
\end{equation}

Now, combining \eqref{eqn:fairness_union_sum_boundz} with \eqref{eqn:fairness_bound_avg_of_selected}, and using the triangle inequality as before,
\begin{equation}
  \left|\unfair{p} - \unfair{S}\right| \leq 10\eta + 7\Delta\left(\frac{\delta}{6N}\right) + \frac{2\left(\jenepsilonwithoutn\right)}{\jencwithoutn-2\left(\jenepsilonwithoutn\right)} + \frac{2\left(\jenepsilonwithoutn\right)}{1-\jencwithoutn-2\left(\jenepsilonwithoutn\right)}  \label{eqn:fairness_final_bound}
\end{equation}

Recalling from \eqref{eqn:delta_definition} that $\Delta = 16\sqrt{2\frac{d \log\big(\frac{2en}{d}\big) + \log\big(\frac{24}{\delta}\big)}{n(\tau - \eta)}}$, we obtain that 
\begin{equation}
  \left|\unfair{p} - \unfair{S}\right| \leq \mathcal{O}\left(\eta\right) + \widetilde{\mathcal{O}}\left(\frac{1}{\sqrt{n}}\right).
\end{equation}

\end{proof}

\subsection{Proof of Lemma \ref{lemma:uniform_convergence}}
\label{sec:proof_lemma}

Let $S = \{(x_i, y_i, a_i)\}_{i=1}^n$. For $a \in \{0,1\}$, denote: 
\begin{align}
\gamma^a_S(h) = \frac{\sum_{i=1}^n \mathbbm{1}\{h(x_i) = 1, a_i = a\}}{\sum_{i=1}^n \mathbbm{1}\{a_i = a\}}
\end{align}
and 
\begin{align}
\gamma^a_p(h) = \mathbb{P}(h(X) = 1| A = a),
\end{align}
so that $\unfair{S}(h) = \gamma^0_S(h) - \gamma^1_S(h)$ and $\unfair{p}(h) = \gamma^0_p(h) - \gamma^1_p(h)$.

First we use a technique of \citet{woodworth17learning, agarwal2018reductions} for proving concentration results about conditional probability estimates to bound the probability of a large deviation of $\unfair{S}(h)$ from $\unfair{p}(h)$, for a fixed hypothesis $h\in\mathcal{H}$. Our result is similar to the one in \citet{woodworth17learning}, but for demographic parity, instead of equal odds.

\begin{lemma}
\label{lemma:non-uniform-bound-demog-parity}
Let $h\in\mathcal{H}$ be a fixed hypothesis and $p\in \mathcal{P}(\prodspace)$ be a fixed distribution. Let $\tau = \min_{a\in \{0,1\}}\mathbb{P}_{(X,Y,A) \sim p}(A = a) \in (0,0.5]$. Then for any dataset $S$, drawn \iid from $p$, of size $n$ and for any $\delta\in(0,1)$ and any $t>0$:
\begin{equation}
\label{eqn:non-uniform-bound-demog-parity}
\mathbb{P}\left(\left|\unfair{S}(h) - \unfair{p}(h)\right| > 2t \right) \leq 6\exp\left(-\frac{t^2 \tau n}{8}\right).
\end{equation}
\end{lemma}
\begin{proof}
Denote by $S_{a} = \{i\in [n]: a_i = a\}$ the set of indexes of the points in $S$ for which the protected group is $a$. Let $c_a \coloneqq |S_{a}|$ and $P_a = \mathbb{P}_{(X,Y,A) \sim p}(A = a)$, so that $\tau = \min_a P_a$. For both $a\in\{0,1\}$, we have:
\begin{align*}
\mathbb{P}\left(\left|\gamma^a_S - \gamma^a_p\right| > t\right) & = \sum_{S_{a}} \mathbb{P}\left(\left|\gamma^a_S - \gamma_{a}\right| > t\middle| S_{a}\right)\mathbb{P}(S_{a}) \\
& \leq \mathbb{P}\left(c_a \leq \frac{1}{2}P_an\right) + \sum_{S_{a}: c_a > \frac{1}{2}P_an}\mathbb{P}\left(\left|\gamma^a_S - \gamma_{a}\right| > t\middle| S_{a}\right)\mathbb{P}(S_{a}) \\ 
& \leq \exp \left(-\frac{P_a n}{8}\right) + \sum_{S_{a}: c_a > \frac{1}{2}P_an} 2 \exp\left(-2t^2c_a\right) \mathbb{P}(S_{a}) \\ 
& \leq \exp \left(-\frac{P_a n}{8}\right) + 2\exp\left(-t^2 P_{a}n\right)\\
& \leq 3\exp\left(-\frac{t^2 \tau n}{8}\right).
\end{align*}
The triangle law gives:
\begin{align*}
|(\gamma^{0}_{S} - \gamma^{1}_{S}) - (\gamma^{0}_p - \gamma^{1}_p)| = |\gamma^{0}_{S} - \gamma^{1}_{S} - \gamma^{0}_p + \gamma^{1}_p| & \leq |\gamma^{0}_{S} - \gamma^{0}_{p}| + |\gamma^{1}_{S} -  \gamma^{1}_{p}|.
\end{align*}
Combining the previous two results:
\begin{align*}
\mathbb{P}(|(\gamma^{0}_{S} - \gamma^{1}_{S}) - (\gamma^{0}_p - \gamma^{1}_p)| > 2t) & \leq \mathbb{P}\left(\left|\gamma^{0}_{S} - \gamma^{0}_{p}\right| + \left|\gamma^{1}_{S} - \gamma^{1}_{p}\right| > 2t\right) \\ & \leq \mathbb{P}\left(\left(\left|\gamma^{0}_{S} - \gamma^{0}_{p}\right| > t\right) \lor \left(\left|\gamma^{1}_{S} - \gamma^{1}_{p}\right| > t\right)\right) \\ & \leq \mathbb{P}\left(\left|\gamma^{0}_{S} - \gamma^{0}_{p}\right| > t\right) + \mathbb{P}\left(\left|\gamma^{1}_{S} - \gamma^{1}_{p}\right| > t\right) \\ & \leq 6\exp\left(-\frac{t^2 \tau n}{8}\right).
\end{align*}
\end{proof}

Finally, we prove Lemma \ref{lemma:uniform_convergence} by extending the previous result to hold uniformly over the whole hypothesis space, for any hypothesis space $\mathcal{H}$ with a finite VC-dimension $d \vcentcolon = \VC(\mathcal{H})$. The extension is essentially identical to \citet{konstantinov2022theotherone} and is included here for completeness.

\addtocounter{lemma}{-2}
\begin{lemma}[Uniform convergence for demographic parity]
    Let $d = \VC(\HYPS) \geq 1$ and let $\tau = \min_{a\in\{0,1\}}\mathbb{P}_{(X,Y,A)\sim p}(A = a)$ for some constant $\tau\in (0,0.5]$. Then for any dataset $S$ of size $n \geq \max \Big\{\frac{8\log\left(\frac{8}{\delta}\right)}{\tau}, \frac{d}{2}\Big\}$ sampled \iid from $p$, for all $\delta \in (0,1/2)$:
    \begin{equation}
        \mathbb{P}_S \Bigg( \sup_{h \in \H}\left|\unfairS(h) - \unfair{p}(h)\right|
        \geq  16\sqrt{2\frac{d \log\big(\frac{2en}{d}\big) + \log\big(\frac{24}{\delta}\big)}{n\tau}} \Bigg) \leq \delta
    \end{equation}
 \end{lemma} 

\begin{proof}
To extend Lemma \ref{lemma:non-uniform-bound-demog-parity} to hold uniformly over $\mathcal{H}$, we first prove a version of the classic symmetrization lemma \citep{vapnik2013nature} for $\unfair{}$ and then proceed via a standard growth function argument.

1) Consider a ghost sample $S' = \{(x'_i, a'_i, y'_i)\}_{i=1}^n$ also sampled \iid from $p$. For any $h \in \mathcal{H}$, let $\unfair{S'}(h)$ be the empirical estimate of $\unfair{p}(h)$ based on $S'$.

We show the following symmetrization inequality for the $\unfair{}$ measure:
\begin{align}
\label{eqn:symmetrization-demog-par}
\mathbb{P}_{S}\left(\sup_{h\in\mathcal{H}}\left|\unfair{S}(h) - \unfair{p}(h)\right| \geq t\right) & \leq 
 2 \mathbb{P}_{S, S'}\left(\sup_{h\in\mathcal{\mathcal{H}}}\left|\unfair{S'}(h) - \unfair{S}(h)\right| \geq t/2\right),
\end{align}
for any constant $t \geq 8\sqrt{\frac{2\log(12)}{n\tau}}$.
 
Indeed, let $h^*$ be the hypothesis achieving the supremum on the left-hand side.\footnote{If the supremum is not attained, the argument can be repeated for each element of a sequence of classifiers approaching the supremum} 
Then:
\begin{align*}
\mathbbm{1}(\left|\unfair{S}(h^*) - \unfair{p}(h^{*})\right| \geq t) & \mathbbm{1}(\left|\unfair{S'}(h^*) - \unfair{p}(h^*)\right| \leq t/2) \leq \mathbbm{1}(\left|\unfair{S'}(h^*) - \unfair{S}(h^*)\right| \geq t/2).
\end{align*}
Taking expectation with respect to $S'$:
\begin{align*}
\mathbbm{1}(\left|\unfair{S}(h^*) - \unfair{p}(h^{*})\right| \geq t)\mathbb{P}_{S'}(\left|\unfair{S'}(h^*,S') - \unfair{p}(h^*)\right| \leq t/2) \leq \mathbb{P}_{S'}(\left|\unfair{S'}(h^*) - \unfair{S}(h^*)\right| \geq t/2).
\end{align*}
Now using Lemma \ref{lemma:non-uniform-bound-demog-parity}:
\begin{align*}
\mathbb{P}_{S'}\left(\left|\unfair{S'}(h^*) - \unfair{p}(h^*)\right| \leq t/2\right)  \geq 1 - 6\exp\left(-\frac{t^2 \tau n}{128}\right) \geq 1 - \frac{1}{2} = \frac{1}{2},
\end{align*}
where the second inequality follows from the condition $t\geq 8\sqrt{\frac{2\log(12)}{n\tau}}$. Therefore,
\begin{align*}
\frac{1}{2}\mathbbm{1}(\left|\unfair{S}(h^*) - \unfair{p}(h^{*})\right| \geq t) \leq \mathbb{P}_{S'}(\left|\unfair{S'}(h^*) - \unfair{S}(h^*)\right| \geq t/2).
\end{align*}
Taking expectation with respect to $S$:
\begin{align*}
\mathbb{P}_{S}(\left|\unfair{S}(h^*) - \unfair{p}(h^{*})\right| \geq t) & \leq 2 \mathbb{P}_{S, S'}(\left|\unfair{S'}(h^*) - \unfair{S}(h^*)\right| \geq t/2) \\
& \leq 2 \mathbb{P}_{S, S'}(\sup_{h\in\mathcal{H}}\left|\unfair{S'}(h) - \unfair{S}(h)\right| \geq t/2).
\end{align*}

2) Next we use the symmetrization inequality \eqref{eqn:symmetrization-demog-par} to bound the large deviation of $\unfair{S}(h)$ uniformly over $\mathcal{H}$.

Specifically, given $n$ points $x_1, \ldots, x_n \in \mathcal{X}$, denote $$ \mathcal{H}_{x_1, \ldots, x_n} \{(h(x_1), \ldots, h(x_n)): h \in \mathcal{H}\}.$$ 

Then define the growth function of $\mathcal{H}$ as:
\begin{equation}
\label{eqn:growth-function}
G_{\mathcal{H}}(n) = \sup_{x_1, \ldots, x_n} |\mathcal{H}_{x_1, \ldots, x_n}|.
\end{equation}
We will use that well-known Sauer's lemma \citep{vapnik2013nature}, which states that whenever $n\geq d$, $G_{\mathcal{H}}(n) \leq \left(\frac{en}{d}\right)^d$

Notice that given the two datasets $S, S'$, the values of $\unfair{S}$ and $\unfair{S'}$ depend only on the values of $h$ on $S$ and $S'$ respectively. Therefore, for any $t \geq  8\sqrt{\frac{2\log(12)}{\tau n}}$,

    \begin{align}
        \mathbb{P}_S \Big( \sup_{h \in \HYPS}|\unfair{S} - \unfair{p}(h)| \geq t \Big) 
        & \leq 2\mathbb{P}_{S, S'}\Big(\sup_{h \in \HYPS}|\unfair{S'}(h) - \unfair{S}(h)| \geq \frac{t}{2}\Big) \\
        & \leq 2G_{\HYPS}(2n)\mathbb{P}_{S, S'}\left(|\unfair{S'}(h) - \unfair{S}(h)| \geq \frac{t}{2} \right)\\
        & \leq 2G_{\HYPS}(2n)\mathbb{P}_{S, S'}\left(\left(|\unfair{S}(h) - \unfair{p}(h)| \geq \frac{t}{4} \right) \vee \left(|\unfair{S'}(h) - \unfair{p}(h)| \geq \frac{t}{4}\right) \right) \\
        & \leq 4G_{\HYPS}(2n) \mathbb{P}_{S}\Big(|\unfair{S}(h) - \unfair{p}(h)| \geq \frac{t}{4} \Big) \\
        & \leq 24G_{\HYPS}(2n)\exp\left(-\frac{t^2\tau n}{516}\right) \\
        &\leq 24\left(\frac{2en}{d}\right)^d \exp\left(-\frac{t^2\tau n}{516}\right).
    \end{align}
    Here the second-to-last inequality is due to the same bound on the difference between $\unfair{S}$ and $\Gamma{p}$ that was used in the previous lemma, and the last one follows from Sauer's lemma. Now if we use the threshold $t=16\sqrt{2\frac{d \log\big(\frac{2eN}{d}\big) + \log\big(\frac{24}{\delta} \big) }{\tau n}  } > 8\sqrt{\frac{2\log(12)}{\tau n}}$, we get:
    \begin{equation}
        \mathbb{P}_S \Bigg( \sup_{h \in \mathcal{H}}\left|\unfair{S}(h) - \unfair{p}(h)\right| \geq 16\sqrt{2\frac{d \log\big(\frac{2en}{d}\big) + \log\big(\frac{24}{\delta} \big) }{\tau n}  } \Bigg) < \delta.
    \end{equation}
\end{proof}

\bibliographystylesupp{icml2022}
\bibliographysupp{ms}

\end{document}